\documentclass{article} 
\usepackage{arxiv,times}

\usepackage{amsmath,amsfonts,bm}

\def\eqref#1{equation~\ref{#1}}

\def\1{\bm{1}}

\DeclareMathAlphabet{\mathsfit}{\encodingdefault}{\sfdefault}{m}{sl}
\SetMathAlphabet{\mathsfit}{bold}{\encodingdefault}{\sfdefault}{bx}{n}

\newcommand{\R}{\mathbb{R}}

\usepackage{hyperref}
\usepackage[capitalise,noabbrev]{cleveref}
\usepackage{amsmath}
\usepackage{amsthm}
\usepackage[shortlabels]{enumitem}
\usepackage{url}
\usepackage{amssymb}
\usepackage{booktabs}
\usepackage{shuffle}
\usepackage{tikz}
\usetikzlibrary{arrows.meta,positioning}

\theoremstyle{plain}
\newtheorem{theorem}{Theorem}
\newtheorem{proposition}{Proposition}
\newtheorem{lemma}{Lemma}
\newtheorem{corollary}{Corollary}

\theoremstyle{definition}
\newtheorem{definition}{Definition}
\newtheorem{example}{Example}

\theoremstyle{remark}
\newtheorem{remark}{Remark}

\usepackage{xcolor}

\newcounter{issuecounter}

\title{Extending SSMs with the Exponentially Weighted Signature}

\author{Alexandre Bloch$^{1}$, Benjamin Walker$^{1,2}$, Jo\"{e}l Mouterde$^{3}$, Sam Morley$^{1}$, Samuel N.~Cohen$^{1}$ \& Terry Lyons$^{1,4}$ \\
$^{1}$Mathematical Institute, University of Oxford \quad $^{2}$Continuous Intelligence \quad $^{3}$SKF \\
$^{4}$Department of Mathematics, Imperial College London
}

\begin{document}

\maketitle

\begin{abstract}
We introduce the exponentially weighted signature (EWS), a continuous-time model that computes iterated integrals of a path, where each increment is weighted by the matrix exponential of a learnable generator over elapsed clock time. We prove that it solves a linear controlled differential equation, keeps the group-like structure and the universality of the signature, and satisfies a modified Chen identity, enabling a parallel scan. At depth one the EWS is a state-space model (SSM), and we map linear time-invariant SSMs, Mamba channels and Mamba-$2$ heads to it in closed form. The EWS extends SSMs through an arbitrary matrix generator, a clock that generalises the step size to causal functionals of the input, and higher truncation depths that are non-linear in the path within a single layer. Empirically, the EWS achieves the highest average accuracy and rank on six long time-series classification datasets, where depth generally helps. Learned clocks prove necessary for state tracking on formal language tasks, and at depth one, the EWS matches or exceeds competing SSMs on regression and forecasting with far fewer parameters.
\end{abstract}

\section{Introduction}

State-space models (SSMs) are a leading architecture for learning on long sequences \citep{gu2021efficiently, gu2024mamba, smith2023simplified, boyer2025learningdissipateenergyoscillatory}. Their update is linear in the state, so they train in parallel yet run recurrently at inference with constant memory per step \citep{martin2018parallelizing, dao2024transformers, smith2023simplified}. Most SSMs are obtained by discretising a continuous-time linear dynamical system forced by the input. In linear time-invariant (LTI) SSMs, including S4, S5, LinOSS and D-LinOSS, the step size and discretised system matrices are fixed, such that each layer is a convolution of its inputs, a structure that performs strongly on long continuous signals such as audio, images and physiological recordings \citep{rusch2025oscillatory, nguyen2022snd, pmlr-v162-goel22a}. In selective SSMs, notably the Mamba family, these quantities depend on the input, letting the model choose what to retain and making it competitive on discrete, information-dense data such as text, which has led large language models to interleave selective SSM layers with attention \citep{gu2024mamba, nvidia2025nvidianemotron3efficient, lahoti2026mamba, dao2024transformers}.

Working with the continuous-time system, rather than a fixed discretisation of it, has practical and theoretical benefits. The hidden state evolves between observations and can be queried at any time, so irregularly sampled and asynchronous data need not be forced onto a regular grid \citep{kidger2020neuralcde, Walker2024LogNCDE, walker2026faithfulembeddingsirregularasynchronous}, and elapsed time becomes a modelling choice, measured by timestamps, observation counts or a quantity learned from the data, as the step size is in selective SSMs \citep{gu2024mamba}. On the theoretical side, since an SSM treats its input as samples of a latent continuous signal, SSMs can be studied as linear controlled differential equations (CDEs) \citep{cirone2024deepSSM}, a framework that characterises what their recurrences can express and shows that structured linear CDEs are maximally expressive on paths \citep{kidger2022neuraldifferentialequations, walker2025advances, walker2025structuredlinearcdesmaximally}.

Behind every linear CDE lies the signature of its driving path, a collection of iterated integrals graded by depth \citep{Chen1954Iterated, lyons2007differential}. Any linear readout of a linear CDE is a series in these integrals, with coefficients set by the matrices of its vector fields \citep{frizmultidimensional, lyons1998differential}; thus the terms a model can use depend on its driving path and these matrices. An LTI SSM is a linear CDE driven by time with an additive input, such that only terms linear in the input appear. Its readouts are fixed linear filters, whereas a selective SSM is driven by a gate computed from the input, letting its filters adapt to context. With diagonal matrices, as in Mamba, the filters adapt only through increments of that gate \citep{cirone2024deepSSM}. The signature itself solves a linear CDE with fixed vector fields, and its linear functionals are universal on suitable compact sets of paths, making it a popular feature map \citep{graham2013sparsearrayssignaturesonline, gyurkó2014extractinginformationsignaturefinancial, levin2016learningpastpredictingstatistics, kiraly2019kernels, Salvi_2021}. However, it has no learnable parameters, and treats the path's history uniformly, offering no mechanism for fading memory \citep{1085649}, whereas SSMs forget at rates set by their learned dynamics and step size, a mechanism regarded as essential for long-sequence training \citep{NEURIPS2020_102f0bb6, gu2023how, NEURIPS2024_fbb07254, morita2025emergence}.

We introduce the exponentially weighted signature (EWS), whose iterated integrals weight each increment, after a learnable lift, by the matrix exponential of a learnable generator scaled by elapsed clock time. At depth one, the EWS is an SSM, and trained LTI SSMs convert to it in closed form, as does each channel of Mamba and head of Mamba-$2$, with its accumulated step size as the clock. The EWS extends SSMs in three ways. Its deeper levels weight non-commutative products of increments, making them non-linear in the path within the layer. Its clock generalises the step size to causal functions of the input path, including the accumulated norm of its increments and learned gates. Its generator may be any matrix, including Jordan blocks whose kernels carry a polynomial--exponential factor that none of these SSMs have. These extensions preserve the algebraic structure of the signature, as the EWS solves a linear CDE whose vector fields are those of the signature together with one built from the generator and driven by the clock. Like the signature, its state is graded and group-like \citep{reutenauer1993free, Chen1957IntegrationOP}, and it satisfies a modified Chen identity \citep{Chen1954Iterated} that permits a parallel associative scan \citep{BlellochTR90, smith2023simplified}.

\paragraph{Contributions}

\begin{enumerate}[leftmargin=1em]\setlength\itemsep{0pt}
\item The exponentially weighted signature (EWS) is introduced as a continuous-time layer with a learnable lift, generator and clock. It contains the signature, the EFM signature \citep{jaber2025exponentiallyfadingmemorysignature} and depth-one SSMs as special cases.
\item Trained LTI SSMs, each channel of Mamba and each head of Mamba-2 are shown to convert to depth-one EWS in closed form (Section~\ref{sec:ssm}). For any matrix generator, the EWS solves a linear CDE, satisfies a modified Chen identity (Proposition~\ref{prop:cde}, Theorem~\ref{thm:structure}(ii)), and is group-like and maximally expressive (when it separates paths), while its Jordan blocks give polynomial--exponential memory kernels at depth one that no diagonalisable SSM generator reaches (Proposition~\ref{prop:jordan}).
\item A comprehensive empirical evaluation showing that truncation depth and learned clocks each extend what SSMs can learn, with learned clocks needed for state tracking, depth necessary on synthetic targets and beneficial on the long time-series classification, enabling a single EWS layer to achieve the highest average accuracy and rank.
\item An open-source JAX implementation of the EWS, including the conversion of trained SSMs, with code to fully reproduce all experiments. This is available at \url{https://github.com/alexandre-bloch/exponentially-weighted-signature}.
\end{enumerate}

\paragraph{Related work} Memory has been added to the signature through sliding or expanding windows, which fix the horizon in advance \citep{NEURIPS2019_d2cdf047, morrill2021generalisedsignaturemethodmultivariate, cohen2025nowcastingusingregressionsignatures, drobac2025slidingwindowsignaturestimeseries}, through exponential weighting at prescribed rates \citep{jaber2025exponentiallyfadingmemorysignature} and through general memory kernels in the Volterra signature \citep{HARANG202134, doi:10.1142/S0219493723500028, Harang2024, hager2026volterrasignature}. Appendix~\ref{app:efm-volterra} details how these last two differ from the EWS. The EWS is itself a structured linear CDE \citep{walker2025structuredlinearcdesmaximally}, in the line of neural CDEs \citep{walker2025advances, kidger2022neuraldifferentialequations, kidger2020neuralcde, morrill2021neuralrough, Walker2024LogNCDE, walker2026faithfulembeddingsirregularasynchronous} and the theory of linear CDEs \citep{cirone2024deepSSM}. Input-dependent recurrences with non-diagonal transitions \citep{yang2024parallelizing, yang2025gated, siems2025deltaproduct, terzic2025structured, terzić2026flashpdssmmemoryoptimizedstructured} are compared to empirically in Section~\ref{sec:experiments}.

\paragraph{Background} Let $X \in \mathcal{V}^{1}([0,T];\mathbb{R}^d)$ be a path of bounded variation, with $\|X\|_{1,[s,t]}$ its total variation over $[s,t]$. Notation for path spaces and the tensor algebra is collected in Appendix~\ref{app:background}, where paths of finite $p$-variation with $p <2$ are also treated. The signature of $X$ over $[s,t] \subseteq [0,T]$ is $S(X)_{s,t} = \big(1, S(X)^{(1)}_{s,t},  S(X)^{(2)}_{s,t}, \dots\big)\in T((\mathbb{R}^d)) := \bigoplus_{n \geq 0}(\mathbb{R}^d)^{\otimes n}$ where
\begin{equation}\label{eq:sig}
   S(X)^{(n)}_{s,t} =  \int_{s < t_1 < \cdots < t_n <t} dX_{t_1} \otimes \cdots \otimes dX_{t_n}.
\end{equation}
The signature is graded, with depth $n$ holding the iterated integrals of order $n$, and $S(X)^{i_1, \dots, i_n}_{s,t}$ denoting the coefficient of $e_{i_1} \otimes \cdots \otimes e_{i_n} \in (\mathbb{R}^d)^{\otimes n}$ \citep{Chen1954Iterated, lyons2025signaturemethodsmachinelearning}. The universality of its linear functionals is recalled in Appendix~\ref{app:signature}.

A linear CDE driven by $X$ evolves a state $h_t \in \mathbb{R}^m$ by $dh_t = \sum_{i=1}^d M_i h_t dX^i_t$, where the matrices $M_i \in \mathbb{R}^{m \times m}$ are linear vector fields. When $X$ is linear on $[s,t]$ its solution is  $h_t=\exp(\sum_{i=1}^d(X^i_t-X^i_s)M_i)h_s$; for general $X$ see Appendix~\ref{app:cde}. With learnable matrices, an initial state computed from the first observation, and a linear readout, we have a linear neural CDE \citep{kidger2020neuralcde, cirone2024deepSSM}, and with structured matrices a SLiCE \citep{walker2025structuredlinearcdesmaximally}, for which the states at observation times are computed by a parallel scan over the flows.

An SSM maps an input signal $x$ to a state $h \in \mathbb{R}^m$ and an output $y$ through $\dot{h}(t) = -Ah(t) + Bx(t)$ with $y(t) = Ch(t)$,
where we write $-A$ rather than $A$ to match the EWS convention in Section~\ref{sec:ews}. On a grid it is computed by $h_j = \bar{A}_jh_{j-1} + \bar{B}_jx_j$, where $\bar{A}_j$ and $\bar{B}_j$ are determined by $A$, $B$, a step size $\Delta_j$ and the discretisation rule \citep{gu2021efficiently, gu2024mamba}. In LTI SSMs the step size and matrices are fixed, so $\bar{A}_j$ and $\bar{B}_j$ do not depend on $j$. With $X_t = \int^t_0 x_s ds$, the SSM  is the affine CDE $dh_t = -A h_t dt + BdX_t$, driven by time with an additive input; appending a constant coordinate to the state puts it in the form of a linear CDE (Appendix~\ref{app:cde});  Appendix~\ref{app:ssm} specifies the models considered.

\section{The Exponentially Weighted Signature}\label{sec:ews}

\begin{definition}[Clock]\label{def:clock}
A clock is a causal map $\theta:\mathcal{V}^1([0,T];\mathbb{R}^d) \to \mathcal{V}^1([0,T]; \mathbb{R})$, that is, $\theta(X)_t = \theta(X_{\cdot \wedge t})_t$ for all $X \in \mathcal{V}^1([0,T];\mathbb{R}^d)$ and $t \in [0,T]$, where $X_{\cdot \wedge t}$ is the path stopped at $t$.
\end{definition}

The time clock $\theta(X)_t = t$ and the variation clock $\theta(X)_t = \|X\|_{1, [0,t]}$ are both examples. When the path is obvious, we write $\theta_t$ for $\theta(X)_t$. 

\begin{definition}[EWS]\label{def:ews}
    Consider a matrix pair $\mathbb{A}=(A,B)$, with terms identified as the \emph{generator} $A\in \mathbb{R}^{m\times m}$ and \emph{lift} $B\in \mathbb{R}^{m\times d}$. The exponentially weighted signature (EWS) of $X$ over $[s,t] \subseteq [0,T]$ with clock $\theta$ is $S_{\mathbb{A}}(X;\theta)_{s,t} = \big(1, S_{\mathbb{A}}(X;\theta)^{(1)}_{s,t}, S_{\mathbb{A}}(X;\theta)^{(2)}_{s,t}, \dots \big) \in T((\mathbb{R}^m ))$ where
    \begin{align}
        S_{\mathbb{A}}(X;\theta)^{(n)}_{s,t}&=  \int_{s < t_1 < \cdots < t_n<t} \big(e^{-(\theta_t - \theta_{t_1})A}BdX_{t_1} \big) \otimes \cdots \otimes \big(e^{-(\theta_t - \theta_{t_n})A}BdX_{t_n}  \big) .\label{eq:ews-n}
    \end{align}
\end{definition}

Each increment of the lifted path $BX$ is weighted by a matrix exponential of the generator over the clock time elapsed until $t$. Truncation at depth $N$ keeps the levels $n\leq N$, denoted $S_{\mathbb{A}}(X;\theta)^{\leq N}_{s,t}\in T^N(\mathbb{R}^m)$; Appendix~\ref{app:ews-coordinates} writes the EWS in coordinates, gives a recursive form and works through an explicit example at depth two. When $A=0$ and $B$ is the identity, the EWS is the signature $S(X)$. When $A$ is a diagonal matrix with positive entries, $B$ is the identity and $\theta$ is the time clock, it is the exponentially fading memory signature (EFM) of \citet{jaber2025exponentiallyfadingmemorysignature} (Appendix~\ref{app:efm-volterra}). The EWS therefore extends the EFM through its lift, its clock, its generator and the fact that its parameters are learnable; Section~\ref{sec:ews-depth-one} compares the two through the spectrum of the generator.

For $r \in \mathbb{R}$, the flow operator $D^r_A \colon T((\mathbb{R}^m)) \to T((\mathbb{R}^m))$ acts by $D^r_A(v_1 \otimes \cdots \otimes v_n) = (e^{-rA}v_1) \otimes \cdots \otimes (e^{-rA}v_n)$, and the derivation induced by $A$, $\Lambda_A = -\frac{d}{dr}\bigl|_{r=0}D^r_A \colon T((\mathbb{R}^m)) \to T((\mathbb{R}^m))$, acts by $\Lambda_A(v_1 \otimes \cdots \otimes v_n) = \sum_{j=1}^n v_1 \otimes \cdots \otimes Av_j \otimes \cdots \otimes v_n$. Both are extended linearly, and Appendix~\ref{app:operators} discusses their algebraic properties.
\begin{proposition}[Linear CDE form -- Appendix~\ref{app:ews-cde}]\label{prop:cde}
    The EWS is the unique solution to the linear controlled differential equation 
    \begin{equation}\label{eq:ews-cde}
        dS_{\mathbb{A}}(X;\theta)_{s,t} = -\Lambda_AS_{\mathbb{A}}(X;\theta)_{s,t} d\theta_t + S_{\mathbb{A}}(X;\theta)_{s,t} \otimes BdX_t, \qquad S_{\mathbb{A}}(X;\theta)_{s,s} = 1. 
    \end{equation}
\end{proposition}

Without the first term, (\ref{eq:ews-cde}) is the CDE satisfied by the signature of the lifted path $BX$ (Appendix~\ref{app:signature}). The first term is an additional vector field built from the generator and integrated against the clock, which applies the exponential weighting. Since $\Lambda_A$ is grade-preserving and $BdX_t$ lies at depth one, depth $n \geq 1$ evolves as $dS_{\mathbb{A}}(X;\theta)^{(n)}_{s,t} = -\Lambda_AS_{\mathbb{A}}(X;\theta)^{(n)}_{s,t} d\theta_t + S_{\mathbb{A}}(X;\theta)^{(n-1)}_{s,t} \otimes BdX_t$, where the first term weights depth $n$ in place and the second adds depth $n-1$ tensored with the lifted increment. Identifying $T^N(\mathbb{R}^m)$ with $\mathbb{R}^K$ for $K = \sum_{n=0}^N m^n$ turns the equation for the truncated EWS into a linear CDE driven by $(\theta, X)$, with matrices built from $A$ and $B$. These structured matrices make the EWS a SLiCE, and Appendix~\ref{app:ews-cde} constructs them explicitly.
  \begin{theorem}[{Structure -- Appendices~\ref{app:ews-reweighted}, \ref{app:ews-algebraic} and~\ref{app:ews-uniqueness}}]\label{thm:structure}
Let $\theta$ be a clock and $0 \leq s \leq t \leq T$.
\begin{enumerate}[leftmargin=2em,label=\textup{(\roman*)}]
    \item  The EWS over the interval $[s,t]$ is the signature of a re-weighted path, $S_{\mathbb{A}}(X;\theta)_{s,t} = S(Z^{[t]})_{s,t}$ where $Z^{[t]}_{\tau} =\int^{\tau}_s e^{-(\theta_t- \theta_u)A}BdX_u$ for $\tau \in [s,t]$. In particular, the product of a coordinate at depth $n_1$ and a coordinate at depth $n_2$ is a linear combination of coordinates at depth $n_1 +n_2$.
    \item For every $u \in [s,t]$, the EWS satisfies the modified Chen identity
    \begin{equation}\label{eq:ews-chen}
        S_{\mathbb{A}}(X;\theta)_{s,t} = D^{\theta_t - \theta_u}_{A} \big(S_{\mathbb{A}}(X;\theta)_{s,u} \big) \otimes S_{\mathbb{A}}(X;\theta)_{u,t}.
    \end{equation}
    \item For a continuous clock and a compact set of paths, linear functionals of $S_{\mathbb{A}}(X;\theta)_{0,T}$ are dense in the continuous functions that take equal values on paths with the same EWS.
\end{enumerate}
\end{theorem}
Since the path in (i) depends on the interval, (ii) does not follow from the classical Chen identity. However, on a fixed interval, (i) makes the EWS a signature, hence group-like \citep{Chen1957IntegrationOP, reutenauer1993free}. By (iii), the EWS is universal on paths up to the equivalence it induces (similarly to the signature), and hence on all paths whenever it separates them. For instance, it separates time-augmented paths with a common start when the lift and generator satisfy mild structural conditions (Corollary~\ref{cor:injective}).

If $\theta(X \circ \varphi) = \theta(X) \circ \varphi$ for every increasing homeomorphism $\varphi$ of $[0,T]$, as holds for the variation clock but not the time clock, then $S_{\mathbb{A}}(X \circ \varphi; \theta)_{s,t} = S_{\mathbb{A}}(X;\theta)_{\varphi(s), \varphi(t)}$, and hence the EWS is invariant under reparametrisation of $X$ (Proposition~\ref{prop:reparam}). An LTI SSM lacks this invariance, and the signature has it but no fading memory, whereas the EWS with the variation clock can have both.

\subsection{Spectral properties at depth one} \label{sec:ews-depth-one}

At depth one, for a given clock, (\ref{eq:ews-n}) is a linear functional of the increments of the lifted path, whose weights $e^{-rA}$ depend on the spectrum of $A$ and elapsed clock time $r$:
\begin{equation}\label{eq:level-one}
    S_{\mathbb{A}}(X; \theta)^{(1)}_{s,t} = \int^t_s e^{-(\theta_t - \theta_u)A} BdX_u \in \mathbb{R}^m.
\end{equation}
If $A=\mathrm{diag}(\lambda_1, \dots, \lambda_m)$ with real entries, the $k$-th coordinate of (\ref{eq:level-one}) is $\int^t_s e^{-\lambda_k(\theta_t - \theta_u)}(BdX_u)^k$, a finite-window Laplace transform in clock time, and depth one is then a bank of $m$ Laplace modes, each lifted channel forgotten at its own monotone rate. 

If instead $A = P \Gamma P^{-1}$ is diagonalisable over $\mathbb{C}$ with complex eigenvalues, it applies a diagonal bank to the transformed path and mixes the result via $P$. A conjugate pair $a \pm i\omega$, carried by the real block $\left(\begin{smallmatrix}
    a&-\omega\\\omega&a
\end{smallmatrix} \right)$, contributes the damped oscillatory kernels $e^{-ar}\cos(\omega r)$ and $e^{-ar}\sin(\omega r)$. Depth one thus becomes a bank of damped Fourier modes, their number fixed by the lift; that is, a learned STFT \citep{mallat2009wavelet} whose window is the exponential memory set by $A$. Deeper levels combine these filtered components across time and channels. A real diagonal generator can only approximate these kernels with much larger width (Appendix~\ref{app:ews-spectral}). A diagonalisable generator gives weights that are combinations of pure exponentials and damped oscillations. A polynomial factor must instead come from a Jordan block. 
\begin{proposition}[Polynomial memory -- Appendix~\ref{app:ews-spectral}]\label{prop:jordan}
    Let $\theta_t=t$, let $X^i$ be a channel of $X$ and $\lambda \geq 0$. If $A$ has a Jordan block of size $K+1$, with eigenvalue $\lambda$, and $B$ maps $X^i$ into the last coordinate of that block and no other channel into the block, then $\int_0^t \frac{(t-u)^{\ell}}{\ell !}e^{-\lambda(t-u)}dX^i_u$ is, up to sign, a depth-one coordinate for every $\ell \leq K$. For $\ell \geq 1$, it is not a linear readout of depth one for any diagonalisable $A$, of any size.
\end{proposition}
Up to normalisation the kernels in Proposition~\ref{prop:jordan} are the gamma kernels of the gamma memory of \citet{VRIES1992565}. When $\lambda =0$ each of these functionals is a linear functional of the signature of $(t,X)$ requiring at least depth $\ell +1$, whereas the EWS gives it at depth one. For $\lambda >0$, no finite depth of the signature suffices, and for any $\lambda$, linear functionals of the EFM only approximate it.

\subsection{The EWS Layer}\label{sec:ews-layer}
Up to this point, the generator, lift, and clock are fixed. An EWS layer trains them and applies the transform to a time series, with $A$ and the clock drawn from Table~\ref{tab:families} in our experiments, and $B$ free. Writing $X\in\mathcal{V}^1([0,T];\mathbb{R}^d)$ for a path built from observations, the layer is the map $X|_{[0,t]}\mapsto S_{\mathbb{A}}(X;\theta)^{\leq N}_{0,t}$, defined for every $t\in[0,T]$. Read at the observation times through a linear readout,
\begin{equation}\label{eq:ews-layer}
X|_{[0,t]}\;\xrightarrow{\;S_{\mathbb{A}}(\,\cdot\,;\theta)\;}\;S_{\mathbb{A}}(X;\theta)^{\leq N}_{0,t}\;\xrightarrow{\;\text{linear readout}\;}\;y_t ,
\end{equation}
it is a sequence-to-sequence model, and by Theorem~\ref{thm:structure}(iii) such models are maximally expressive in the sense of \citet{walker2025structuredlinearcdesmaximally} as truncation depth grows. 
\begin{table}[th]
\centering
\small
\caption{\small
Some generator families and clocks. A clock is: invariant (i) when the
layer does not depend on the path's parametrisation;
stopping (s) when $f$ can vanish on non-zero increments.
}
\label{tab:families}
\scriptsize
\setlength{\tabcolsep}{2.4pt}
\renewcommand{\arraystretch}{1.08}
\begin{minipage}[t]{0.55\linewidth}
\centering
\begin{tabular}{@{}lll@{}}
\toprule
Family & Spectrum & Kernels \\
\midrule
Zero & $\{0\}$ & $1$ \\
Diagonal & real & $e^{-\lambda r}$ \\
Rotation & $a\pm i\omega$ & $e^{-ar}\cos(\omega r),\ e^{-ar}\sin(\omega r)$ \\
Dense & general & mixtures \\
Jordan, real & $\lambda$, repeated & $r^\ell e^{-\lambda r}$ \\
Jordan, complex & $a\pm i\omega$, repeated & $r^\ell e^{-ar}\cos(\omega r),\ r^\ell e^{-ar}\sin(\omega r)$ \\
\bottomrule
\end{tabular}
\end{minipage}
\hfill
\begin{minipage}[t]{0.43\linewidth}
\centering
\begin{tabular}{@{}llcc@{}}
\toprule
Clock & $f(v)$ & i & s \\
\midrule
Time & $\Delta t$ & & \\
Variation & $\|v\|$ & \checkmark & \\
Blended & $\alpha\Delta t+(1-\alpha)\|v\|$ & & \\
Gated & $\mathrm{ReLU}(w\cdot v-b)$ & $b{=}0$ & \checkmark \\
Clamped & $\frac{1}{c}\sum_{k=1}^{K}1\wedge\mathrm{ReLU}(w\cdot v-b_k)$ & & \checkmark \\
\bottomrule
\end{tabular}
\end{minipage}
\end{table}
The clocks used are built from path increments. Fix $f\colon\mathbb{R}^d\to\mathbb{R}$ and set $\theta_{t_j}=\sum_{t_i\leq t_j}f(\Delta X_i)$, interpolated linearly between observation times, with choices of $f$ in Table~\ref{tab:families}.  Here $\Delta t$ is the time coordinate of $v$, the increment of the time channel of $X$, or one when $X$ has none. Taking $f$ non-negative makes the clock non-decreasing, so that increments are damped as clock time passes. When $f$ vanishes the clock stops and nothing decays in the state; a run of uninformative observations thus costs the model none of its memory. When $f(cv)=cf(v)$ for $c>0$, the clock is unchanged by reparametrisation of the path (Proposition~\ref{prop:increment-clock}), and thus the layer is invariant by Proposition~\ref{prop:reparam}. This is a choice: invariance suits data whose observation times carry no information; a clock without it may be better when the timing is itself a signal.

The EWS layer can also be incorporated into an EWS block, similarly to how SSM layers are used in deep architectures \citep{gu2021efficiently, rusch2025oscillatory, terzić2026flashpdssmmemoryoptimizedstructured}. Each block applies an EWS layer, whose state is normalised at each depth separately and read down to the stream width by a linear map, followed by a non-linear transformation (Appendix~\ref{app:ews-block}). After several blocks, the sequence is decoded to produce the output.

Proposition~\ref{prop:cde} makes the EWS layer a SLiCE with vector fields generated by $A$ and $B$ and driven by the clock. Three properties follow from this structure rather than from training: \emph{(a) Parameter count.} An EWS layer has $\mathcal{O}(m^2+dm)$ generator and lift parameters at any depth, and $\mathcal{O}(m+dm)$ for the diagonal, rotation and Jordan families of Table~\ref{tab:families}. Among the maximally expressive structures in \citet{walker2025structuredlinearcdesmaximally}, the cheapest has $\mathcal{O}(dK)$ parameters at state size $K = \sum_{n \leq N}m^n$, whereas the EWS generates all $K$ coordinates from $A$ and $B$; readout parameters are linear in $K$ for either model. 
\emph{(b) Closure.} By Theorem~\ref{thm:structure}(i) every degree $N$ polynomial in depth-one coordinates is linear in the EWS truncated at depth $N$,  whereas the coordinates of a free linear CDE carry no such relation. 
\emph{(c) Scale.} When the generator is positive definite, and the clock dominates the path (e.g.~variation clock), depth $n$ of the EWS is bounded by $C^n/n!$ on every interval, with $C$ independent of the length of the path (Corollary~\ref{cor:scale}), whereas a free linear CDE can grow exponentially in the path's total variation. Together with the meaning of its coordinates, the spectral filters of Section~\ref{sec:ews-depth-one} and their products, these properties are the inductive bias of the layer.

 The layer is computed as a parallel associative scan over the output grid. Each interval $[t_j, t_{j+1})$ contributes the pair $(\theta_{t_{j+1}} - \theta_{t_j}$, $S_{\mathbb{A}}(X;\theta)_{t_j, t_{j+1}})$ of its clock increment and EWS, and Theorem~\ref{thm:structure}(ii) makes $(r_1, S_1) \star(r_2, S_2) = \big(r_1+r_2, \, D^{r_2}_A(S_1) \otimes S_2\big)$ associative, so the output state is obtained by scanning $\star$ over these pairs. The EWS over each interval is the signature of some re-weighted path (Theorem~\ref{thm:structure}(i)), which can be computed exactly at any depth. The path $X_t = \int_0^t x_sds$ is taken linear between observations, which is the continuous-time counterpart of the zero-order hold used in state-space models \citep{gu2024mamba}, and the layer is causal, as the state depends only on prior observations.

 \section{Experiments against the signature \& EFM}\label{sec:main-synthetic}

Before comparing the EWS with SSMs, we verify that the properties established in Section~\ref{sec:ews} are achieved on synthetic tasks, each of which has the same form: a path $X$ is generated, a target path $Y_t = F(X)_t$ is computed, and a model is trained to regress $Y_t$ from observations of $X|_{[0,t]}$. Since $F$ is known exactly, Section~\ref{sec:ews} tells us which generators and truncation depths represent it.

We compare the EWS with the simpler signature and EFM. All three are the layer of (\ref{eq:ews-layer}) with the time clock and linear readout. Within a task, they differ only in the generator structure, which is zero for the signature, positive diagonal for EFM (with learned parameters rather than fixed ones as in \citet{jaber2025exponentiallyfadingmemorysignature}), and from the rotation family or a free matrix for the EWS. Each model $(A,B)$ is identically trained over the same grid of truncation depths and widths; we report test error of the best validation at each depth. 

We consider two tasks, with different $F$. Full details, and a higher-order expressivity task, are given in Appendix~\ref{app:synthetic}. The first is oscillatory, and has $F$ be one component of an SDE driven by $X_t=(t, W_t)$ for $W$ a 2D Brownian path. The second tests whether the complex spectrum and deeper levels of the EWS are necessary; its path is $X_t =(t, \int^t_0 x_s ds)$, where $x$ is a sum of one or two persistent low-frequency tones and up to two high-frequency tones that switch on in bursts. We estimate three targets: 
(i) a causal spectrogram with frequencies $\{\omega_j\}_{j=1}^4$ and rate $a$, $F_1^j(t) = (\int_0^t e^{-a(t-u)}\cos(\omega_j(t-u))x_udu, \int_0^t e^{-a(t-u)}\sin(\omega_j(t-u))x_udu )$. By Section~\ref{sec:ews-depth-one}, $F_1$ is a linear readout of depth one of an EWS, but cannot be given by a linear readout at depth one with a real diagonal generator.
 (ii) $F_2(t) =  \sum_{j\ \mathrm{high}}\|F_1^j(t)\|^2$, so $F_2$ is the energy in the high band. This is a quadratic in the depth-one coordinates so, by Theorem~\ref{thm:structure}, a linear combination of level-two coordinates; no linear readout of depth one works. 
 (iii) $F_3(t)$ is binary, equal to one if a high-frequency tone is on at time $t$. 

On the oscillator, the EWS  is $4.6 \times$ and $6.3 \times$ better than the EFM and signature respectively, and its learned generator carries a complex pair on every seed, which the EFM cannot represent. On $F_1$ the EWS reaches the target at depth one, whereas the EFM recovers the low band but not the high one at any depth or width. On $F_2$ and $F_3$ the EWS fails at depth one at every width and succeeds at depth two, while the EFM and signature fail at every depth. In addition, thresholding the energy $F_2$ recovers $F_3$ with only $81\%$ accuracy. Linear readouts therefore need depth on these targets.

\begin{table}[th]
\centering
\caption{\small Synthetic targets evaluated at depths $n=1$ and $n=2$, mean $\pm$ standard deviation at the best validation width. Oscillator task reports test RMSE on targets scaled to $[0,1]$ over ten seeds, $F_1$ and $F_2$ normalised test MSE, and $F_3$ test accuracy over three seeds (majority-class accuracy $69.5\%$). }
\label{tab:synthetic}
\scriptsize
\setlength{\tabcolsep}{2.4pt}
\renewcommand{\arraystretch}{1.08}
\begin{tabular}{@{}lcccccc@{}}
\toprule
& \multicolumn{2}{c}{Signature}
& \multicolumn{2}{c}{EFM}
& \multicolumn{2}{c}{EWS} \\
\cmidrule(lr){2-3} \cmidrule(lr){4-5} \cmidrule(lr){6-7}
Target & $n{=}1$ & $n{=}2$ & $n{=}1$ & $n{=}2$ & $n{=}1$ & $n{=}2$ \\
\midrule
Oscillator
& $0.135{\pm}0.005$
& $0.127{\pm}0.005$
& $0.0922{\pm}0.0034$
& $0.0936{\pm}0.0046$
& $\mathbf{0.0201{\pm}0.0014}$
& $0.0257{\pm}0.0026$ \\
\midrule
$F_1$ coeff.
& $0.940{\pm}0.000$
& $0.881{\pm}0.000$
& $0.512{\pm}0.002$
& $0.448{\pm}0.004$
& $\mathbf{(3.5{\pm}1.2)\times10^{-5}}$
& $(6.6{\pm}3.7)\times10^{-5}$ \\
$F_2$ high-band
& $0.990{\pm}0.001$
& $0.906{\pm}0.001$
& $0.908{\pm}0.001$
& $0.909{\pm}0.004$
& $0.921{\pm}0.022$
& $\mathbf{(7.3{\pm}3.4)\times10^{-4}}$ \\
$F_3$ burst (\%)
& $69.5{\pm}0.0$
& $69.2{\pm}0.0$
& $68.7{\pm}0.0$
& $69.0{\pm}0.0$
& $69.5{\pm}0.0$
& $\mathbf{94.7{\pm}0.2}$ \\
\bottomrule
\end{tabular}
\end{table}

\section{SSMs are depth-one EWS}\label{sec:ssm}
We now demonstrate a close connection between SSMs and the depth-one terms in the EWS. Consider first the continuous-time dynamics of an LTI SSM, $\dot{h}_t = -Ah_t +Bx_t$ with $h_0 = 0$ and $y_t = Ch_t$. Let $X_t = \int^t_0 x_sds$. The solution $h_t = \int^t_0e^{-(t-s)A}Bx_sds=S_{\mathbb{A}}(X;\theta)^{(1)}_{0,t}$ is the depth-one component of the EWS of $X$ with the time clock. Thus, the dynamics underlying the SSM are those of a depth-one EWS with generator $A$ and lift $B$. At depth one, the modified Chen identity of Theorem~\ref{thm:structure}(ii) reads
\begin{equation}\label{eq:chen-depth-one}
   S_{\mathbb{A}}(X;\theta)^{(1)}_{0,t} =  e^{-(t-s)A}S_{\mathbb{A}}(X;\theta)^{(1)}_{0,s} + S_{\mathbb{A}}(X;\theta)^{(1)}_{s,t},\qquad 0\le s\le t.
\end{equation}
On a grid $t_j =j\Delta$, with $X$ linear between observations, writing $G_{\Delta}(A) = \int^{\Delta}_0 e^{-sA}ds$ and  $S_j = S_{\mathbb{A}}(X;\theta)^{(1)}_{0,t_j}$, we see that (\ref{eq:chen-depth-one}) becomes
\begin{equation}\label{eq:level-one-recurrence}
    S_j = e^{-\Delta A}S_{j-1} + G_{\Delta}(A)Bx_j.
\end{equation}
This is the same form as the recurrence $h_j = \bar{A}h_{j-1} + \bar{B}x_j$ that SSMs compute, with $\bar{A} = e^{-\Delta A}$ and $\bar{B} = G_{\Delta}(A)B$, the zero-order hold. Matching coefficients in the other direction, a trained LTI SSM storing $(\bar{A}, \bar{B})$ and step $\Delta$ is reproduced at every observation by an EWS with  
\begin{equation}\label{eq:conversion}
    A = -\tfrac{1}{\Delta}\log(\bar{A}), \qquad B=G_{\Delta}(A)^{-1}\bar{B},
\end{equation}
whenever $\bar A$ has a real matrix logarithm. This holds for every model in Table~\ref{tab:conversion}, except some D-LinOSS blocks (Appendix~\ref{app:ssm}). Appendix~\ref{app:ssm-numerical} verifies this numerically, with random and trained parameters.
Unlike an SSM, the EWS never discretises its dynamics at depth one. The input $x$ is observed at discrete times, yet the second term of (\ref{eq:chen-depth-one}) requires it on the whole interval. The EWS holds $x$ constant between observations, the rectilinear embedding  \citep{morrill2022on, walker2026faithfulembeddingsirregularasynchronous}, so that $X$ is piecewise linear. The zero-order hold of S4D and S5 yields the same formula, whereas every other rule used by popular SSMs approximates it to 2nd order for S4 and  Mamba-3 and to 1st order for LinOSS, Mamba and Mamba-$2$ \citep{lahoti2026mamba}.  We note that Theorem~\ref{thm:structure}(i) demonstrates that deeper levels of the EWS are not the signature of the latent state of the SSM $t \mapsto h_t$.
\begin{table}[th]
\centering
\caption{\small State-space models as depth-one EWS, with the block and spectrum of the EWS generator reproducing each trained model (unit clock for LTI models, accumulated step for Mamba). The $\mu_k$ are the eigenvalues of $\Lambda-PQ^*$, and $\gamma_k$, $\omega_k$, $\psi_k$, $\varphi_k$ are defined in Appendix~\ref{app:ssm}.}
\label{tab:conversion}
\scriptsize
\setlength{\tabcolsep}{2.2pt}
\renewcommand{\arraystretch}{1.08}
\begin{tabular}{@{}llllll@{}}
\toprule
Model & Stores & Discretisation & Block & Spectrum of $A$ & Route \\
\midrule
S4D \citep{NEURIPS2022_e9a32fad} & $a_k$, step $\Delta_c$ & ZOH & $2\times2$ rotation & $\Delta_c(\gamma_k\pm i\omega_k)$ & read-off\\
S5 \citep{smith2023simplified} & $\Lambda_k$, step $\Delta t_k$ & ZOH & $2\times2$ rotation & $\Delta t_k(\gamma_k\pm i\omega_k)$ & read-off\\
S4 \citep{gu2021efficiently} & $\Lambda-PQ^*$, step $\Delta_c$ & bilinear & dense, diagonalisable & $-2\,\mathrm{artanh}(\Delta_c\mu_k/2)$ & (\ref{eq:conversion})\\
LRU \citep{pmlr-v202-orvieto23a} & $e^{-\nu_k+i\vartheta_k}$ & none & $2\times2$ rotation & $\nu_k\pm i\vartheta_k$ & (\ref{eq:conversion})\\
LinOSS-IM \citep{rusch2025oscillatory} & $\mathsf{A}_k\ge0$, step $\Delta t_k$ & implicit & $2\times2$ oscillator & $\tfrac12\log(1+\Delta t_k^2\mathsf{A}_k)\pm i\psi_k$ & (\ref{eq:conversion})\\
D-LinOSS \citep{boyer2025learningdissipateenergyoscillatory} & $\mathsf{A}_k,\mathsf{G}_k$, step $\Delta t_k$ & IMEX & $2\times2$ oscillator & $\tfrac12\log(1+\Delta t_k\mathsf{G}_k)\pm i\varphi_k$ & (\ref{eq:conversion})\\
\midrule
Mamba-1 \citep{gu2024mamba} & $a_k$, step $\Delta_t$ & exp.\ Euler & scalar & $a_k$, per channel & clocked\\
Mamba-2 \citep{dao2024transformers} & $a$, step $\Delta_t$ & exp.\ Euler & scalar & $a$, per head & clocked\\
\bottomrule
\end{tabular}
\end{table}
A selective SSM computes its step $\Delta_t$ from its input. The accumulated step $\theta_t = \sum_{s \leq t}\Delta_s$ is  thus a clock as in Definition~\ref{def:clock}. The selective transition $e^{-\Delta_tA}$ can be rewritten as $e^{-(\theta_t-\theta_{t-1})A}$, and so (\ref{eq:chen-depth-one}) applies, with $\theta_t-\theta_s$ instead of $t-s$. With this clock, the state of each channel of Mamba-1 is a depth-one EWS with a real diagonal generator, driven by a path with increments $G_{\Delta_t}(A)^{-1}\Delta_t^2B_tx_t$. The state of each head of Mamba-2 is the same construction, with generator $aI$ and increments $G_{\Delta_t}(a)^{-1} \Delta^2_tB_tx_t^{\top}$. A selective layer is thus a family of depth-one EWS layers, one per channel or head, over paths built by its selection mechanism. 

The state of every model in Table~\ref{tab:conversion} is contained in the EWS state, which extends via depth and flexibility in the  clock and generator. Deeper levels hold weighted products of increments at different times, which SSMs reach only by stacking layers with non-linearities between them \citep{pmlr-v235-orvieto24a, NEURIPS2023_ea8608c6}. The clock may be any causal functional, including path variation. The generator may be any matrix; up to a change of basis absorbed by the lift and readout, the diagonal and rotation families of Table~\ref{tab:families} cover every model in Table~\ref{tab:conversion} and the block of Mamba-3, while its Jordan blocks, and the gamma memories they carry (Proposition~\ref{prop:jordan}), are more general (reached by SSMs only in the non-learnable HiPPO-LagT matrix \citep{NEURIPS2020_102f0bb6} and edge cases of LinOSS and D-LinOSS, Appendix~\ref{app:ssm}). Retention \citep{sun2023retentivenetworksuccessortransformer} and the Mamba-$2$ SSD layer are also depth-one EWS whose path is a sequence of key--value outer products; a rotary embedding \citep{10.1016/j.neucom.2023.127063} is a rotation block with angle set by the clock.

\section{Experiments}\label{sec:experiments}

\subsection{Long time-series classification}\label{sec:uea}

We consider the long-range benchmark of \citet{Walker2024LogNCDE}, six datasets from the UEA multivariate time-series  classification archive \citep{bagnall2018ueamultivariatetimeseries, middlehurst2026multiversetimeseriesmachine} with sequence lengths from $405$ to $17\,984$ observations. We follow the training procedure of \citet{Walker2024LogNCDE}, with the per-dataset selection of Appendix~\ref{app:uea}. The SSM baselines are stacks of two to six blocks with non-linearities between them. The EWS is a single layer with linear readout (\ref{eq:ews-layer}), so that its only non-linearity is from depth and the clock. In Table \ref{tab:uea}, the EWS performs best on accuracy and rank,  while using on average $3.8k$ parameters per dataset, fewer than any baseline with a known count. The largest margin is EthanolConcentration, where the selected configuration is a rotation generator with the time clock and depth two. Excluding this dataset, the EWS keeps the best average rank and is within $0.1$ points of the best average accuracy. Validation selects a truncation depth of two or three on every dataset, which raises the selected configuration's test accuracy over depth one on five of the six datasets, including by eight points on EigenWorms, and the ablation of Appendix~\ref{app:uea} largely supports the benefits of depth. Training time per step is within the range of baselines on five datasets (Appendix~\ref{app:uea}).
\begin{table}[th]
\centering
\caption{\small
Test accuracy (\%) on six UEA datasets, mean $\pm$ standard deviation over five seeds, with the average and average
rank. Baselines taken from
\citet{Walker2024LogNCDE, walker2025structuredlinearcdesmaximally, boyer2025learningdissipateenergyoscillatory, terzić2026flashpdssmmemoryoptimizedstructured}.
}
\label{tab:uea}

\scriptsize
\setlength{\tabcolsep}{1.8pt}
\renewcommand{\arraystretch}{1.08}

\begin{tabular}{@{}lcccccc|cc@{}}
\toprule
Model
& EigenWorms
& Ethanol
& Heartbeat
& Motor
& SCP1
& SCP2
& Avg.
& Rank \\
\midrule

Log-NCDE
& $85.6{\pm}5.1$
& $34.4{\pm}6.4$
& $75.2{\pm}4.6$
& $53.7{\pm}5.3$
& $83.1{\pm}2.8$
& $53.7{\pm}4.1$
& $64.3$
& $7.5$ \\

S5
& $81.1{\pm}3.7$
& $24.1{\pm}4.3$
& $77.7{\pm}5.5$
& $47.7{\pm}5.5$
& $\mathbf{89.9{\pm}4.6}$
& $50.5{\pm}2.6$
& $61.8$
& $7.83$ \\

S6
& $85.0{\pm}16.1$
& $26.4{\pm}6.4$
& $76.5{\pm}8.3$
& $51.3{\pm}4.7$
& $82.8{\pm}2.7$
& $49.9{\pm}9.4$
& $62.0$
& $9.17$ \\

Mamba
& $70.9{\pm}15.8$
& $27.9{\pm}4.5$
& $76.2{\pm}3.8$
& $47.7{\pm}4.5$
& $80.7{\pm}1.4$
& $48.2{\pm}3.9$
& $58.6$
& $10.33$ \\

DE-SLiCE
& $88{\pm}4$
& $26{\pm}5$
& $74{\pm}5$
& $50{\pm}6$
& $82{\pm}4$
& $50{\pm}3$
& $61.6$
& $9.67$ \\

D-DE-SLiCE
& $86{\pm}6$
& $27{\pm}7$
& $74{\pm}4$
& $55{\pm}4$
& $85{\pm}4$
& $52{\pm}5$
& $63.0$
& $7.83$ \\

BD-SLiCE
& $86{\pm}4$
& $29{\pm}7$
& $77{\pm}6$
& $53{\pm}3$
& $85{\pm}2$
& $54{\pm}8$
& $64.0$
& $6.50$ \\

LinOSS-IM
& $\mathbf{95.0{\pm}4.4}$
& $29.9{\pm}0.6$
& $75.8{\pm}3.7$
& $60.0{\pm}7.5$
& $87.8{\pm}2.6$
& $58.2{\pm}6.9$
& $67.8$
& $3.67$ \\

D-LinOSS
& $93.9{\pm}3.2$
& $29.9{\pm}0.6$
& $75.8{\pm}4.9$
& $\mathbf{61.1{\pm}2.0}$
& $88.9{\pm}3.0$
& $\mathbf{58.6{\pm}2.3}$
& $68.0$
& $3.17$ \\

PD-SSM
& $90.0{\pm}5.7$
& $34.7{\pm}4.0$
& $\mathbf{80.0{\pm}2.6}$
& $60.0{\pm}3.7$
& $80.9{\pm}2.0$
& $56.1{\pm}8.6$
& $67.0$
& $4.50$ \\

FLASH PD-SSM
& $93.8{\pm}1.4$
& $43.7{\pm}2.9$
& $77.2{\pm}6.6$
& $57.9{\pm}6.0$
& $86.1{\pm}2.8$
& $56.1{\pm}3.1$
& $69.1$
& $3.83$ \\

\midrule

EWS
& $93.3{\pm}1.4$
& $\mathbf{71.9{\pm}4.8}$
& $\mathbf{80.0{\pm}4.5}$
& $\mathbf{61.1{\pm}5.0}$
& $87.1{\pm}2.5$
& $56.5{\pm}5.8$
& $\mathbf{75.0}$
& $\mathbf{2.33}$ \\
\bottomrule
\end{tabular}
\end{table}

\subsection{Regression \& forecasting}\label{sec:regression}

We consider a regression and a forecasting task at depth one, where the EWS is a selective SSM on the first (clamped clock) and an LTI SSM on the second (time clock). The first is PPG-DaLiA \citep{s19143079}, in which heart rate is regressed from six channels, in windows of $49\,920$ observations. The second is the weather task of \citet{haoyietal-informer-2021}, in which one month of hourly climatological data ($720$ observations) is predicted from the previous month, framed as a masked sequence-to-sequence task as in \citet{rusch2025oscillatory, boyer2025learningdissipateenergyoscillatory,gu2021efficiently}, with our  hyperparameter search given in Appendix~\ref{app:ppg-weather}. Table~\ref{tab:regression} displays the results. On PPG-DaLiA,  with $13\times$ fewer parameters than D-LinOSS, the EWS has the lowest mean error reported. On weather it also has the lowest reported error, using a single EWS block with a Jordan generator where the next most accurate, LinOSS-IMEX, stacks five. PPG-DaLiA uses four blocks; both configurations, with training time and memory for weather, are in Appendix~\ref{app:ppg-weather}.
\begin{table}[th]
\centering
\caption{\small
Regression and forecasting. Left, PPG-DaLiA test MSE ($\times10^{-2}$),
mean $\pm$ standard deviation over five seeds. Right, weather test
error at a horizon of 720 steps. Baselines are taken from
\citet{Walker2024LogNCDE,rusch2025oscillatory,
boyer2025learningdissipateenergyoscillatory,
terzić2026flashpdssmmemoryoptimizedstructured,
haoyietal-informer-2021,gu2021efficiently}. On weather, EWS and D-LinOSS are five seed mean $\pm$ standard deviation; others are single seed.
}
\label{tab:regression}

\scriptsize
\setlength{\tabcolsep}{2.4pt}
\renewcommand{\arraystretch}{1.08}

\begin{minipage}[t]{0.46\linewidth}
\centering
\begin{tabular}{@{}lcc@{}}
\toprule
\multicolumn{3}{c}{PPG-DaLiA} \\
\midrule
Model & MSE & Params. \\
\midrule
NCDE
& $13.54{\pm}0.69$
& $66$k \\
Mamba
& $10.65{\pm}2.20$
& $797$k \\
Log-NCDE
& $9.56{\pm}0.59$
& $76$k \\
LinOSS-IMEX
& $7.5{\pm}0.46$
& $26$k \\
FLASH PD-SSM
& $6.5{\pm}0.35$
& -- \\
LinOSS-IM
& $6.4{\pm}0.23$
& $29$k \\
D-LinOSS
& $6.16{\pm}0.73$
& $151$k \\
\midrule
EWS
& $\mathbf{5.90{\pm}0.73}$
& $\mathbf{11}$k \\
\bottomrule
\end{tabular}
\end{minipage}
\hfill
\begin{minipage}[t]{0.51\linewidth}
\centering
\begin{tabular}{@{}lcc@{}}
\toprule
\multicolumn{3}{c}{Weather} \\
\midrule
Model & MAE & Params. \\
\midrule
Informer
& $0.731$
& $19.5$M \\
LSTnet
& $0.757$
& -- \\
S4
& $0.578$
& $\approx 100$k \\
LinOSS-IM
& $0.528$
& $136$k \\
LinOSS-IMEX
& $0.508$
& $1.3$M \\
D-LinOSS
& $0.538 \pm 0.052$
& $301$k \\
\midrule
EWS
& $\mathbf{0.473 \pm 0.002}$
& $\mathbf{19}$k \\
\bottomrule
\end{tabular}
\end{minipage}
\end{table}

\subsection{State Tracking}\label{sec:state-tracking}

The formal language benchmark of \citet{deletang2023neural} organises language tasks by the Chomsky hierarchy \citep{1056813}. We use the four regular tasks, which can be solved by sequentially updating a fixed set of internal states without external memory, i.e. by state tracking; the tasks are described in Appendix~\ref{appendix:language}. Recurrent models such as LSTMs \citep{6795963} and xLSTMs \citep{NEURIPS2024_c2ce2f27} perform well, while parallel-in-time models generally struggle \citep{3692070.3693514} unless their transition matrices can have negative eigenvalues \citep{grazzi2025unlocking}. Following \citet{walker2025structuredlinearcdesmaximally}, models are trained on lengths $3$--$40$ and evaluated on $40$--$256$, using two stacked blocks as in \citet{NEURIPS2024_c2ce2f27}. The EWS uses depth one in both blocks, with a gated or clamped clock in the first and variation in the second. Generators come from the rotation family, with some parametrised as skew-symmetric matrices. The lower rows of Table~\ref{tab:formal} ablate these choices by replacing the clocks with the time clock and the generators with real diagonal ones, keeping other hyperparameters fixed; replacing both gives the learned EFM. Among parallel-in-time models, EWS achieves the highest average, solving three tasks and trailing only D-DE-SLiCE on modular arithmetic. Replacing the clocks with the time clock reduces all tasks to the level of the LTI models. When the gate returns zero the state is not rotated, and otherwise the rotation is applied once, so the total rotation depends on which tokens occur rather than on sequence length alone. For parity, a rotation by $\pi$ at each one tracks whether the number of ones is odd, which no function of length can recover. Unlike the rotations of Mamba-$3$, which always come with a decay set by the same step, an EWS rotation block can turn without decaying and keeps this parity over long inputs. With real diagonal generators, whose transitions have positive real eigenvalues, three tasks fail while even pairs is largely unaffected, consistent with \citet{grazzi2025unlocking}.
\begin{table}[th]
\centering
\caption{\small
Formal language tasks. Validation accuracy (\%) on sequence lengths
40--256, over five seeds for EWS and three for its
ablations. Baselines are taken from
\citet{walker2025structuredlinearcdesmaximally,
terzić2026flashpdssmmemoryoptimizedstructured}.
}
\label{tab:formal}

\scriptsize
\setlength{\tabcolsep}{2.4pt}
\renewcommand{\arraystretch}{1.08}

\begin{tabular}{@{}lccccc@{}}
\toprule
Model
& Cycle nav.
& Even pairs
& Mod. arith.
& Parity
& Avg. \\
\midrule

LSTM
& $100.0{\pm}0.0$
& $100.0{\pm}0.0$
& $99.9{\pm}0.1$
& $100.0{\pm}0.0$
& $100.0$ \\

xLSTM[1:1]
& $53.5{\pm}5.6$
& $99.0{\pm}1.9$
& $29.3{\pm}1.6$
& $100.0{\pm}0.0$
& $70.5$ \\

\midrule

Transformer
& $24.4{\pm}0.5$
& $90.4{\pm}10.4$
& $23.6{\pm}0.7$
& $52.2{\pm}0.4$
& $47.7$ \\

S4D
& $23.7{\pm}1.1$
& $68.7{\pm}4.7$
& $21.7{\pm}0.4$
& $51.2{\pm}1.0$
& $41.3$ \\

Mamba
& $48.4{\pm}2.2$
& $\mathbf{100.0{\pm}0.0}$
& $33.1{\pm}6.6$
& $54.2{\pm}2.1$
& $58.9$ \\

Mamba-3
& $69.0{\pm}11.7$
& $\mathbf{100.0{\pm}0.0}$
& $46.8{\pm}5.5$
& $76.5{\pm}7.8$
& $73.0$ \\

DeltaNet$[-1,1]$
& $46.7{\pm}6.1$
& $\mathbf{100.0{\pm}0.0}$
& $66.4{\pm}8.8$
& $97.7{\pm}2.0$
& $77.7$ \\

Gated DeltaProduct$[-1,1]$
& $46.3{\pm}6.6$
& $\mathbf{100.0{\pm}0.0}$
& $78.4{\pm}10.9$
& $98.0{\pm}1.4$
& $80.7$ \\

BD-SLiCE
& $99.8{\pm}0.2$
& $85.9{\pm}11.3$
& $54.0{\pm}12.5$
& $95.3{\pm}3.9$
& $83.8$ \\

D-DE-SLiCE
& $73.3{\pm}29.4$
& $84.8{\pm}8.5$
& $\mathbf{98.4{\pm}0.7}$
& $83.8{\pm}11.3$
& $85.1$ \\

PD-SSM
& $99.5{\pm}0.7$
& $99.7{\pm}0.3$
& $96.2{\pm}3.4$
& $99.9{\pm}0.1$
& $98.8$ \\

FLASH PD-SSM
& $92.7{\pm}1.9$
& $98.8{\pm}0.8$
& $94.7{\pm}2.9$
& $99.0{\pm}1.5$
& $96.3$ \\

EWS
& $\mathbf{99.99{\pm}0.01}$
& $\mathbf{100.0{\pm}0.0}$
& $96.7{\pm}0.8$
& $\mathbf{99.99{\pm}0.01}$
& $\mathbf{99.2}$ \\

\midrule

EWS, time clock
& $28.9{\pm}1.6$
& $58.1{\pm}1.3$
& $20.7{\pm}0.4$
& $53.1{\pm}0.8$
& $40.2$ \\

EWS, real diagonal generator
& $52.1{\pm}3.2$
& $99.6{\pm}0.3$
& $39.7{\pm}3.2$
& $54.9{\pm}1.0$
& $61.6$ \\

EWS, both (learned EFM)
& $50.2{\pm}9.5$
& $90.9{\pm}2.0$
& \textit{diverged}
& $54.7{\pm}0.6$
& -- \\

\midrule

Random
& $20.0$
& $50.0$
& $20.0$
& $50.0$
& $35.0$ \\

\bottomrule
\end{tabular}
\end{table}

\section{Conclusions}\label{sec:conclusion}

This paper introduced the exponentially weighted signature, a continuous-time model that generalises both the signature and state-space models. It keeps the group-like structure and universality of the signature and satisfies a modified Chen identity that enables a parallel associative scan. At depth one, it is an SSM to which trained LTI SSMs, and each channel of Mamba and head of Mamba-2, convert in closed form. The EWS extends these models through its truncation depth, clock and generator, each of which was shown to matter empirically. A complex spectrum was necessary on oscillatory targets, while rotation and Jordan generators were selected by validation on several datasets. State tracking was lost when the learned clocks were replaced by the time clock, and depth above one was selected on every classification dataset, where the EWS achieved the highest average accuracy and rank. Its limitations, including the memory cost of deeper levels, are discussed in Appendix~\ref{app:limitations}.

\subsection*{Acknowledgements}

The authors would like to thank J\'er\^{o}me Tomezyk for engaging and insightful discussions regarding efficient numerical methods.

This work was supported by the Laboratory for AI Security Research (LASR). The views expressed in this paper are those of the authors and do not necessarily reflect the position of LASR or His Majesty’s Government. Samuel N. Cohen\ acknowledges the support of the UKRI Prosperity Partnership Scheme (FAIR) under EPSRC Grant EP/V056883/1, and EPSRC Grant EP/Y028872/1 (Mathematical Foundations of Intelligence: An Erlangen Programme for AI). Terry Lyons is supported by UK Research and Innovation (UKRI) through the Engineering and Physical Sciences Research Council (EPSRC) via Programme Grants [Grant No.\ UKRI1010: High order mathematical and computational infrastructure for streamed data that enhance contemporary generative and large language models], [Grant No.\ EP/S026347/1: Unparameterised multi-model data, high order signatures and the mathematics of data science], [Grant No.\ EP/Y028872/1: Mathematical Foundations of Intelligence: An Erlangen Programme for AI], and the UKRI AI for Science award [Grant No.\ UKRI2385: Creating Foundational Benchmarks for AI in Physical and Biological Complexity]. Terry Lyons is also supported by The Alan Turing Institute under the Defence and Security Programme (funded by the UK Government) and through the provision of research facilities; by the UK Government; and through CIMDA@Oxford, part of the AIR@InnoHK initiative funded by the Innovation and Technology Commission, HKSAR Government. Benjamin Walker is supported by UK Research and Innovation (UKRI) through the Engineering and Physical Sciences Research Council (EPSRC) via Programme Grant [Grant No.\ UKRI1010: High order mathematical and computational infrastructure for streamed data that enhance contemporary generative and large language models] and CIMDA@Oxford, part of the AIR@InnoHK initiative funded by the Innovation and Technology Commission, HKSAR Government.

\bibliography{references}
\bibliographystyle{references}

\appendix

Appendix~\ref{app:background} collects the relevant mathematical tools needed to understand the main text: the tensor algebra, paths and their integration, the signature, and controlled differential equations. The theory underlying the EWS is somewhat involved, and Appendix~\ref{app:ews} presents it in full, from the operators through which the generator acts on each tensor depth to the linear controlled differential equation it solves and the proofs of the results of Section~\ref{sec:ews}. It also writes out the EWS in coordinates for small examples, which makes concrete what the EWS actually computes. This appendix is not essential for following the main text, and also works in greater generality. Where the main text takes paths in $\mathbb{R}^d$, the theory takes them in a general Banach space $V$, which adds no particular complexity over the Euclidean case. Where the main text takes paths of bounded variation, the theory takes paths of finite $p$-variation, which requires more care in both the statements and their proofs.

Appendix~\ref{app:limitations} discusses briefly the limitations of the EWS: its memory, its cost on densely sampled series, the automata a single layer can represent, and its scalar clock. Appendix~\ref{app:efm-volterra} compares the EWS with two related objects, the exponentially fading memory signature and the Volterra signature, and in particular explains why the EWS is group-like while the Volterra signature is not. Appendix~\ref{app:ssm} gives the conversions of Section~\ref{sec:ssm} in full, writing each SSM, from S4 to Mamba-$2$, as an EWS, and verifying the conversions numerically. Appendix~\ref{app:experiments} gives the details of every experiment, as well as an additional experiment on EEG abnormality detection and one on synthetic data.

\section{Limitations}\label{app:limitations}
Depth comes at the cost of state size, as depth $n$ of a layer with width $m$ has $m^n$ coordinates. Thus, while the generator and lift parameter counts do not grow with depth, memory does, and Table~\ref{tab:uea-time} shows the cost, most of which likely comes from our implementation rather than from the model. This is evidenced by the MotorImagery UEA task whose EWS's output state occupies roughly $1.7$GB, an order of magnitude below the measured peak of $34.5$GB. Our implementation is far from the hardware-aware, device-specific engineering behind the selective scan of selective SSMs such as Mamba \citep{gu2024mamba}, and so there is clear room for improvement. The state can also be compressed: since the EWS is group-like, much of this state is redundant, and log-signature coordinates would reduce depth two from $m^2$ to $m(m-1)/2$ coordinates and depth three from $m^3$ to $(m^3-m)/3$. Similarly to other SSMs, the EWS also suffers from having to compute its state at every observation, and is therefore costly on densely sampled time series. However, the EWS is the solution of a linear CDE; for such systems, the log-ODE method \citep{Walker2024LogNCDE} can be used to compute updates over intervals containing several observations, saving both time and memory. An analogue of this method for the EWS is thus a natural next step.

The vector fields of the EWS's linear CDE generate a solvable Lie algebra (Remark~\ref{rem:solvable}), and a single layer therefore cannot represent a non-solvable automaton, although stacking layers approximates such dynamics \citep{heo2026depthmattersparallelizablesequence}. Finally, the clock is a single scalar, shared by all rotation and Jordan blocks of the generator, whereas Mamba-$3$ gives each rotation its own data-dependent angle. Extending to a vector-valued clock is future work. Together with a path and readout computed from the input, it would let a single EWS layer reproduce a full Mamba-$1$ or Mamba-$2$ layer.

\section{Background}\label{app:background}

\subsection{The Tensor algebra}\label{app:tensor-algebra}

We first define the tensor algebra, the space in which the signature is defined. This requires recalling the definition of the tensor product of two vector spaces and establishing what norms are considered ``admissible" on these spaces.

\begin{definition}[Tensor Product \citep{lang2002algebra}]
    Let $U$ and $V$ be vector spaces over a field $\mathbb{F}$. A tensor product of $U$ and $V$ is defined as a vector space $U \otimes V$ equipped with a bilinear map $\tau:U \times V \rightarrow U \otimes V$. This space is uniquely characterised (up to isomorphism) by the following universal property: for any $\mathbb{F}$-vector space $W$ and any bilinear map $\kappa: U \times V \mapsto W$, there exists a unique linear map $\iota: U \otimes V \to W $ such that $\kappa = \iota \circ \tau$.
\end{definition}

The tensor product is thus the space on which bilinear maps become linear, and we write $u\otimes v$ for the image of $(u,v)$ under $\tau$. For $U=\mathbb{R}^a$ and $V=\mathbb{R}^b$ it is the outer product, $(u\otimes v)_{ij}=u_iv_j$, so that $U\otimes V$ is the space $\mathbb{R}^{a\times b}$ of matrices and any bilinear $\kappa(u,v)$ is the linear map $\iota(u\otimes v)$. Iterating the construction gives the tensor powers $V^{\otimes n}=V\otimes V^{\otimes n-1}$, with $V^{\otimes1}=V$ and $V^{\otimes0}=\mathbb{R}$ by convention.

\begin{definition}[Admissible Tensor Norms \citep{lyons2025signaturemethodsmachinelearning}]\label{def:admissible}
    Given a Banach space $V$, a family of norms on $\{V^{\otimes n}\}_{n=1}^{\infty}$ is said to be admissible if for all integers $n \geq 1$ we have chosen a norm on $V^{\otimes n}$ such that the following conditions are satisfied:
    \begin{enumerate}
         \item For all $n \geq 1$, the norm $||\cdot||_{V^{\otimes n}}$ is invariant under the action of the symmetric group $S_n$ on $V^{\otimes n}$. That is, 
\begin{equation*}
             ||\rho v||_{V^{\otimes n}} = ||v||_{V^{\otimes n}}, \quad \forall v \in V^{\otimes n}, \forall \rho \in S_n,
         \end{equation*}
         where $\rho(v_1 \otimes \cdots \otimes v_n) = v_{\rho(1)} \otimes \cdots \otimes v_{\rho(n)}$ for $v_i \in V$.
         \item For all $n,m \geq 1$, $||v \otimes w||_{V^{\otimes (n+m)}} \leq ||v||_{V^{\otimes n}} ||w||_{V^{\otimes m}}$, $\forall v \in V^{\otimes n}$, $w \in V^{\otimes m}$. That is, the norm is sub-multiplicative.
         \item For all $n,m \geq 1$ and for any dual elements $\phi \in (V^{\otimes n})^{\star}$ and $\sigma \in (V^{\otimes m})^{\star}$, we have $||\phi \otimes \sigma||_{(V^{\otimes(n+m)})^{\star}} \leq ||\phi||_{(V^{\otimes n})^{\star}} || \sigma||_{(V^{\otimes m})^{\star}}$.
     \end{enumerate}
\end{definition}

We are now able to define the tensor algebra, the space in which path signatures live.

\begin{definition}[Tensor Algebra \citep{lyons2007differential}]
    Let $\{V^{\otimes n}\}_{n=0}^{\infty}$ be equipped with admissible norms in the sense of the above definition, and let $V^{\otimes 0} = \mathbb{R}$ by convention. The tensor algebra is the space 
    \[
    T((V)) = \{\mathbf{a} = (a_0,a_1,\dots)| \forall n \geq 0, a_n \in V^{\otimes n}\},
    \]
    For two elements of $T((V))$, $\mathbf{a} = (a_0,a_1,\dots)$ and $\mathbf{b} = (b_0,b_1,\dots)$, addition is defined as 
    \[
    \mathbf{a} + \mathbf{b} = (a_0 + b_0, a_1 + b_1, \dots),
    \]
    and product defined as
    \[
    \mathbf{a} \otimes \mathbf{b} = (c_0, c_1, \dots),
    \]
    where for each $n \geq 0$
    \[
    c_n = \sum_{k=0}^n a_k \otimes b_{n-k}.
    \]
\end{definition}

Given the natural action of $\R$ by $\lambda \mathbf{a} = (\lambda a_0, \lambda a_1, \dots)$, the space $T((V))$ is a real non-commutative algebra with unit $\mathbf{1} = (1,0,0,\dots)$. We denote by $\tilde{T}((V))$ the space of elements $\mathbf{a}$ with $a_0 =1$; this space is a group with $\mathbf{a}^{-1} = 1 - (\mathbf{a}-1) + (\mathbf{a}-1)^{\otimes 2} - \cdots$. Finally, $T^N(V)$ is the truncated tensor algebra whose elements are of the form $\mathbf{a} = (a_0, \dots, a_N)$.

When $V =\mathbb{R}^d$ with standard basis $\{e_i\}_{i=1}^d$, the space $(\mathbb{R}^d)^{\otimes n}$ has dimension $d^n$ with basis $\big\{e_w=e_{i_1}\otimes\dots\otimes e_{i_n} : w=(i_1,\dots,i_n) \in \{1,\cdots,d\}^n\big\}$ indexed by the words $w$ of length $|w|=n$ in $d$ letters, with depth zero being spanned by the empty word. An element of $T((\mathbb{R}^d))$ is then of the form $a=\sum_{n\geq0}\sum_{|w|=n}a^we_w$; an element of $T(\mathbb{R}^d)$ is one with finitely many non-zero coefficients, which acts on $T((\mathbb{R}^d))$ as the linear functional $\langle l,a\rangle=\sum_wl_wa^w$; and an element of $T^N(\mathbb{R}^d)$ is one whose non-zero coefficients are for words of length at most $N$.

\begin{definition}[Shuffle product \citep{Ree1958LieEA}]\label{def:shuffle}
Let $\varphi\in\big(V^{\otimes n}\big)^\star$ and $\sigma\in\big(V^{\otimes k}\big)^\star$ for $n,k\geq1$. The shuffle of $\varphi$ and $\sigma$ is the linear functional $\varphi\shuffle\sigma\in\big(V^{\otimes(n+k)}\big)^\star$ defined by
\begin{equation}
\langle\varphi\shuffle\sigma,w\rangle=\sum_{\rho\in\mathrm{Sh}(n,k)}\langle\varphi\otimes\sigma,\rho(w)\rangle,\qquad w\in V^{\otimes(n+k)},
\end{equation}
where $\mathrm{Sh}(n,k)=\{\rho\in S_{n+k}\;:\;\rho(1)<\dots<\rho(n)\text{ and }\rho(n+1)<\dots<\rho(n+k)\}$ and $\rho$ acts on $V^{\otimes(n+k)}$ by permuting the tensor factors. The shuffle is commutative and associative, and is extended bilinearly to $T(V)$.
\end{definition}

The shuffle of a functional on depth $n$ with one on depth $k$ is a functional on depth $n+k$.
 
\begin{definition}[Group-like element \citep{reutenauer1993free}]\label{def:grouplike}
An element $a\in\widetilde{T}((V))$ is called group-like if $\langle l_1,a\rangle\,\langle l_2,a\rangle=\langle l_1\shuffle l_2,a\rangle$ for all $l_1,l_2\in T(V)$.
\end{definition}
 
For a group-like element, a product of two coordinates at levels $n$ and $k$ is therefore a linear combination of coordinates at depth $n+k$, and more generally a polynomial of degree $N$ in the depth-one coordinates is a linear functional of the depths up to $N$.

\subsection{Paths and integration}\label{app:paths}

With the tensor algebra established, we can now define the suitable framework for integration to be used in the definition of the signature.

\begin{definition}[$p$-Variation \citep{Young1936AnIO, lyons2007differential}]
    Let $V$ be a real Banach space and $X:[a,b] \to V$ be a path. For $p \geq 1$, the $p$-variation of $X$ is defined as 
    \[
    ||X||_{p,[a,b]} =  \Bigl( \sup_{\mathcal{P}}\sum_{i=0}^{n-1}||X_{t_{i+1}}- X_{t_{i}}||_V^p\Bigr)^{\frac{1}{p}},
    \]
    where $\mathcal{P}$ is the set of all finite partitions $\{t_i\}_{i=0}^n$ such that $a = t_0 < \cdots <t_n = b$.
\end{definition}

Paths in $\mathcal{V}^1$ are of bounded variation and $\|X\|_{1,[s,t]}$ is the total variation of $X$ over $[s,t]$. If $p\leq q$ then $\mathcal{V}^p\subseteq\mathcal{V}^q$. Piecewise linear paths built from finitely many observations, as found in the main text, belong in $\mathcal{V}^1$.

We assume throughout that every integrand and driver is regular enough for the integrals below to be defined in the Young sense \citep{Young1936AnIO}. That is, for $p,q\geq 1$ with $\frac1p+\frac1q>1$, $X\in\mathcal{V}^p([0,T];V)$ and $Y\in\mathcal{V}^q([0,T];L(V,W))$, the integral $\int_s^tY_u\,dX_u$ exists as a limit of Riemann sums independently of the partitions, and satisfies the Young estimate
\begin{equation}\label{eq:app-young}
\Big|\int_s^tY_u\,dX_u-Y_s\big(X_t-X_s\big)\Big|\leq C_{p,q}\,\|Y\|_{q,[s,t]}\,\|X\|_{p,[s,t]}
\end{equation}
for a constant $C_{p,q}$ depending only on $p$ and $q$ \citep[Theorem 1.16]{lyons2007differential}.
 
For $p=1$ this is a Riemann--Stieltjes integral, and the condition $\frac1p+\frac1q>1$ is then satisfied by every integrand of bounded variation.

\subsection{The Signature}\label{app:signature}

The signature of Section~\ref{sec:ews} is actually defined for every $X\in\mathcal{V}^p([0,T];V)$ with $p<2$, since $\frac1p+\frac1p>1$ makes each iterated integral a Young integral of the depth below it \citep{lyons2007differential}. For $p=1$ its levels decay factorially,
\begin{equation}\label{eq:app-sigdecay}
\big\|S(X)^{(n)}_{s,t}\big\|\leq\frac{\|X\|^n_{1,[s,t]}}{n!},
\end{equation}

with the bound for $1<p<2$ being the neo-classical inequality of \citet{lyons2007differential}.

\begin{theorem}[Chen Identity \citep{Chen1954Iterated}]
    Let $X \in \mathcal{V}^p([s,t],V)$ with $p<2$. Then for $u \in [s,t]$
    \[  
    S(X)_{s,t} = S(X)_{s,u} \otimes S(X)_{u,t}.
    \]
\end{theorem}

This tells us that the signature of a path over an interval can be decomposed as the product of signatures of sub-intervals.

\begin{theorem}[Shuffle Product Identity \citep{Ree1958LieEA}]\label{thm:shuffle}
    Let $X \in \mathcal{V}^p([s,t], V)$ for $p<2$, and $n,m \geq 0$ integers. Then for all bounded $\phi \in (V^{\otimes n})^{\star}$ and $\sigma \in (V^{\otimes m})^{\star}$, we have
    \[
    \langle \phi, S(X)^{(n)}_{s,t} \rangle \langle \sigma, S(X)^{(m)}_{s,t} \rangle = \langle \phi \shuffle \sigma, S(X)^{(n+m)}_{s,t} \rangle.
    \]
\end{theorem}
This tells us that polynomial functions in the lower depth terms of the signature can be expressed as linear functions of the higher depth terms. The signature is a canonical representation of the path; as shown in \citet{hambly2010uniqueness}, two paths of bounded variation have the same signature if and only if they are tree-like equivalent. This result was later extended to the $p<2$ setting in \citet{BOEDIHARDJO2016720}. In practice, when $V = \mathbb{R}^d$, this equivalence can be removed by augmenting the path with a strictly monotone channel, such as time, which ensures that the signature is injective on the space of paths (up to a translation constant). Another key property of the signature is its universal approximation capability.

\begin{theorem}[Universality]
    Let $K \subset \mathcal{V}^p([0, T], V)$ be a compact set of paths with $p < 2$. Let $\mathcal{F} \subset C(K, \mathbb{R})$ be the set of continuous functions such that $F(X) = F(Y)$ whenever $S(X) = S(Y)$. Then the set of linear functionals of the signature is dense in $\mathcal{F}$. That is, for any $F \in \mathcal{F}$ and $\epsilon > 0$, there exists a linear functional $l$ such that$$\sup_{X \in K} |F(X) - \langle l, S(X) \rangle| < \epsilon.$$
\end{theorem}

\subsection{Controlled differential equations}\label{app:cde}

\begin{definition}[Controlled differential equation \citep{lyons2007differential}]\label{def:cde}
Let $V$ and $W$ be normed spaces, let $f\colon W\to L(V,W)$ be a vector field and let $X\in\mathcal{V}^p([0,T];V)$ with $p<2$. A path $h\colon[0,T]\to W$ solves the CDE driven by $X$ with initial condition $h_0$ if
\begin{equation}
h_t=h_0+\int_0^tf(h_u)\,dX_u\qquad\text{for all }t\in[0,T],
\end{equation}
which we write as $dh_t=f(h_t)\,dX_t$. The CDE is linear if $f$ is linear in the state.
\end{definition}
 
In the main text, $V=\mathbb{R}^d$ and $W=\mathbb{R}^m$, such that $f(h)\in L(\mathbb{R}^d;\mathbb{R}^m)$. Writing $dX_t=\sum_{i=1}^de_i\,dX^i_t$ gives $f(h)\,dX_t=\sum_{i=1}^df(h)e_i\,dX^i_t$, and since each $h\mapsto f(h)e_i$ is a linear map on $\mathbb{R}^m$, setting $M_ih:=f(h)e_i$ defines matrices $M_1,\dots,M_d\in\mathbb{R}^{m\times m}$ for which we have the linear CDE 
\begin{equation}\label{eq:app-lcde}
dh_t=\sum_{i=1}^dM_ih_t\,dX^i_t.
\end{equation}

\begin{theorem}[Linear CDEs \citep{walker2025advances}]\label{thm:lcde}
Let $X\in\mathcal{V}^p([0,T];V)$ with $p<2$ and $M\in L(V,L(W,W))$, and consider the linear CDE $dh_t=M(dX_t)h_t$, that is, $f(h)v=M(v)h$. Let $\rho_M\colon T((V))\to L(W,W)$ be the linear map with $\rho_M(\mathbf{1})=\mathrm{id}_W$ and $\rho_M(v_1\otimes\dots\otimes v_n)=M(v_n)\cdots M(v_1)$, extended by linearity and continuity. Then the CDE has a unique solution for every initial condition, given for $s\leq t$ by the absolutely convergent series
\begin{equation}\label{eq:lcde-series}
h_t=\rho_M\big(S(X)_{s,t}\big)h_s=\sum_{n\geq0}\rho_M\big(S(X)^{(n)}_{s,t}\big)h_s.
\end{equation}
\end{theorem}

When $V=\mathbb{R}^d$ and $W=\mathbb{R}^m$, setting $M_i=M(e_i)$ gives the linear CDE of the background paragraph in Section~\ref{sec:ews}, and the series reads
\begin{equation}\label{eq:lcde-euclid}
h_t=\sum_{n\geq0}\sum_{i_1,\dots,i_n=1}^dM_{i_n}\cdots M_{i_1}\,S(X)^{i_1,\dots,i_n}_{s,t}\,h_s.
\end{equation}

The signature is itself the unique solution of a linear CDE $dS_t=S_t\otimes dX_t \in T((V))$, with $S_s=\mathbf{1}$. On a sufficiently short interval the matrix series $\rho_M\big(\log S(X)_{s,t}\big)$ converges and the solution takes the flow form, $h_t=\exp\big(\rho_M(\log S(X)_{s,t})\big)h_s$, the Magnus expansion \citep{frizmultidimensional}. If $X$ is linear on $[s,t]$ then every bracket term vanishes and $h_t=\exp\big(\sum_{i=1}^d(X^i_t-X^i_s)M_i\big)h_s$.

The state-space model $\dot h_t=-Ah_t+Bx_t$ is, with $X_t=\int_0^tx_u\,du$, the equation $dh_t=-Ah_t\,dt+B\,dX_t$, which is affine rather than linear in $h$. Appending a constant coordinate, $\bar h_t=(h_t,1)\in\mathbb{R}^{m+1}$, gives the linear CDE driven by $(t,X)$
\begin{equation}\label{eq:ssm-lcde}
d\bar h_t=M_0\bar h_t\,dt+\sum_{i=1}^dM_i\bar h_t\,dX^i_t,\qquad M_0=\begin{pmatrix}-A&0\\0&0\end{pmatrix},\qquad M_i=\begin{pmatrix}0&Be_i\\0&0\end{pmatrix},
\end{equation}
which has the form of (\ref{eq:app-lcde}). 

\section{The EWS}\label{app:ews}

Throughout this appendix $V$ and $W$ are Banach spaces whose tensor powers carry admissible norms (Definition~\ref{def:admissible}), which we assume also satisfy $\|T_1\otimes\dots\otimes T_n\|\leq\|T_1\|\cdots\|T_n\|$ for bounded operators $T_j$ on $W$, as the projective and injective norms do. The main text is the case $V=\mathbb{R}^d$ and $W=\mathbb{R}^m$, in which the Euclidean norm on coefficients satisfies these conditions. Definitions~\ref{def:clock} and~\ref{def:ews} are read for $1\leq p<2$, with $\mathcal{V}^p$ in place of $\mathcal{V}^1$, $V$ and $W$ in place of $\mathbb{R}^d$ and $\mathbb{R}^m$, $A\in L(W,W)$ and $B\in L(V,W)$, the EWS taking values in $T((W))$ and its iterated integrals being Young integrals.

\subsection{Coordinates and example}\label{app:ews-coordinates}
\paragraph{Coordinates.} When $W=\mathbb{R}^m$, write $E(r) = e^{-rA}$, with entries $E_{ij}(r)$. The EWS coordinate associated with the word $(i_1,\dots,i_n)$ is, 
\begin{equation}\label{eq:ews-coords}
    S_{\mathbb{A}}(X;\theta)^{i_1, \dots, i_n}_{s,t} = \int_{s < t_1 < \cdots <t_n <t} \sum_{j_1, \dots j_n =1}^{m} \prod_{l=1}^{n} E_{i_l j_l}(\theta_t - \theta_{t_l}) \, d(BX)^{j_1}_{t_1} \cdots d(BX)^{j_n}_{t_n},
\end{equation}
and the coefficient of a word is obtained from those of the words one letter shorter by
\begin{equation}\label{eq:app-recursion}
S_{\mathbb{A}}(X;\theta)^{i_1,\dots,i_n}_{s,t}=\!\!\sum_{k_1,\dots,k_{n-1}=1}^m\int_s^t\sum_{j=1}^m\Big(\prod_{l=1}^{n-1}E_{i_lk_l}(\theta_t-\theta_u)\Big)S_{\mathbb{A}}(X;\theta)^{k_1,\dots,k_{n-1}}_{s,u}E_{i_nj}(\theta_t-\theta_u)\,d(BX)^j_u,
\end{equation}
which is the coordinate form of Lemma~\ref{lem_recursive_def} below.

\begin{example}
    To illustrate the structure of the weighted iterated integrals, let us consider the simple setting where $V=W=\mathbb{R}^2$, $B$ is the identity and the clock is the time clock, $\theta_t=t$, with $A\in\mathbb{R}^{2\times2}$.
\textbf{Depth 1.} For a single-letter word $(i_1)$, we have
    \[
    S_{\mathbb{A}}(X;\theta)^{i_1}_{s,t}
    =
    \sum_{j_1 = 1}^{2}
    \int_s^t
    E_{i_1, j_1}(t - t_1)\,dX^{j_1}_{t_1},
    \qquad i_1 \in \{1,2\}.
    \]
    That is,
    \[
    S_{\mathbb{A}}(X;\theta)^{1}_{s,t} = \int_s^t \!\!\big(E_{11}(t-u)\,dX^1_u + E_{12}(t-u)\,dX^2_u\big),
    \qquad
    S_{\mathbb{A}}(X;\theta)^{2}_{s,t} = \int_s^t \!\!\big(E_{21}(t-u)\,dX^1_u + E_{22}(t-u)\,dX^2_u\big).
    \]
  \textbf{Depth 2.} For a two-letter word $(i_1,i_2)$, we obtain 
    \[
    S_{\mathbb{A}}(X;\theta)^{i_1,i_2}_{s,t}
    =
    \sum_{j_1,j_2=1}^{2}
    \int_s^t\!\!\int_s^{t_2}
    E_{i_1,j_1}(t - t_1)\,E_{i_2,j_2}(t - t_2)\,
    dX^{j_1}_{t_1}\,dX^{j_2}_{t_2}.
    \]
    Expanding the indices explicitly gives
\begin{align*}
S_{\mathbb{A}}(X;\theta)^{1,1}_{s,t}
&= \int_s^t\!\!\int_s^{t_2}
   \big(E_{11}(t-t_1)E_{11}(t-t_2)\,dX^1_{t_1}dX^1_{t_2}
   + E_{11}(t-t_1)E_{12}(t-t_2)\,dX^1_{t_1}dX^2_{t_2}\\[-3pt]
   &\qquad\qquad
   +\,E_{12}(t-t_1)E_{11}(t-t_2)\,dX^2_{t_1}dX^1_{t_2}
   + E_{12}(t-t_1)E_{12}(t-t_2)\,dX^2_{t_1}dX^2_{t_2}\big),\\[4pt]
S_{\mathbb{A}}(X;\theta)^{1,2}_{s,t}
&= \int_s^t\!\!\int_s^{t_2}
   \big(E_{11}(t-t_1)E_{21}(t-t_2)\,dX^1_{t_1}dX^1_{t_2}
   + E_{11}(t-t_1)E_{22}(t-t_2)\,dX^1_{t_1}dX^2_{t_2}\\[-3pt]
   &\qquad\qquad
   +\,E_{12}(t-t_1)E_{21}(t-t_2)\,dX^2_{t_1}dX^1_{t_2}
   + E_{12}(t-t_1)E_{22}(t-t_2)\,dX^2_{t_1}dX^2_{t_2}\big),
\end{align*}
and analogous formulas for $(i_1,i_2)=(2,1)$ and $(2,2)$.
\end{example}

\subsection{Flow and derivation operators}\label{app:operators}

The iterated integrals defining the EWS repeatedly apply the operator exponential $e^{-rA}$ to increments of the lifted path. To express these operations compactly across all tensor levels, we extend the action of $e^{-rA}$ from $W$ to the entire tensor algebra $T((W))$. When $A$ is diagonal, as for the EFM, $e^{-rA}$ acts on each word by a scalar $e^{-r(\lambda_{i_1}+ \cdots + \lambda_{i_n})}$, and every coordinate of the EWS evolves separately. For a general matrix, it mixes the coordinates within each depth (Example~\ref{example}), so the proofs below rely on $D^r_A$ being an algebra automorphism and $\Lambda_A$ a derivation. Since the clock is a path of finite $p$-variation rather than time, they also rely on the Young chain and product rules (Lemma~\ref{lemma_ews_cde}).

Recall that $T((W))$ is the free unital associative algebra generated by $W$. In particular, any linear map $f\colon W\to W$ admits a unique extension to an algebra homomorphism on $T((W))$, a consequence of the universal property of the tensor algebra \citep{lang2002algebra}. We apply this construction to the linear map $v\mapsto e^{-rA}v$ where, since $A\in L(W,W)$ is bounded, the operator exponential $e^{-rA}$ is well defined for all $r\in\mathbb{R}$ and defines a uniformly continuous one-parameter group on $W$ \citep{rudin1991}.

\begin{definition}\label{def:flow}
Let $A\in L(W,W)$ and $r\in\mathbb{R}$. The flow operator $D^r_A\colon T((W))\to T((W))$ is the unique algebra homomorphism extending the linear map $v\mapsto e^{-rA}v$ on $W$. That is,
\begin{equation}
D^r_A|_W=e^{-rA}.
\end{equation}
Equivalently, for any elementary tensor $v_1\otimes\dots\otimes v_n\in W^{\otimes n}$,
\begin{equation}\label{eq:flow-elem}
D^r_A(v_1\otimes\dots\otimes v_n)=(e^{-rA}v_1)\otimes\dots\otimes(e^{-rA}v_n).
\end{equation}
\end{definition}

\begin{lemma}\label{lem:flowgroup}
The family $\{D^r_A:r\in\mathbb{R}\}$ is a one-parameter subgroup of $\mathrm{Aut}(T((W)))$, the group of algebra automorphisms of $T((W))$.
\end{lemma}

\begin{proof}
Since $\mathrm{Aut}(T((W)))$ is a group, it suffices to show that the family is closed under composition and inverses. For any $r_1,r_2\in\mathbb{R}$ and any $v\in W$,
\[
D^{r_1}_AD^{r_2}_A(v)=e^{-r_1A}e^{-r_2A}v=e^{-(r_1+r_2)A}v=D^{r_1+r_2}_A(v).
\]
Both $D^{r_1}_AD^{r_2}_A$ and $D^{r_1+r_2}_A$ are algebra homomorphisms on $T((W))$ that agree on the generators $W$, and by the universal property of the tensor algebra they coincide on all of $T((W))$. Moreover $D^r_AD^{-r}_A=D^0_A=\mathrm{id}$, and hence $(D^r_A)^{-1}=D^{-r}_A$.
\end{proof}

We define the derivation induced by $A$ as the infinitesimal generator of the one-parameter group of Lemma~\ref{lem:flowgroup}, which is differentiable in $r$ since it acts on $W^{\otimes n}$ as $(e^{-rA})^{\otimes n}$.

\begin{definition}\label{def:derivation}
Let $A\in L(W,W)$. The operator $\Lambda_A\colon T((W))\to T((W))$ defined by
\begin{equation}
\Lambda_A=-\frac{d}{dr}\Big|_{r=0}D^r_A
\end{equation}
is called the derivation induced by $A$.
\end{definition}

\begin{lemma}\label{lem:derivation}
The operator $\Lambda_A$ is a derivation, that is, $\Lambda_A(a\otimes b)=\Lambda_A(a)\otimes b+a\otimes\Lambda_A(b)$ for all $a,b\in T((W))$. It satisfies $\Lambda_A(\mathbf{1})=0$ and, for any elementary tensor $v_1\otimes\dots\otimes v_n\in W^{\otimes n}$,
\begin{equation}\label{eq:lambda-elem}
\Lambda_A(v_1\otimes\dots\otimes v_n)=\sum_{k=1}^nv_1\otimes\dots\otimes(Av_k)\otimes\dots\otimes v_n.
\end{equation}
\end{lemma}

\begin{proof}
Since each $D^r_A$ is an algebra homomorphism, $D^r_A(a\otimes b)=D^r_A(a)\otimes D^r_A(b)$, and differentiating at $r=0$ gives
\[
\Lambda_A(a\otimes b)=\Big(-\frac{d}{dr}\Big|_{r=0}D^r_A(a)\Big)\otimes D^0_A(b)+D^0_A(a)\otimes\Big(-\frac{d}{dr}\Big|_{r=0}D^r_A(b)\Big)=\Lambda_A(a)\otimes b+a\otimes\Lambda_A(b).
\]
For $v\in W$, $\Lambda_A(v)=-\frac{d}{dr}\big|_{r=0}e^{-rA}v=Av$ by differentiability of the operator exponential, and (\ref{eq:lambda-elem}) follows from the Leibniz rule. Since $\Lambda_A(\mathbf{1})=\Lambda_A(\mathbf{1}\otimes\mathbf{1})=2\Lambda_A(\mathbf{1})$, we have $\Lambda_A(\mathbf{1})=0$.
\end{proof}

\subsection{The linear CDE form}\label{app:ews-cde}

\begin{lemma}\label{lem_recursive_def}
    For any $[s,t] \subseteq [0,T]$, the EWS satisfies     \begin{equation}\label{eq_ews_cde_integral}
        S_{\mathbb{A}}(X;\theta)_{s,t} = \mathbf{1} + \int^t_s D^{\theta_t -\theta_u}_A(S_{\mathbb{A}}(X;\theta)_{s,u} \otimes B\,dX_u),
    \end{equation}
    where $\mathbf{1} = (1,0,0,\dots)$ is the identity element in $T((W))$.
\end{lemma}

\begin{proof}
    Since trivially $S_{\mathbb{A}}(X;\theta)_{s,t}^{(0)} = 1$, we just need to show that the depth $n$ term of the integral in the claim equals the depth $n$ term of the EWS.
        \begin{align*}
            \Bigl(\int^t_s D^{\theta_t - \theta_u}_A(S_{\mathbb{A}}(X;\theta)_{s,u} \otimes B\,dX_u)\Bigr)^{(n)} & = \Bigl(\int^t_s D^{\theta_t - \theta_u}_A(S_{\mathbb{A}}(X;\theta)_{s,u}) \otimes D^{\theta_t - \theta_u}_A(B\,dX_u)\Bigr)^{(n)} \\
            & = \int^t_s \Bigl(D^{\theta_t - \theta_u}_A(S_{\mathbb{A}}(X;\theta)_{s,u}) \otimes D^{\theta_t - \theta_u}_A(B\,dX_u)\Bigr)^{(n)} \\
            &= \int^t_s \sum_{k=0}^n \Bigl(D^{\theta_t - \theta_u}_A(S_{\mathbb{A}}(X;\theta)_{s,u})\Bigr)^{(k)} \otimes \Bigl(D^{\theta_t - \theta_u}_A(B\,dX_u)\Bigr)^{(n-k)}\\
            &=\int^t_s \Bigl(D^{\theta_t - \theta_u}_A(S_{\mathbb{A}}(X;\theta)_{s,u})\Bigr)^{(n-1)} \otimes \Bigl(D^{\theta_t - \theta_u}_A(B\,dX_u)\Bigr)^{(1)}\\
            &= \int^t_s D^{\theta_t - \theta_u}_A(S_{\mathbb{A}}(X;\theta)_{s,u}^{(n-1)}) \otimes (e^{-(\theta_t - \theta_u)A}B\,dX_u),
        \end{align*}

        where the second equality uses the algebra homomorphism property of the flow operator, and the fourth equality uses the fact that $BdX_u$ is only non-zero at depth one. We can then plug in the iterated integral expression for the EWS to get
        \[
        \begin{aligned}
&\Bigl(\!\int_s^t \! D_A^{\theta_t-\theta_u}
(S_{\mathbb{A}}(X;\theta)_{s,u}\!\otimes B\,dX_u)\!\Bigr)^{(n)}\\
&= \!\int_s^t \!D_A^{\theta_t-\theta_u}\!\Bigl(\!
  \int_s^u\!\int_s^{t_{n-1}}\!\!\cdots\!\int_s^{t_2}
  e^{-(\theta_u-\theta_{t_1})A}B\,dX_{t_1}\!\otimes\!\cdots\otimes\! e^{-(\theta_u-\theta_{t_{n-1}})A}B\,dX_{t_{n-1}}
\!\Bigr)\!\otimes\! e^{-(\theta_t-\theta_u)A}\!\,B\,dX_u \\[4pt]
&= \int_s^t\!\int_s^u\!\int_s^{t_{n-1}}\!\cdots\!\int_s^{t_2}
   e^{-(\theta_t-\theta_{t_1})A}B\,dX_{t_1}\otimes\cdots\otimes
   e^{-(\theta_t-\theta_{t_{n-1}})A}B\,dX_{t_{n-1}}\otimes e^{-(\theta_t-\theta_u)A}\!\,B\,dX_u.
\end{aligned}
\]
        Changing notation from $u$ to $t_n$ gives us exactly $S_{\mathbb{A}}(X;\theta)_{s,t}^{(n)}$, as required.
\end{proof}

This integral form now lets us write the EWS as the solution to a linear CDE; the following lemma is Proposition~\ref{prop:cde} for paths of finite $p$-variation with $p<2$.

\begin{lemma}\label{lemma_ews_cde}
The EWS is the unique solution of the linear controlled differential equation driven by the augmented path $(\theta,X)\in\mathcal{V}^p([0,T];\mathbb{R} \times V)$,
\begin{equation}\label{eq_ews_cde}
      dS_{\mathbb{A}}(X;\theta)_{s,t} = -\Lambda_A S_{\mathbb{A}}(X;\theta)_{s,t}\,d\theta_t + S_{\mathbb{A}}(X;\theta)_{s,t}\otimes B\,dX_t,\qquad S_{\mathbb{A}}(X;\theta)_{s,s}=\mathbf{1}.
  \end{equation}
  Equivalently, in Young integral form,
  \begin{equation}
      S_{\mathbb{A}}(X;\theta)_{s,t}=\mathbf{1}+\int_s^t\bigl(-\Lambda_A S_{\mathbb{A}}(X;\theta)_{s,u}\bigr)\,d\theta_u+\int_s^t S_{\mathbb{A}}(X;\theta)_{s,u}\otimes B\,dX_u.
  \end{equation}
\end{lemma}
\begin{proof}
We argue by induction on the depth $n\geq1$, assuming that $u\mapsto S_{\mathbb{A}}(X;\theta)^{(n-1)}_{s,u}$ is continuous with finite $p$-variation, which holds for $n=1$ since $S_{\mathbb{A}}(X;\theta)^{(0)}_{s,u}=1$. Since $D^r_A$ is an algebra homomorphism that is grade-preserving and $B\,dX_u$ lies at depth one, Lemma~\ref{lem_recursive_def} at depth $n$ reads
\[
S_{\mathbb{A}}(X;\theta)^{(n)}_{s,t}=\int_s^tD^{\theta_t-\theta_u}_A\big(S_{\mathbb{A}}(X;\theta)^{(n-1)}_{s,u}\otimes B\,dX_u\big).
\]
By Lemma~\ref{lem:flowgroup}, $D^{\theta_t-\theta_u}_A=D^{\theta_t}_AD^{-\theta_u}_A$, and taking the bounded operator $D^{\theta_t}_A$ out of the integral gives $S_{\mathbb{A}}(X;\theta)^{(n)}_{s,t}=D^{\theta_t}_AY_t$ with
\[
Y_t=\int_s^tD^{-\theta_u}_A\big(S_{\mathbb{A}}(X;\theta)^{(n-1)}_{s,u}\otimes B\,dX_u\big),
\]
a Young integral whose integrand is a locally Lipschitz function of $\theta_u$ and $S_{\mathbb{A}}(X;\theta)^{(n-1)}_{s,u}$, both continuous with finite $p$-variation. By Lemma~\ref{lem:flowgroup} and Definition~\ref{def:derivation}, $r\mapsto D^r_A$ is continuously differentiable on each level, with $\frac{d}{dr}D^r_A=\lim_{\epsilon\to0}\epsilon^{-1}\big(D^\epsilon_A-\mathrm{Id}\big)D^r_A=-\Lambda_AD^r_A$. Hence $t\mapsto D^{\theta_t}_A$ is continuous with finite $p$-variation, and by the chain rule for Young integrals and the definition of $Y$,
\[
dD^{\theta_t}_A=-\Lambda_AD^{\theta_t}_A\,d\theta_t,\qquad dY_t=D^{-\theta_t}_A\big(S_{\mathbb{A}}(X;\theta)^{(n-1)}_{s,t}\otimes B\,dX_t\big).
\]
The product rule for Young integrals then gives
\[
dS_{\mathbb{A}}(X;\theta)^{(n)}_{s,t}=-\Lambda_AD^{\theta_t}_AY_t\,d\theta_t+D^{\theta_t}_AD^{-\theta_t}_A\big(S_{\mathbb{A}}(X;\theta)^{(n-1)}_{s,t}\otimes B\,dX_t\big)=-\Lambda_AS_{\mathbb{A}}(X;\theta)^{(n)}_{s,t}\,d\theta_t+S_{\mathbb{A}}(X;\theta)^{(n-1)}_{s,t}\otimes B\,dX_t.
\]
The chain and product rules hold because $p<2$. The second-order terms of their Riemann sums, $\sum_i|\theta_{u_{i+1}}-\theta_{u_i}|^2$ and $\sum_i\|D^{\theta_{u_{i+1}}}_A-D^{\theta_{u_i}}_A\|\,\|Y_{u_{i+1}}-Y_{u_i}\|$, vanish as the mesh tends to zero for continuous paths of finite $p$-variation \citep{frizmultidimensional}. In particular $t\mapsto S_{\mathbb{A}}(X;\theta)^{(n)}_{s,t}=D^{\theta_t}_AY_t$ is continuous with finite $p$-variation, which completes the induction. At depth zero $S_{\mathbb{A}}(X;\theta)^{(0)}_{s,t}=1$ and $\Lambda_A(\mathbf{1})=0$, and the equation holds trivially. Reassembling the levels gives (\ref{eq_ews_cde}), with $S_{\mathbb{A}}(X;\theta)_{s,s}=\mathbf{1}$. For uniqueness, let $S'$ be any solution of (\ref{eq_ews_cde}). Since $\Lambda_A$ preserves levels and the level-$n$ part of $S'\otimes B\,dX$ involves $S'$ only through depth $n-1$, $\pi_N(S')$ solves the equation of Lemma~\ref{lem:truncated} for every $N$, and hence $\pi_N(S')=\pi_N\big(S_{\mathbb{A}}(X;\theta)_{s,\cdot}\big)$ for every $N$.
\end{proof}

\begin{lemma}\label{lem:truncated}
Let $X\in\mathcal{V}^p([0,T];V)$ for $p<2$ and fix an integer $N\geq1$. Define $f\colon T^N(W)\to L(\mathbb{R}\times V,T^N(W))$ by
\[
f(a)(\vartheta,v)=-\Lambda_A(a)\cdot\vartheta+\pi_N(a\otimes Bv),
\]
where $\pi_N\colon T((W))\to T^N(W)$ denotes projection onto the first $N+1$ levels (including the $0$-th level). Then the unique solution to the CDE
\[
dS_t=f(S_t)\,d(\theta,X)_t,\qquad S_s=(1,0,\dots,0),
\]
is the truncated EWS $S_t=\pi_N(S_{\mathbb{A}}(X;\theta)_{s,t})$.
\end{lemma}

\begin{proof}
Since $p<2$, by Theorem~\ref{thm:lcde}, existence and uniqueness hold. To verify that $\pi_N(S_{\mathbb{A}}(X;\theta)_{s,t})$ solves this equation, we project the EWS dynamics from Lemma~\ref{lemma_ews_cde}. Applying $\pi_N$ to both sides gives
\[
d\pi_N(S_{\mathbb{A}}(X;\theta)_{s,t})=-\Lambda_A\pi_N(S_{\mathbb{A}}(X;\theta)_{s,t})\,d\theta_t+\pi_N(S_{\mathbb{A}}(X;\theta)_{s,t}\otimes B\,dX_t).
\]
As the tensor product $S_{\mathbb{A}}(X;\theta)_{s,t}\otimes B\,dX_t$ only affects levels up to $N$ through components of $S_{\mathbb{A}}(X;\theta)_{s,t}$ at levels up to $N-1$, and $\Lambda_A$ preserves tensor degrees (it acts level-wise), we have
\[
\pi_N(S_{\mathbb{A}}(X;\theta)_{s,t}\otimes B\,dX_t)=\pi_N(\pi_N(S_{\mathbb{A}}(X;\theta)_{s,t})\otimes B\,dX_t).
\]
Thus, $\pi_N(S_{\mathbb{A}}(X;\theta)_{s,t})$ satisfies the CDE with vector field $f$ and initial condition $\pi_N(S_{\mathbb{A}}(X;\theta)_{s,s})=\pi_N((1,0,\dots,0))=(1,0,\dots,0)$. By uniqueness, it is the solution.
\end{proof}

Proposition \ref{prop:cde} shows that, just like the signature, the EWS is the solution to a linear CDE which takes values in the tensor algebra. However, to make contact with standard finite-dimensional models, we let $V = \mathbb{R}^d$ and $W = \mathbb{R}^m$, and truncate the tensor algebra at depth $N$ to work with the truncated EWS. Applying the canonical grade--lexicographic flattening
\begin{equation}
    \Phi: T^N(\mathbb{R}^m)\xrightarrow{\;\cong\;}\mathbb{R}^K,
\qquad K=\sum_{n=0}^N m^n,
\end{equation}
the linear CDE satisfied by the truncated EWS becomes a finite-dimensional linear CDE in $\mathbb{R}^K$. Writing $dX_t = \sum_{i=1}^d e_i dX^i_t$, one obtains the flattened CDE
\begin{equation}\label{eq:flattened}
    dS_{\mathbb{A}}(X;\theta)^{\leq N}_{s,t}
=
-\,L\, S_{\mathbb{A}}(X;\theta)^{\leq N}_{s,t}\, d\theta_t
+
\sum_{i=1}^d \rho(Be_i)\, S_{\mathbb{A}}(X;\theta)^{\leq N}_{s,t}\, dX_t^i,
\end{equation}
where 
\begin{equation}L=\operatorname{blockdiag}\big(L^{(0)},L^{(1)},\dots,L^{(N)}\big),
\qquad L^{(0)}=[0],
\end{equation}
is the matrix representation of $\Lambda_A$, and each $\rho(v) \in \mathbb{R}^{K\times K}$ is the sparse strictly lower-triangular matrix representing right tensor-multiplication by $v\in\mathbb{R}^m$, so that $\rho(Be_i)$ maps depth $n-1$ into depth $n$. Equivalently, writing $M_i$ for the coefficient matrices, the system is a linear CDE driven by the path $(\theta_t,X_t)\in\mathbb{R}^{1+d}$ in the standard form
\begin{equation}\label{eq:Mi_form}
dS_{\mathbb{A}}(X;\theta)^{\leq N}_{s,t} = M_0 S_{\mathbb{A}}(X;\theta)^{\leq N}_{s,t}\,d\theta_t+\sum_{i=1}^d M_i S_{\mathbb{A}}(X;\theta)^{\leq N}_{s,t}\,dX^i_t,
\qquad
\begin{aligned}M_0&=-L,\\ M_i&=\rho(Be_i).\end{aligned}
\end{equation}
When the clock is a linear functional of the path, $\theta_t=\sum_{i}\ell_iX^i_t$, the first driver may be removed and the system is driven by $X$ alone, with $M_i=\rho(Be_i)-\ell_iL$.

This representation makes explicit that the truncated EWS evolves according to a linear CDE with a highly structured collection of coefficient matrices. All $d+1$ of them are generated by the $m^2+dm$ entries of $A$ and $B$. Hence when the parameters of $\mathbb{A} = (A,B)$ are learnable, the truncated EWS defines a structured linear NCDE within the SLiCE framework that encodes a specific inductive bias. In particular, its hidden state has dimension $K$, which is determined by the truncation depth $N$ and the dimension $m$ of the lifted path. We illustrate this structure with a simple example.

\begin{example}\label{example}
We consider the case of a path $X$ taking values in $\mathbb{R}^2$ and compute explicitly the CDE matrices for the EWS truncated at depth $N=2$ with $B$ the identity, so that $m=d=2$. Let $
A =
\begin{pmatrix}
\alpha & \beta \\
\gamma & \delta
\end{pmatrix}.
$ The truncated tensor algebra has dimension $K = 1 + 2 + 4 = 7$, with grade--lexicographic basis $
\{1,\; e_1,\; e_2,\; e_1 \otimes e_1,\; e_1 \otimes e_2,\; e_2 \otimes e_1,\; e_2 \otimes e_2\}.$ The derivation matrix $L = \operatorname{blockdiag}(L^{(0)}, L^{(1)}, L^{(2)})$ has graded blocks
\begin{equation*}
    L^{(0)} = [0],
\qquad
L^{(1)} =
\begin{pmatrix}
\alpha & \beta \\
\gamma & \delta
\end{pmatrix}, \qquad
    L^{(2)} =
\begin{pmatrix}
2\alpha & \beta & \beta & 0 \\
\gamma & \alpha+\delta & 0 & \beta \\
\gamma & 0 & \alpha+\delta & \beta \\
0 & \gamma & \gamma & 2\delta
\end{pmatrix}.
\end{equation*}

Right tensor multiplication by $e_1$ and $e_2$ is represented by
\begin{equation*}
    \rho(e_1) =
\begin{pmatrix}
0 & 0 & 0 & 0 & 0 & 0 & 0 \\
1 & 0 & 0 & 0 & 0 & 0 & 0 \\
0 & 0 & 0 & 0 & 0 & 0 & 0 \\
0 & 1 & 0 & 0 & 0 & 0 & 0 \\
0 & 0 & 0 & 0 & 0 & 0 & 0 \\
0 & 0 & 1 & 0 & 0 & 0 & 0 \\
0 & 0 & 0 & 0 & 0 & 0 & 0
\end{pmatrix},
\qquad
\rho(e_2) =
\begin{pmatrix}
0 & 0 & 0 & 0 & 0 & 0 & 0 \\
0 & 0 & 0 & 0 & 0 & 0 & 0 \\
1 & 0 & 0 & 0 & 0 & 0 & 0 \\
0 & 0 & 0 & 0 & 0 & 0 & 0 \\
0 & 1 & 0 & 0 & 0 & 0 & 0 \\
0 & 0 & 0 & 0 & 0 & 0 & 0 \\
0 & 0 & 1 & 0 & 0 & 0 & 0
\end{pmatrix}.
\end{equation*}

The instantaneous linear operator appearing in the flattened EWS CDE, $dS_{\mathbb A}(X;\theta)^{\le 2}_{s,t}
=
\mathcal{M}_t \,
S_{\mathbb A}(X;\theta)^{\le 2}_{s,t},$ is therefore
\begin{equation*}
    \mathcal{M}_t
=
- L \, d\theta_t
+
\rho(e_1) \, dX_t^1
+
\rho(e_2) \, dX_t^2,
\end{equation*}
which explicitly equals
\begin{equation*}
    \begin{pmatrix}
0 & 0 & 0 & 0 & 0 & 0 & 0 \\[6pt]
dX_t^1 & -\alpha\,d\theta_t & -\beta\,d\theta_t & 0 & 0 & 0 & 0 \\[6pt]
dX_t^2 & -\gamma\,d\theta_t & -\delta\,d\theta_t & 0 & 0 & 0 & 0 \\[8pt]
0 & dX_t^1 & 0 & -2\alpha\,d\theta_t & -\beta\,d\theta_t & -\beta\,d\theta_t & 0 \\[8pt]
0 & dX_t^2 & 0 & -\gamma\,d\theta_t & -(\alpha+\delta)\,d\theta_t & 0 & -\beta\,d\theta_t \\[8pt]
0 & 0 & dX_t^1 & -\gamma\,d\theta_t & 0 & -(\alpha+\delta)\,d\theta_t & -\beta\,d\theta_t \\[8pt]
0 & 0 & dX_t^2 & 0 & -\gamma\,d\theta_t & -\gamma\,d\theta_t & -2\delta\,d\theta_t
\end{pmatrix}.
\end{equation*}
\end{example}

\begin{remark}[Solvability]\label{rem:solvable}
The Lie algebra $\mathfrak{g}$ generated by the vector fields $M_0,\dots,M_d$ of (\ref{eq:Mi_form}) is solvable. Let $\mathfrak{n}$ be the Lie algebra of strictly block lower-triangular matrices that map each depth into higher depths, and which is nilpotent. For $i = 1,\dots,d$, each $M_i=\rho(Be_i)$ lies in $\mathfrak{n}$, and $M_0=-L$ is block diagonal, which gives $[M_0,\mathfrak{n}]\subseteq\mathfrak{n}$. Thus, $\mathfrak{g}\subseteq\mathbb{R}M_0\oplus\mathfrak{n}$ and $[\mathfrak{g},\mathfrak{g}]\subseteq\mathfrak{n}$, and $\mathfrak{g}$ is solvable.
\end{remark}

\subsection{Equivalence to the signature of a re-weighted path}\label{app:ews-reweighted}

We now show that, over a fixed interval, EWS admits an equivalent representation as the classical signature of a linearly transformed path. We consider $X\in\mathcal{V}^p([0,T];V)$ and the bounded linear operators $A\in L(W,W)$ and $B\in L(V,W)$ as before. We define the exponentially re-weighted path over $[s,t]\subseteq[0,T]$ by
\begin{equation}\label{eq:reweighted}
Z^{[t]}_\tau:=\int_s^\tau e^{-(\theta_t-\theta_u)A}B\,dX_u,\qquad\tau\in[s,t],
\end{equation}
where the integral is understood in the Young sense. 
The superscript $[t]$ denotes the horizon from which each increment is weighted, and different horizons yield different paths. For a given clock, this construction defines a linear memory transform of $X$ up to the horizon $t$, with operator-valued kernel $r\mapsto e^{-rA}B$. The main observation of this subsection is that the EWS of $X$ over the interval $[s,t]$ coincides exactly with the classical signature of the weighted path $Z^{[t]}$ over the same interval.

\begin{proposition}\label{prop:reweighted}
Let $[s,t]\subseteq[0,T]$ and let $Z^{[t]}$ be defined as above. Then,
\begin{equation}
S_{\mathbb{A}}(X;\theta)_{s,t}=S(Z^{[t]})_{s,t}
\end{equation}
as elements of the tensor algebra.
\end{proposition}

\begin{proof}
We show that this holds for the depth $n$ term by substituting $dZ^{[t]}_u=e^{-(\theta_t-\theta_u)A}B\,dX_u$ into the classical definition of the signature:
\begin{align*}
S(Z^{[t]})^{(n)}_{s,t}&=\int_s^t\int_s^{t_n}\cdots\int_s^{t_2}dZ^{[t]}_{t_1}\otimes\cdots\otimes dZ^{[t]}_{t_n}\\
&=\int_s^t\int_s^{t_n}\cdots\int_s^{t_2}\big(e^{-(\theta_t-\theta_{t_1})A}B\,dX_{t_1}\big)\otimes\cdots\otimes\big(e^{-(\theta_t-\theta_{t_n})A}B\,dX_{t_n}\big)\\
&=S_{\mathbb{A}}(X;\theta)^{(n)}_{s,t}.
\end{align*}
\end{proof}

\begin{remark}\label{rem:pointwise}
The identity $S_A(X;\theta)_{s,t}=S(Z^{[t]})_{s,t}$ holds for each terminal time $t$ separately, with a path $Z^{[t]}$ that depends on $t$. Therefore, the trajectory $t\mapsto S_A(X;\theta)_{s,t}$ is not the signature of a single evolving path, so the classical Chen identity does not apply, and Lemma~\ref{lem:chen} proves the modified form directly.
\end{remark}

The identity reduces the proofs of group-likeness, the factorial bound, continuity and universality to classical signature results for the single
path $Z^{[t]}$. It is not used for the linear CDE form (Lemma~\ref{lemma_ews_cde}) or the modified Chen identity (Lemma~\ref{lem:chen}), and it does not help to compute the EWS over time, since $S(Z^{[t]})_{s,t}$ and $S(Z^{[t+h]})_{s,t+h}$ are related only through the flow operator.

\begin{remark}\label{rem:ssm-state}
For the time clock and $X_t=\int_0^tx_s\,ds$, the latent state of the SSM of Section~\ref{sec:ssm} is $h_\tau=\int_0^\tau e^{-(\tau-s)A}B\,dX_s$, whereas $Z^{[t]}_\tau=\int_0^\tau e^{-(t-s)A}B\,dX_s$. For $\tau<t$ the two paths differ, as $h$ decays relative to the moving time $\tau$ and $Z^{[t]}$ relative to the fixed horizon $t$. Both start at zero and meet at $\tau=t$, and hence agree at depth one, but $S_{\mathbb{A}}(X;\theta)_{0,t}=S(Z^{[t]})_{0,t}$ is in general not the signature of $h$ at deeper levels.
\end{remark}

\subsection{Analytic properties}

We now establish analytic well-posedness of the EWS for paths $X\in\mathcal{V}^p([0,T];V)$ with $p<2$, meaning that the defining tensor series converges. Since the EWS coincides with the classical signature of a suitably re-weighted path, well-posedness for general $p<2$ follows directly from standard signature theory. For completeness, we demonstrate the factorial decay bound explicitly in the finite variation case, which makes the convergence mechanism and the dependence on the operator $A$ transparent.

\begin{lemma}\label{lem:decay}
For $X\in\mathcal{V}^1([0,T];V)$, the EWS is well-defined, and for all $[s,t]\subseteq[0,T]$, the following bound holds
\begin{equation}\label{eq:ewsdecay}
\|S_{\mathbb{A}}(X;\theta)^{(n)}_{s,t}\|_{W^{\otimes n}}\leq\frac{\big(C_{A,\theta}\|B\|_{\mathrm{op}}\|X\|_{1,[s,t]}\big)^n}{n!},
\end{equation}
where $C_{A,\theta}=\sup_{u\in[s,t]}\|e^{-(\theta_t-\theta_u)A}\|_{\mathrm{op}}<\infty$ since $A$ is bounded and $\theta$ is continuous.
\end{lemma}

\begin{proof}
By Proposition~\ref{prop:reweighted} and the factorial decay of the signature, $\|S_{\mathbb{A}}(X;\theta)^{(n)}_{s,t}\|_{W^{\otimes n}}=\|S(Z^{[t]})^{(n)}_{s,t}\|_{W^{\otimes n}}\leq\|Z^{[t]}\|^n_{1,[s,t]}/n!$. Since $\|e^{-(\theta_t-\theta_u)A}B\|_{\mathrm{op}}\leq C_{A,\theta}\|B\|_{\mathrm{op}}$ for all $u\in[s,t]$, we have $\|Z^{[t]}\|_{1,[s,t]}\leq C_{A,\theta}\|B\|_{\mathrm{op}}\|X\|_{1,[s,t]}$, and the bound follows.
\end{proof}

Lemma~\ref{lem:decay} establishes the existence of the EWS for paths of finite variation ($p=1$). To extend this result to the case of $1<p<2$, note that the continuous map $u\mapsto e^{-(\theta_t-\theta_u)A}$ is of finite $p$-variation since $\theta\in\mathcal{V}^p([0,T];\mathbb{R})$, and since $\frac1p+\frac1p>1$ the existence of $Z^{[t]}$ is guaranteed by the Young estimate
\[
\|Z^{[t]}\|_{p,[s,t]}\leq C_{\theta,A}\|X\|_{p,[s,t]},
\]
where the constant $C_{\theta,A}$ depends on $\sup_{u\in[s,t]}\|e^{-(\theta_t-\theta_u)A}\|_{\mathrm{op}}$, on $\|B\|_{\mathrm{op}}$ and on the $p$-variation of the map $u\mapsto e^{-(\theta_t-\theta_u)A}$. By Proposition~\ref{prop:reweighted}, the existence of the EWS for $p<2$ then follows directly from the existence of the signature of Young paths.

\begin{corollary}[Scale]\label{cor:scale}
Let $X\in\mathcal{V}^1([0,T];V)$ and $[s,t]\subseteq[0,T]$. Suppose that $\|e^{-rA}\|_{\mathrm{op}}\leq e^{-\mu r}$ for all $r\geq0$ and some $\mu>0$. Assume that the clock is non-decreasing and that $\|X\|_{1,[u,v]}\leq L(\theta_v-\theta_u)$ for all $s\leq u\leq v\leq t$. Then for all $n\geq1$, we have
\[
\|S_{\mathbb{A}}(X;\theta)^{(n)}_{s,t}\|_{W^{\otimes n}}\leq\frac{1}{n!}\Big(\frac{L\|B\|_{\mathrm{op}}}{\mu}\Big)^n,
\]
independently of the length of $[s,t]$ and of the total variation of $X$.
\end{corollary}

\begin{proof}
From Lemma~\ref{lem:decay}, $\|S_{\mathbb{A}}(X;\theta)^{(n)}_{s,t}\|_{W^{\otimes n}}\leq\|Z^{[t]}\|^n_{1,[s,t]}/n!$. Then, 
\[
\|Z^{[t]}\|_{1,[s,t]}\leq\|B\|_{\mathrm{op}}\int_s^te^{-\mu(\theta_t-\theta_u)}\,d\|X\|_{1,[s,u]}\leq L\|B\|_{\mathrm{op}}\int_s^te^{-\mu(\theta_t-\theta_u)}\,d\theta_u=\frac{L\|B\|_{\mathrm{op}}}{\mu}\big(1-e^{-\mu(\theta_t-\theta_s)}\big)\leq\frac{L\|B\|_{\mathrm{op}}}{\mu}.
\]
\end{proof}

When $V=\mathbb{R}^d$ and $W=\mathbb{R}^m$ as in the main text, the condition on $A$ holds whenever $A+A^\top\geq2\mu I$, and the condition on the clock holds with $L=1$ for the variation clock. By contrast, an unrestricted linear CDE has no bound independent of the length of the path. On an interval where $X$ is linear its solution is $h_t=\exp\big(\sum_{i=1}^d(X^i_t-X^i_s)M_i\big)h_s$. If $\sum_iv_iM_i$ has an eigenvalue $\lambda$ with $\mathrm{Re}\,\lambda>0$ for some $v\in\mathbb{R}^d$, then the norm $\|h_t\|$ grows at least like $e^{\mathrm{Re}\,\lambda\,(t-s)}$ along the path $X_t=tv$ for generic $h_s$, that is, exponentially in the total variation $\|X\|_{1,[s,t]}=(t-s)\|v\|$.

A clock is continuous if it is continuous for the $p$-variation topologies on $\mathcal{V}^p([0,T];V)$ and $\mathcal{V}^p([0,T];\mathbb{R})$. The time clock is continuous for every $p<2$, and the variation clock for $p=1$.

\begin{corollary}\label{cor:continuity}
For each $p\in[1,2)$, each integer $N\geq0$ and each continuous clock $\theta$, the truncated EWS defines a continuous mapping
\[
\pi_N\circ S_{\mathbb{A}}(\,\cdot\,;\theta)_{s,t}\colon\mathcal{V}^p([0,T];V)\to T^N(W)
\]
with respect to the $p$-variation topology on the domain.
\end{corollary}

\begin{proof}
By Proposition~\ref{prop:reweighted}, $\pi_N(S_{\mathbb{A}}(X;\theta)_{s,t})=\pi_N(S(Z^{[t]})_{s,t})$, where $Z^{[t]}_\tau=\int_s^\tau K_u\,dX_u$ with $K_u=e^{-(\theta_t-\theta_u)A}B$. Let $X^{(k)}\to X$ in $p$-variation. Since $\theta$ is continuous, $\theta(X^{(k)})\to\theta(X)$ in $p$-variation, and since $r\mapsto e^{-rA}B$ is smooth, the corresponding integrands $K^{(k)}$ converge to $K$ in $p$-variation. The Young estimate (\ref{eq:app-young}) makes the Young integral continuous in its integrand and its driver, and hence the re-weighted paths of $X^{(k)}$ converge to $Z^{[t]}$ in $p$-variation. The truncated signature is continuous on $\mathcal{V}^p([s,t];W)$ for $p<2$ \citep{lyons2007differential}, which gives the claim.
\end{proof}

\subsection{Algebraic properties}\label{app:ews-algebraic}

We now show that the EWS satisfies the key algebraic properties of the classical path signature. In particular, it is multiplicative with respect to path concatenation (Chen's identity), and linear functionals of the EWS are closed under multiplication, forming a commutative algebra via the shuffle product.

\paragraph{Chen's identity.} As for the signature \citep{lyons2007differential}, we first establish the modified Chen identity for $X\in\mathcal{V}^1([0,T];V)$ and then extend it to the $p\in(1,2)$ regime. The proof for $p =1$ works directly with the weighted iterated integrals and does not use Proposition~\ref{prop:reweighted}; the extension to $p<2$ uses it only through the continuity of Corollary~\ref{cor:continuity}. 

\begin{lemma}\label{lem:chen}
Let $X\in\mathcal{V}^p([0,T];V)$ for $p=1$. Then for any $0\leq s\leq u\leq t\leq T$, the exponentially-weighted signature satisfies
\begin{equation}\label{eq:chen}
S_{\mathbb{A}}(X;\theta)_{s,t}=\big(D^{\theta_t-\theta_u}_AS_{\mathbb{A}}(X;\theta)_{s,u}\big)\otimes S_{\mathbb{A}}(X;\theta)_{u,t}.
\end{equation}
\end{lemma}

\begin{proof}
We will show that the statement holds at an arbitrary depth $n$. For all $s\leq u\leq t$, we can split the region of integration as
\begin{align*}
S_{\mathbb{A}}(X;\theta)^{(n)}_{s,t}
&=\int_s^t\int_s^{t_n}\cdots\int_s^{t_2}\big(e^{-(\theta_t-\theta_{t_1})A}B\,dX_{t_1}\big)\otimes\cdots\otimes\big(e^{-(\theta_t-\theta_{t_n})A}B\,dX_{t_n}\big)\\
&=\underbrace{\int_s^u\int_s^{t_n}\cdots\int_s^{t_2}\big(e^{-(\theta_t-\theta_{t_1})A}B\,dX_{t_1}\big)\otimes\cdots\otimes\big(e^{-(\theta_t-\theta_{t_n})A}B\,dX_{t_n}\big)}_{(\mathbf{I})}\\
&\quad+\underbrace{\int_u^t\int_u^{t_n}\cdots\int_u^{t_2}\big(e^{-(\theta_t-\theta_{t_1})A}B\,dX_{t_1}\big)\otimes\cdots\otimes\big(e^{-(\theta_t-\theta_{t_n})A}B\,dX_{t_n}\big)}_{(\mathbf{II})}\\
&\quad+\underbrace{\sum_{k=1}^{n-1}\int_u^t\int_u^{t_n}\cdots\int_u^{t_{k+2}}\int_s^u\int_s^{t_k}\cdots\int_s^{t_2}\big(e^{-(\theta_t-\theta_{t_1})A}B\,dX_{t_1}\big)\otimes\cdots\otimes\big(e^{-(\theta_t-\theta_{t_n})A}B\,dX_{t_n}\big)}_{(\mathbf{III})}
\end{align*}
where $(\mathbf{I})=D^{\theta_t-\theta_u}_AS_{\mathbb{A}}(X;\theta)^{(n)}_{s,u}$ and $(\mathbf{II})=S_{\mathbb{A}}(X;\theta)^{(n)}_{u,t}$. We then expand $(\mathbf{III})$ as
\begin{align*}
(\mathbf{III})=\sum_{k=1}^{n-1}&\underbrace{\Big(\int_s^u\int_s^{t_k}\cdots\int_s^{t_2}\big(e^{-(\theta_t-\theta_{t_1})A}B\,dX_{t_1}\big)\otimes\cdots\otimes\big(e^{-(\theta_t-\theta_{t_k})A}B\,dX_{t_k}\big)\Big)}_{(\mathbf{IV})}\\
&\otimes\underbrace{\Big(\int_u^t\int_u^{t_n}\cdots\int_u^{t_{k+2}}\big(e^{-(\theta_t-\theta_{t_{k+1}})A}B\,dX_{t_{k+1}}\big)\otimes\cdots\otimes\big(e^{-(\theta_t-\theta_{t_n})A}B\,dX_{t_n}\big)\Big)}_{(\mathbf{V})},
\end{align*}
where $(\mathbf{V})=S_{\mathbb{A}}(X;\theta)^{(n-k)}_{u,t}$. For $(\mathbf{IV})$, it remains to introduce the correct factors involving $u$:
\begin{align*}
(\mathbf{IV})&=\int_s^u\int_s^{t_k}\cdots\int_s^{t_2}e^{-(\theta_t-\theta_u)A}e^{-(\theta_u-\theta_{t_1})A}B\,dX_{t_1}\otimes\cdots\otimes e^{-(\theta_t-\theta_u)A}e^{-(\theta_u-\theta_{t_k})A}B\,dX_{t_k}\\
&=\big(e^{-(\theta_t-\theta_u)A}\big)^{\otimes k}\int_s^u\int_s^{t_k}\cdots\int_s^{t_2}e^{-(\theta_u-\theta_{t_1})A}B\,dX_{t_1}\otimes\cdots\otimes e^{-(\theta_u-\theta_{t_k})A}B\,dX_{t_k}\\
&=D^{\theta_t-\theta_u}_AS_{\mathbb{A}}(X;\theta)^{(k)}_{s,u}.
\end{align*}
Putting everything together, we get
\[
S_{\mathbb{A}}(X;\theta)^{(n)}_{s,t}=\sum_{k=0}^nD^{\theta_t-\theta_u}_AS_{\mathbb{A}}(X;\theta)^{(k)}_{s,u}\otimes S_{\mathbb{A}}(X;\theta)^{(n-k)}_{u,t}.
\]
\end{proof}
The above proof will not hold for $p\in(1,2)$ since it makes use of Fubini's theorem. We will extend the statement by using the continuity of the truncated EWS in the $p$-variation topology (Corollary~\ref{cor:continuity}), which allows us to pass to the limit from bounded variation approximations.
\begin{corollary}\label{cor:chenp}
Let $X\in\mathcal{V}^p([0,T];V)$ for $p<2$, and let $\theta$ be a continuous clock on $\mathcal{V}^{p'}([0,T];V)$ for some $p'\in(p,2)$. Then for any $0\leq s\leq u\leq t\leq T$, the exponentially-weighted signature satisfies
\[
S_{\mathbb{A}}(X;\theta)_{s,t}=\big(D^{\theta_t-\theta_u}_AS_{\mathbb{A}}(X;\theta)_{s,u}\big)\otimes S_{\mathbb{A}}(X;\theta)_{u,t}.
\]
\end{corollary}

\begin{proof}
Fix $s\leq u\leq t$ and let $p'$ be as in the statement. By the density of smooth (hence finite variation) paths in $\mathcal{V}^p([0,T];V)$ for the $p'$-variation topology \citep{frizmultidimensional}, there exists a sequence $(X^{(m)})_{m\geq0}$ of finite variation paths such that $X^{(m)}\to X$ in $p'$-variation as $m\to\infty$.

By Lemma~\ref{lem:chen}, Chen's identity holds for each finite variation path $X^{(m)}$:
\[
S_{\mathbb{A}}(X^{(m)};\theta)_{s,t}=\Big(D^{\theta^{(m)}_t-\theta^{(m)}_u}_AS_{\mathbb{A}}(X^{(m)};\theta)_{s,u}\Big)\otimes S_{\mathbb{A}}(X^{(m)};\theta)_{u,t},
\]
where $\theta^{(m)}:=\theta(X^{(m)})$ denotes the clock evaluated along $X^{(m)}$. Fix an arbitrary depth $n\geq0$. Projecting Chen's identity to depth $n$ gives
\[
S_{\mathbb{A}}(X^{(m)};\theta)^{(n)}_{s,t}=\sum_{k=0}^n\Big(D^{\theta^{(m)}_t-\theta^{(m)}_u}_AS_{\mathbb{A}}(X^{(m)};\theta)_{s,u}\Big)^{(k)}\otimes\Big(S_{\mathbb{A}}(X^{(m)};\theta)_{u,t}\Big)^{(n-k)}.
\]
We now pass to the limit as $m\to\infty$. By Corollary~\ref{cor:continuity}, the truncated EWS is continuous in the $p'$-variation topology, so
\begin{align*}
\pi_n\big(S_{\mathbb{A}}(X^{(m)};\theta)_{s,t}\big)&\to\pi_n\big(S_{\mathbb{A}}(X;\theta)_{s,t}\big),\\
\pi_k\big(S_{\mathbb{A}}(X^{(m)};\theta)_{s,u}\big)&\to\pi_k\big(S_{\mathbb{A}}(X;\theta)_{s,u}\big)\quad\text{for each }k\leq n,\\
\pi_{n-k}\big(S_{\mathbb{A}}(X^{(m)};\theta)_{u,t}\big)&\to\pi_{n-k}\big(S_{\mathbb{A}}(X;\theta)_{u,t}\big)\quad\text{for each }k\leq n.
\end{align*}
Since the clock is continuous, $\theta^{(m)}\to\theta$ in $p'$-variation, and in particular uniformly: $\theta^{(m)}_r\to\theta_r$ for all $r\in[0,T]$. In particular, $\theta^{(m)}_t-\theta^{(m)}_u\to\theta_t-\theta_u$. The flow operator $D^r_A$ depends continuously on $r$: for any fixed $\mathbf{a}\in T^n(W)$, the map $r\mapsto D^r_A\mathbf{a}$ is continuous since $r\mapsto e^{-rA}$ is continuous in operator norm. Combined with continuity in the tensor algebra argument, we have
\[
D^{\theta^{(m)}_t-\theta^{(m)}_u}_A\pi_k\big(S_{\mathbb{A}}(X^{(m)};\theta)_{s,u}\big)\to D^{\theta_t-\theta_u}_A\pi_k\big(S_{\mathbb{A}}(X;\theta)_{s,u}\big)
\]
for each $k\leq n$. Since the tensor product $\otimes\colon T^k(W)\times T^{n-k}(W)\to T^n(W)$ is continuous, taking limits in the projected Chen identity yields
\[
S_{\mathbb{A}}(X;\theta)^{(n)}_{s,t}=\sum_{k=0}^n\Big(D^{\theta_t-\theta_u}_AS_{\mathbb{A}}(X;\theta)_{s,u}\Big)^{(k)}\otimes\big(S_{\mathbb{A}}(X;\theta)_{u,t}\big)^{(n-k)}=\Big(\big(D^{\theta_t-\theta_u}_AS_{\mathbb{A}}(X;\theta)_{s,u}\big)\otimes S_{\mathbb{A}}(X;\theta)_{u,t}\Big)^{(n)}.
\]
Since $S_{\mathbb{A}}(X;\theta)_{s,t}$ and $\big(D^{\theta_t-\theta_u}_AS_{\mathbb{A}}(X;\theta)_{s,u}\big)\otimes S_{\mathbb{A}}(X;\theta)_{u,t}$ agree at every depth $n\geq0$, they are equal as elements of $T((W))$.
\end{proof}

\begin{remark}\label{rem:scan}
Beyond its algebraic significance, the modified Chen identity produces the associative operation of the parallel associative scan of Section~\ref{sec:ews-layer}, which aggregates local EWS pieces in $O(\log J)$ parallel steps rather than $O(J)$ sequential ones, for $J$ observations.
\end{remark}

\begin{lemma}\label{lem:star}
On $\mathbb{R}\times T((W))$ define $(r_1,S_1)\star(r_2,S_2)=\big(r_1+r_2,\,D^{r_2}_A(S_1)\otimes S_2\big)$. Then $\star$ is associative with identity $(0,\mathbf{1})$, and $\pi_N\big(D^{r_2}_A(S_1)\otimes S_2\big)=\pi_N\big(D^{r_2}_A(\pi_NS_1)\otimes\pi_NS_2\big)$ for every $N$. For times $0\leq t_0<t_1<\dots<t_J\leq T$ and $g_j=\big(\theta_{t_j}-\theta_{t_{j-1}},\,S_{\mathbb{A}}(X;\theta)_{t_{j-1},t_j}\big)$,
\[
g_1\star\dots\star g_j=\big(\theta_{t_j}-\theta_{t_0},\,S_{\mathbb{A}}(X;\theta)_{t_0,t_j}\big),\qquad j=1,\dots,J.
\]
\end{lemma}

\begin{proof}
By Lemma~\ref{lem:flowgroup} and the homomorphism property of $D^{r_3}_A$,
\[
\big((r_1,S_1)\star(r_2,S_2)\big)\star(r_3,S_3)=\big(r_1+r_2+r_3,\,D^{r_2+r_3}_A(S_1)\otimes D^{r_3}_A(S_2)\otimes S_3\big)=(r_1,S_1)\star\big((r_2,S_2)\star(r_3,S_3)\big),
\]
and $(0,\mathbf{1})$ is an identity since $D^0_A=\mathrm{Id}$ and $D^r_A(\mathbf{1})=\mathbf{1}$. The truncation identity holds because $D^{r_2}_A$ is level-preserving and the level-$n$ part of a product involves only levels up to $n$ of each factor. The last claim follows simply by induction on $j$.
\end{proof}

\paragraph{Linearisation of shuffle product.} We now establish the fact that products of linear functionals of the EWS can be linearised via the shuffle product (Theorem~\ref{thm:structure}(i)). This follows immediately by reducing the EWS over an interval to the signature.

\begin{lemma}\label{lem:shuffle}
Let $X\in\mathcal{V}^p([0,T];V)$ for $p<2$ and $A\in L(W,W)$ bounded. Then for all $[s,t]\subseteq[0,T]$ and $l_1,l_2\in T(W^\star)$,
\begin{equation}\label{eq:ews-shuffle}
\langle l_1,S_{\mathbb{A}}(X;\theta)_{s,t}\rangle\langle l_2,S_{\mathbb{A}}(X;\theta)_{s,t}\rangle=\langle l_1\shuffle l_2,S_{\mathbb{A}}(X;\theta)_{s,t}\rangle.
\end{equation}
\end{lemma}

\begin{proof}
By Proposition~\ref{prop:reweighted}, $S_{\mathbb{A}}(X;\theta)_{s,t}=S(Z^{[t]})_{s,t}$. The claim is then Theorem~\ref{thm:shuffle} applied to the re-weighted path $Z^{[t]}$.
\end{proof}

\paragraph{Reparametrisation.} The signature is invariant under reparametrisation of the path, while for the EWS, this property depends on the clock.

\begin{proposition}\label{prop:reparam}
Let $X\in\mathcal{V}^p([0,T];V)$ for $p<2$ and $\varphi$ be an increasing homeomorphism of $[0,T]$. Suppose that $\theta(X\circ\varphi)=\theta(X)\circ\varphi$. Then for all $[s,t]\subseteq[0,T]$,
\[
S_{\mathbb{A}}(X\circ\varphi;\theta)_{s,t}=S_{\mathbb{A}}(X;\theta)_{\varphi(s),\varphi(t)},
\]
and in particular $S_{\mathbb{A}}(X\circ\varphi;\theta)_{0,T}=S_{\mathbb{A}}(X;\theta)_{0,T}$.
\end{proposition}

\begin{proof}
Denote $\tilde X=X\circ\varphi$ and $\tilde\theta=\theta(\tilde X)=\theta(X)\circ\varphi$, and let $\tilde Z^{[t]}$ be the re-weighted path of $\tilde X$. For $\tau\in[s,t]$,
\[
\tilde Z^{[t]}_\tau=\int_s^\tau e^{-(\tilde\theta_t-\tilde\theta_v)A}B\,d\tilde X_v=\int_s^\tau e^{-(\theta_{\varphi(t)}-\theta_{\varphi(v)})A}B\,dX_{\varphi(v)}=\int_{\varphi(s)}^{\varphi(\tau)}e^{-(\theta_{\varphi(t)}-\theta_u)A}B\,dX_u=Z^{[\varphi(t)]}_{\varphi(\tau)},
\]
where the third equality uses the substitution $u=\varphi(v)$. Thus, $\tilde Z^{[t]}=Z^{[\varphi(t)]}\circ\varphi$ on $[s,t]$. By Proposition~\ref{prop:reweighted} and the invariance of the signature under reparametrisation \citep{lyons2007differential},
\[
S_{\mathbb{A}}(\tilde X;\theta)_{s,t}=S(\tilde Z^{[t]})_{s,t}=S(Z^{[\varphi(t)]})_{\varphi(s),\varphi(t)}=S_{\mathbb{A}}(X;\theta)_{\varphi(s),\varphi(t)}.
\]
\end{proof}

The variation clock satisfies the required condition for all $\varphi$, as $\|X\circ\varphi\|_{1,[0,t]}=\|X\|_{1,[0,\varphi(t)]}$. However, the time clock does not, since trivially $\theta(X\circ\varphi)_t=t$ holds whereas $(\theta(X)\circ\varphi)_t=\varphi(t)$. Therefore, the EWS with time clock depends on the parametrisation of the input path. For instance, with $V=W=\mathbb{R}$, $A=B=1$, $T=1$, $X_u=u$ and $\varphi(u)=u^2$, depth one gives $S_{\mathbb{A}}(X;\theta)^{(1)}_{0,1}=\int_0^1e^{-(1-u)}\,du=1-e^{-1}$, whereas $S_{\mathbb{A}}(X\circ\varphi;\theta)^{(1)}_{0,1}=\int_0^1e^{-(1-v)}\,2v\,dv=2e^{-1}$. 

LTI SSMs, which are depth-one EWS with the time clock (Section~\ref{sec:ssm}), thus depend on the parametrisation of their input, whereas the EWS is invariant under reparametrisation whenever its clock satisfies the condition of Proposition~\ref{prop:reparam}. Among the clocks of Table~\ref{tab:families}, the variation clock and the gated clock with $b=0$ satisfy it, while the time, blended and clamped clocks do not. The following proposition shows that every clock built from path increments with $f(cv)=cf(v)$ satisfies the condition.

\begin{proposition}\label{prop:increment-clock}
Let $X$ be a piecewise linear path with observation times $t_0<\dots<t_J$, let  $\theta$ be the clock of the form $\theta_{t_j}=\sum_{i=1}^jf(\Delta X_i)$, and suppose that $f\colon V\to\mathbb{R}$ satisfies $f(cv)=cf(v)$ for all $c>0$. Then $\theta(X\circ\varphi)=\theta(X)\circ\varphi$ for every piecewise linear increasing homeomorphism $\varphi$ of $[0,T]$.
\end{proposition}
\begin{proof}
Write $\tilde X=X\circ\varphi$. Define $0=s_0<s_1<\dots<s_K=T$ to be the points of the set $\{\varphi^{-1}(t_0),\dots,\varphi^{-1}(t_J)\}\cup\{\text{breakpoints of }\varphi\text{'s piecewise linear interpolation}\}$. Thus, $\varphi$ is linear on each $[s_{k-1},s_k]$ with image inside a single interval $[t_{j-1},t_j]$, on which $X$ is linear. Thus, $\tilde X$ is linear on $[s_{k-1},s_k]$, with increment $c_k\Delta X_j$, where
\[
c_k=\frac{\varphi(s_k)-\varphi(s_{k-1})}{t_j-t_{j-1}}>0.
\]
Treating each $s_k$ as an observation time of $\tilde X$, the assumption on $f$ gives that the piece $[s_{k-1},s_k]$ advances the clock of $\tilde X$ by $f(c_k\Delta X_j)=c_kf(\Delta X_j)$. The pieces covering a whole interval $[t_{i-1},t_i]$ have $c_k$ summing to one, and together advance the clock by $f(\Delta X_i)$, while the pieces covering $[t_{j-1},\varphi(s_k)]$ advance it by $\frac{\varphi(s_k)-t_{j-1}}{t_j-t_{j-1}}f(\Delta X_j)$. Hence,
\[
\theta(\tilde X)_{s_k}=\theta(X)_{t_{j-1}}+\frac{\varphi(s_k)-t_{j-1}}{t_j-t_{j-1}}\,f(\Delta X_j)=\theta(X)_{\varphi(s_k)},
\]
where the last equality holds because $\theta(X)$ is linear on $[t_{j-1},t_j]$. Both $\theta(\tilde X)$ and $\theta(X)\circ\varphi$ are linear on each $[s_{k-1},s_k]$ and agree at every $s_k$, and are therefore equal.
\end{proof}

\subsection{Uniqueness and universal approximation}\label{app:ews-uniqueness}

The signature of a path uniquely determines the path up to tree-like equivalence, and when the path includes a strictly monotone component this reduces to uniqueness up to translation and reparametrisation \citep{BOEDIHARDJO2016720,hambly2010uniqueness}. While the EWS is the signature of a re-weighted path, distinguishing two re-weighted paths requires them not to be tree-like equivalent, for instance by having monotone components, and it is not obvious that this property is inherited from the input paths, since re-weighting mixes channels through $\mathbb{A}$. This section characterises when two paths yield the same EWS for a given $\mathbb{A}$ and clock, establishes uniqueness under structural assumptions on $A$ and $B$, and proves universal approximation in full generality.

\paragraph{Characterisation of EWS equivalence.} We begin by characterising the equivalence relation induced by the EWS. Two paths are $(\mathbb{A};\theta)$-equivalent if they have the same EWS for a given $\mathbb{A}$ and clock $\theta$. Since the EWS of a path can be viewed as the classical signature of its corresponding re-weighted path, this equivalence is precisely the pull-back of tree-like equivalence through the re-weighting map $X \mapsto Z^{[\cdot]}$.

\begin{definition}\label{def:equivalence}
Let $X,Y\in\mathcal{V}^p([0,T];V)$ for $p<2$, let $\mathbb{A}=(A,B)$ with $A\in L(W,W)$ and $B\in L(V,W)$ bounded linear operators, and let $\theta$ be a clock. We say $X$ and $Y$ are $(\mathbb{A};\theta)$-equivalent over $[s,t]\subseteq[0,T]$, denoted $X\sim_{\mathbb{A};\theta}Y$, if
\[
S_{\mathbb{A}}(X;\theta)_{s,t}=S_{\mathbb{A}}(Y;\theta)_{s,t}.
\]
\end{definition}

By Proposition~\ref{prop:reweighted}, $X\sim_{\mathbb{A};\theta}Y$ if and only if the re-weighted paths $Z^{X,[t]}$ and $Z^{Y,[t]}$ of $X$ and $Y$ are tree-like equivalent \citep{hambly2010uniqueness}. For $A=0$ and $B$ the identity this is tree-like equivalence of $X$ and $Y$; below we give structural conditions under which it determines the path.

\paragraph{Uniqueness of the EWS for structured $A$.} As for the signature of a path with a time channel, uniqueness follows if the re-weighted path inherits a monotone channel from the input path. This is a sufficient condition only. We thus consider paths $X\in\mathcal{V}^p([0,T];V)$ whose first component $c=\pi_1(X)$ is strictly monotone. Such a component can be obtained by augmenting a path with any clock of Table~\ref{tab:families} whose increments are strictly positive, such as time or the blended clock (the clock that performs the re-weighting remains arbitrary). The re-weighted path then inherits a monotone channel under the following structure on $A$ and $B$.

\begin{lemma}\label{lem:monotone}
Let $X\in\mathcal{V}^p([0,T];V)$ for $p<2$ with $c=\pi_1(X)$ strictly monotone. Let $[s,t]\subseteq[0,T]$ and $\pi\colon W\to\mathbb{R}$ be a coordinate projection. Suppose that $A\in L(W,W)$ and $B\in L(V,W)$ satisfy
\begin{equation}\label{eq:monotone-cond}
\pi\circ e^{-rA}=e^{-\alpha r}\cdot\pi\quad\forall r\in\mathbb{R},\qquad\pi\circ B=\beta\cdot\pi_1,
\end{equation}
with $\alpha\in\mathbb{R}$ and $\beta\neq0$. Then $\zeta=\pi(Z^{[t]})$ is strictly monotone, with
\[
\zeta_\tau=\beta\int_s^\tau e^{-\alpha(\theta_t-\theta_u)}\,dc_u.
\]
\end{lemma}

\begin{proof}
For any $v\in V$, the conditions (\ref{eq:monotone-cond}) say that this coordinate of $e^{-rA}Bv$ depends only on the first component of $v$:
\[
\pi(e^{-rA}Bv)=e^{-\alpha r}\cdot\beta\,\pi_1(v).
\]
Applying this to the re-weighted path gives
\[
\zeta_\tau=\pi\Big(\int_s^\tau e^{-(\theta_t-\theta_u)A}B\,dX_u\Big)=\int_s^\tau\pi\big(e^{-(\theta_t-\theta_u)A}B\,dX_u\big)=\beta\int_s^\tau e^{-\alpha(\theta_t-\theta_u)}\,dc_u.
\]
Since $e^{-\alpha(\theta_t-\theta_u)}>0$ for all $u\in[s,\tau]$, $\beta\neq0$ and $c$ is strictly monotone, $\zeta$ is strictly monotone in $\tau$.
\end{proof}

\begin{remark}\label{rem:monotone}
The condition on $A$ holds whenever $\pi\circ A=\alpha\cdot\pi$. For $W=\mathbb{R}^m$, $V=\mathbb{R}^d$ and $\pi$ the $k$-th coordinate, the two conditions say that the $k$-th row of $A$ is $\alpha$ in position $k$ and zero elsewhere, and that the $k$-th row of $B$ is $(\beta,0,\dots,0)$.
\end{remark}

\begin{corollary}\label{cor:injective}
Let $X,Y\in\mathcal{V}^p([0,T];V)$ with $p<2$ have a strictly monotone component, the same starting point, and let $B$ be injective. Suppose that the first two coordinates of $W$ satisfy the conditions of Lemma~\ref{lem:monotone}, with rates $\alpha_1=0$ and $\alpha_2=\alpha\neq0$. Then $S_{\mathbb{A}}(X;\theta)_{0,T}=S_{\mathbb{A}}(Y;\theta)_{0,T}$ implies that $X=Y\circ\varphi$ for some increasing homeomorphism $\varphi$ of $[0,T]$.
\end{corollary}

\begin{proof}
By Proposition~\ref{prop:reweighted}, $S_{\mathbb{A}}(X;\theta)_{0,T}=S_{\mathbb{A}}(Y;\theta)_{0,T}$ implies that $S(Z^{X,[T]})_{0,T}=S(Z^{Y,[T]})_{0,T}$. By Lemma~\ref{lem:monotone} applied to the first coordinate, both re-weighted paths have a strictly monotone channel, and thus $Z^{X,[T]}=Z^{Y,[T]}\circ\psi$ for some increasing homeomorphism $\psi$ of $[0,T]$ \citep{BOEDIHARDJO2016720}, the translation vanishing since both paths start at zero. Write $\theta^X=\theta(X)$, $\theta^Y=\theta(Y)$, $c^X=\pi_1(X)$ and $c^Y=\pi_1(Y)$. The first coordinates give $c^X_\tau-c^X_0=c^Y_{\psi(\tau)}-c^Y_0$, and since the first components start at the same point, $c^X=c^Y\circ\psi$. The second coordinates give, after substituting $u=\psi(v)$,
\[
\int_0^\tau\Big(e^{-\alpha(\theta^X_T-\theta^X_v)}-e^{-\alpha(\theta^Y_T-\theta^Y_{\psi(v)})}\Big)\,dc^X_v=0,\qquad\tau\in[0,T].
\]
Since $c^X$ is strictly monotone and the integrand is continuous, the integrand vanishes, and since $\alpha\neq0$, $\theta^X_T-\theta^X_v=\theta^Y_T-\theta^Y_{\psi(v)}$ for every $v$. The weights then agree, $e^{-(\theta^Y_T-\theta^Y_{\psi(v)})A}=e^{-(\theta^X_T-\theta^X_v)A}$, and $Z^{X,[T]}=Z^{Y,[T]}\circ\psi$ gives
\[
\int_0^\tau e^{-(\theta^X_T-\theta^X_v)A}B\,d\big(X-Y\circ\psi\big)_v=0,\qquad\tau\in[0,T].
\]
Since each $e^{-(\theta^X_T-\theta^X_v)A}$ is invertible and $B$ is injective, $X-Y\circ\psi$ is constant, and it vanishes at $0$. Taking $\varphi=\psi$ completes the proof.
\end{proof}

These conditions are mild and only prevent the other channels from influencing the two coordinates that read the monotone channel; this leaves oscillatory modes, coupled decay between data channels and clock-dependent decay available; moreover, the monotone channel may still enter every other coordinate of $W$ through $B$.

\paragraph{Universal approximation theorem.} Despite the subtleties of uniqueness, universal approximation holds in full generality. The EWS is always a universal feature map for the class of functions it can distinguish; that is, functions constant on $(\mathbb{A};\theta)$-equivalence classes. When uniqueness holds, this class is all continuous functions.

\begin{lemma}\label{lem:universal-classes}
Let $K\subset\mathcal{V}^p([0,T];V)$ for $p<2$ be a compact set of paths and let $\theta$ be a continuous clock. For any continuous $F\colon K\to\mathbb{R}$ satisfying
\[
X\sim_{\mathbb{A};\theta}Y\ \Rightarrow\ F(X)=F(Y),
\]
and any $\epsilon>0$, there exists an $l\in T(W^\star)$ such that
\begin{equation}\label{eq:universal}
\sup_{X\in K}\big|F(X)-\langle l,S_{\mathbb{A}}(X;\theta)_{0,T}\rangle\big|<\epsilon.
\end{equation}
\end{lemma}

\begin{proof}
Let $\mathcal{A}=\{X\mapsto\langle l,S_{\mathbb{A}}(X;\theta)_{0,T}\rangle:l\in T(W^\star)\}$. Each function in $\mathcal{A}$ involves finitely many levels and is continuous on $K$ by Corollary~\ref{cor:continuity}. The empty word gives the constants, and $\mathcal{A}$ is closed under multiplication by Lemma~\ref{lem:shuffle}, making it a subalgebra of $C(K,\mathbb{R})$. If $S_{\mathbb{A}}(X;\theta)_{0,T}\neq S_{\mathbb{A}}(Y;\theta)_{0,T}$, the two differ at some depth $n$, and the Hahn--Banach theorem \citep{rudin1991} gives $\phi\in(W^{\otimes n})^\star$, placed at depth $n$ in $T(W^\star)$, that separates them. The algebra $\mathcal{A}$ therefore separates the points of $K/\!\sim_{\mathbb{A};\theta}$, and the Stone--Weierstrass theorem gives (\ref{eq:universal}) for every $F$ as in the statement.
\end{proof}

By Corollary~\ref{cor:injective}, when $A$ and $B$ satisfy its conditions and $K$ consists of paths with a strictly monotone  component and a common starting point, every continuous $F$ with $F(X\circ\varphi)=F(X)$ whenever $X,X\circ\varphi\in K$ satisfies the hypothesis of Lemma~\ref{lem:universal-classes}, and is therefore approximated uniformly by linear functionals of the EWS. When the first component of every path in $K$ is time, $X\circ\varphi$ lies in $K$ only for $\varphi$ the identity, and every continuous $F$ on $K$ qualifies. This mirrors the signature, whose linear functionals are universal on unparametrised paths and on all paths once a time channel is added.

\subsection{Spectral properties at depth one}\label{app:ews-spectral}

Throughout this subsection, we work with paths in Euclidean space as in the main text. 
Section~\ref{sec:ews-depth-one} states that a real diagonal generator approximates the damped oscillatory kernels of a rotation generator only with much larger width; Proposition~\ref{prop:oscillatory} quantifies this. We also prove Proposition~\ref{prop:jordan} alongside the claims that follow it.

\begin{proposition}\label{prop:oscillatory}
Let $A=\mathrm{diag}(\lambda_1,\dots,\lambda_m)$ be real and $c\in\mathbb{R}^m$. The linear readout $\langle c,S_{\mathbb{A}}(X;\theta)^{(1)}_{0,t}\rangle$ applies to each channel $X^j$ the kernel $k_j(r)=\sum_{k=1}^mc_kB_{kj}e^{-\lambda_kr}$ in elapsed clock time $r$. If $\sup_{r\geq0}|k_j(r)-e^{-ar}\cos(\omega r)|\leq\varepsilon$ for some $j$, with $a,\omega>0$ and $\varepsilon\in(0,1)$, then $m\geq\frac{\omega}{\pi a}\log\frac1\varepsilon$.
\end{proposition}

\begin{proof}
Set
\begin{equation*}
L:=\frac{\omega}{\pi a}\log\frac1\varepsilon\qquad\text{and}\qquad r_n:=\frac{n\pi}{\omega}.
\end{equation*}
For every integer $n\geq0$ with $n<L$,
\begin{equation*}
e^{-ar_n}\cos(\omega r_n)=(-1)^ne^{-a\pi n/\omega},
\end{equation*}
whose modulus is greater than $\varepsilon$. Since
\begin{equation*}
\big|k_j(r_n)-(-1)^ne^{-a\pi n/\omega}\big|\leq\varepsilon,
\end{equation*}
it follows that $k_j(r_n)$ has sign $(-1)^n$. Let
\begin{equation*}
M:=\#\{n\in\mathbb{Z}_{\geq0}:n<L\}=\lceil L\rceil.
\end{equation*}
Since $k_j$ has opposite signs at consecutive points, by continuity it has at least $M-1$ distinct zeros on $[0,\infty)$. After combining terms with equal $\lambda_k$, $k_j$ is a linear combination of at most $m$ exponentials with distinct real exponents. Moreover, it is non-zero, since
\begin{equation*}
k_j(0)\geq1-\varepsilon>0.
\end{equation*}
A non-zero real exponential sum with at most $m$ distinct exponents has at most $m-1$ zeros on $[0,\infty)$. This follows by induction on $m$, since dividing by one exponential and differentiating removes one term, while between any two distinct zeros there is a zero of the derivative. Therefore
\begin{equation*}
M-1\leq m-1,
\end{equation*}
and hence
\begin{equation*}
m\geq M=\lceil L\rceil\geq L=\frac{\omega}{\pi a}\log\frac1\varepsilon,
\end{equation*}
as claimed.
\end{proof}

Proposition~\ref{prop:oscillatory} shows that a real diagonal generator can approximate a damped oscillation to within $\varepsilon$ only with width $m\gtrsim\frac{\omega}{a}\log\frac1\varepsilon$, where $\omega$ and $a$ are the frequency and the damping rate of the oscillation respectively. Proposition~\ref{prop:jordan} states that a Jordan block generator produces depth-one coordinates with polynomial--exponential kernels, which no diagonalisable generator reaches, and the discussion that follows it compares these kernels with the signature and the EFM at deeper levels. Proposition~\ref{prop:jordan-app} restates Proposition~\ref{prop:jordan} as parts (i) and (ii), and makes this comparison precise as part (iii). For a channel $X^i$ of $X$, $\lambda\geq0$ and an integer $\ell\geq0$, let
\begin{equation*}
g_\ell(X)_t=\int_0^t\frac{(t-u)^\ell}{\ell!}e^{-\lambda(t-u)}\,dX^i_u.
\end{equation*}

\begin{proposition}\label{prop:jordan-app}
Let $\theta_t=t$.
\begin{enumerate}
\item[(i)] Consider the EWS $S_{\mathbb{A}}(X;\theta)$ with a generator $A$ in Jordan normal form, and let $J=\lambda I+N\in\mathbb{R}^{(K+1)\times(K+1)}$ be one of its blocks with $\lambda \geq 0$, $Ne_{k+1} = e_k$ for $k = 1,\dots,K$ and $Ne_1 =0$. Suppose that the rows of $B$ corresponding to the coordinates of $J$ form the matrix $e_{K+1}e_i^\top\in\mathbb{R}^{(K+1)\times d}$, which lifts the channel $X^i$ into the last coordinate of the block. Then $g_\ell(X)_t$ is, up to sign, a depth-one coordinate of $S_{\mathbb{A}}(X;\theta)_{0,t}$ for every $\ell\leq K$.
\item[(ii)] For $\ell\geq1$, there is no real generator $A$, of any size, that is diagonalisable over $\mathbb{C}$, and no $B$ for which $g_\ell(X)_t$ is a linear readout of depth one of $S_{\mathbb{A}}(X;\theta)_{0,t}$.
\item[(iii)] Let $\ell\geq1$. For $p\geq1$, a polynomial of degree $p$ in $g_\ell$ is a linear functional of the EWS of (i), with $K\geq\ell$, at depth $p$. When $\lambda=0$, $g_\ell(X)_t$ is a linear functional of the signature of $(t,X)$ at depth $\ell+1$ and at no smaller depth, and when $\lambda>0$, it is not one at any finite depth. For every $\lambda\geq0$, it is not a linear functional of the EFM of $(t,X)$ at any finite depth.
\end{enumerate}
\end{proposition}

\begin{proof}
(i) Recall that $S_{\mathbb{A}}(X;\theta)^{(1)}_{0,t}=\int_0^te^{-(t-u)A}B\,dX_u$. Since $A$ is in Jordan normal form, $e^{-rA}$ is block diagonal, with $e^{-rJ}$ as the block corresponding to $J$. The components of depth one in the coordinates of $J$ are therefore
\begin{equation*}
\int_0^te^{-(t-u)J}\,(B\,dX_u)_J,
\end{equation*}
where $(B\,dX_u)_J$ denotes the components of $B\,dX_u$ in these coordinates, given by the rows of $B$ corresponding to $J$. These rows are $e_{K+1}e_i^\top$, and
\begin{equation*}
(B\,dX_u)_J=e_{K+1}e_i^\top\,dX_u=e_{K+1}\,dX^i_u.
\end{equation*}
Hence the components of depth one in the coordinates of $J$ are
\begin{equation*}
\int_0^te^{-(t-u)J}e_{K+1}\,dX^i_u.
\end{equation*}
Now $J=\lambda I+N$, and $\lambda I$ commutes with $N$. Therefore, for $r\geq0$, we have that $e^{-rJ}e_{K+1}=e^{-\lambda r}e^{-rN}e_{K+1}$. Since $Ne_{k+1}=e_k$, we have
$N^\ell e_{K+1}=e_{K+1-\ell},\qquad0\leq\ell\leq K$, and $N^{K+1}e_{K+1}=0$. Hence, the exponential series terminates and gives
\begin{equation*}
e^{-rJ}e_{K+1}=e^{-\lambda r}\sum_{\ell=0}^K\frac{(-r)^\ell}{\ell!}e_{K+1-\ell}.
\end{equation*}
Letting $r=t-u$, the $(K+1-\ell)$-th of these components is thus
\begin{equation*}
(-1)^\ell\int_0^t\frac{(t-u)^\ell}{\ell!}e^{-\lambda(t-u)}\,dX^i_u=(-1)^\ell g_\ell(X)_t,
\end{equation*}
for every $\ell\leq K$, as required.

(ii) Suppose, for a contradiction, that for some $t>0$, some $A\in\mathbb{R}^{m\times m}$ diagonalisable over $\mathbb{C}$, some $B\in\mathbb{R}^{m\times d}$, and some $c\in\mathbb{R}^m$,
\begin{equation}\label{eq:readout-assumption}
\big\langle c,S_{\mathbb{A}}(X;\theta)^{(1)}_{0,t}\big\rangle=g_\ell(X)_t
\end{equation}
for every path $X$. By the formula for depth one,
\begin{equation*}
\big\langle c,S_{\mathbb{A}}(X;\theta)^{(1)}_{0,t}\big\rangle=\sum_{j=1}^d\int_0^tk_j(t-u)\,dX^j_u,\qquad k_j(r)=c^\top e^{-rA}Be_j.
\end{equation*}
Fix $\tau\in[0,t)$ and $\delta\in(0,t-\tau)$. Let $X$ be the path whose channels other than $X^i$ are constant, and whose channel $X^i$ increases linearly on $[\tau,\tau+\delta]$ with a total increase of one, and is constant elsewhere. For this path, (\ref{eq:readout-assumption}) becomes
\begin{equation*}
\frac1\delta\int_\tau^{\tau+\delta}k_i(t-u)\,du=\frac1\delta\int_\tau^{\tau+\delta}\frac{(t-u)^\ell}{\ell!}e^{-\lambda(t-u)}\,du.
\end{equation*}
Letting $\delta\to0$ and using continuity gives
\begin{equation*}
k_i(r)=\frac{r^\ell}{\ell!}e^{-\lambda r},\qquad r\in(0,t].
\end{equation*}
Since $A$ is diagonalisable over $\mathbb{C}$, we may write $A=P\Gamma P^{-1}$, where
\begin{equation*}
\Gamma=\operatorname{diag}(\gamma_1,\dots,\gamma_m),\qquad\gamma_k\in\mathbb{C}.
\end{equation*}
Hence
\begin{equation*}
k_i(r)=c^\top Pe^{-r\Gamma}P^{-1}Be_i=\sum_{k=1}^mp_ke^{-\gamma_kr}
\end{equation*}
for suitable $p_k\in\mathbb{C}$. After combining repeated exponents, $k_i$ is therefore a finite linear combination of exponentials with distinct exponents. For $\ell\geq1$, $r^\ell e^{-\lambda r}/\ell!$ cannot be such a linear combination on any interval of positive length, by the linear independence of exponential monomials. This gives a contradiction and proves (ii). Both parts (i) and (ii) extend to a complex pair $a\pm i\omega$ repeated $K+1$ times, by the same argument applied to the realified block, whose exponential has entries $r^\ell e^{-ar}\cos(\omega r)$ and $r^\ell e^{-ar}\sin(\omega r)$ up to sign and scale.

(iii) First let channel $0$ denote the time channel of $(t,X)$. By (i), $g_\ell(X)_t$ is a depth-one coordinate of the EWS, and by Lemma~\ref{lem:shuffle} a polynomial of degree $p$ in it is a linear functional of the EWS at depth $p$. For $\lambda = 0$,
\begin{equation*}
g_\ell(X)_t=\int_{0<u<s_1<\cdots<s_\ell<t}dX^i_u\,ds_1\cdots ds_\ell,
\end{equation*}
which is the coordinate of the signature of $(t,X)$ associated with the word $(i,0,\dots,0)$ with $\ell$ letters $0$.

For the lower bounds, suppose that $g_\ell(X)_t=\langle\psi,F(X)\rangle$ for every $X$, where $F(X)$ is either the signature or the EFM of $(t,X)$ truncated at depth $N$, and $\psi$ is some linear functional. Take the channels of $X$ other than $X^i$ constant, $\varepsilon \in \mathbb{R}$ and $X^i=\varepsilon Y$ for a smooth path $Y$. For this path, the only non-zero coordinates of $F(X)$ are those whose words are made up only of the letters $0$ and $i$. The coordinate of a word with $n$ letters $i$ and all other letters $0$ equals $\varepsilon^n$ times its value at $\varepsilon=1$ as each letter $i$ contributes a factor of $dX^i_u=\varepsilon\,dY_u$. The coefficient of $\varepsilon$ in $\langle\psi,F(X)\rangle$ therefore comes only from words with exactly one letter $i$. Since $g_\ell(X)_t=\varepsilon\,g_\ell(Y)_t$, comparing the coefficients of $\varepsilon$ gives
\begin{equation*}
\int_0^t\varphi(u)\,dY_u=\int_0^tk(u)\,dY_u,\qquad\varphi(u)=\frac{(t-u)^\ell}{\ell!}e^{-\lambda(t-u)},
\end{equation*}
where $k$ is a linear combination of the kernels of the words $(0,\dots,0,i,0,\dots,0)$, with $q_0$ letters $0$ before $i$ and $q_1$ after, and where $q_0+q_1\leq N-1$. Taking $dY_u=(\varphi-k)(u)\,du$ shows that $\varphi=k$ on $[0,t]$.

For the signature, the coordinate of the word $(0,\dots,0,i,0,\dots,0)$, with $q_0$ letters $0$ before $i$ and $q_1$ after, is an iterated integral in which the letter $i$ integrates $dY_u$ and each letter $0$ integrates time. Integrating out the $q_0$ time variables before $u$ and the $q_1$ time variables after $u$ gives the kernel
\begin{equation*}
\frac{u^{q_0}}{q_0!}\,\frac{(t-u)^{q_1}}{q_1!}.
\end{equation*}
Each such kernel is a polynomial in $u$ of degree $q_0+q_1\leq N-1$. Since $k$ is a linear combination of these kernels, it is thus a polynomial of degree at most $N-1$. When $\lambda=0$, $\varphi$ is a polynomial of degree $\ell$, which forces $N\geq\ell+1$. When $\lambda>0$, $\varphi$ is not a polynomial, and so no $N$ suffices.

For the EFM of $(t,X)$ with generator $\mathrm{diag}(\mu_0,\dots,\mu_d)$ and $\mu_j >0$ for all $j$, the coordinate of the same word is an iterated integral in which the letter $i$ integrates $e^{-\mu_i(t-u)}\,dY_u$ and each letter $0$ integrates $e^{-\mu_0(t-s)}\,ds$. Integrating out the $q_0$ time variables before $u$ and the $q_1$ time variables after $u$ yields the kernel
\begin{equation*}
\frac{1}{q_0!}\Big(\frac{e^{-\mu_0(t-u)}-e^{-\mu_0t}}{\mu_0}\Big)^{q_0}e^{-\mu_i(t-u)}\,\frac{1}{q_1!}\Big(\frac{1-e^{-\mu_0(t-u)}}{\mu_0}\Big)^{q_1}.
\end{equation*}
Each such kernel is a finite linear combination of pure exponentials in $u$. Since $k$ is a linear combination of these kernels, it is thus one as well. As $\ell\geq1$, $\varphi$ contains the term $u^\ell e^{\lambda u}$, which is linearly independent of pure exponentials on $[0,t]$. Hence $\varphi\neq k$, and no $N$ suffices, for any $\lambda\geq0$.
\end{proof}

Therefore, the EFM can only approximate these kernels. At fixed depth, for $\ell =1$, $e^{-\mu_i r}(1-e^{-\mu_0 r})/\mu_0\to re^{-\lambda r}$ when $\mu_0\to0$ and $\mu_i\to\lambda$. At fixed rates, the words $(i,0,\dots,0)$ give $e^{-\mu_i r}$ times polynomials in $1-e^{-\mu_0 r}$, which approximate the kernels uniformly on $[0,T]$ as depth grows. A Jordan block instead gives them at depth one.

\section{The EFM and Volterra signature}\label{app:efm-volterra}

Both the EFM and Volterra signature weight the iterated integrals of the signature, and both agree with the EWS under certain regimes. 

\subsection{The EFM}

\begin{definition}
Let $X\in\mathcal{V}^p_{\mathrm{loc}}((-\infty,T];\mathbb{R}^d)$ for $p<2$,  and $\lambda\in\mathbb{R}^d$ with $\lambda\succ0$, where $\prec$ denotes the entrywise order. Assume further that there exists $\rho\in\mathbb{R}^d$ with $0\prec\rho\prec\lambda$ and $\sup_{u\leq T}|e^{-\rho(T-u)} \odot X_u|<\infty$. The $\lambda$-exponentially fading memory signature of $X$ is $\mathbb{X}^{\lambda}_{s,t}=(1,\mathbb{X}^{\lambda,1}_{s,t},\mathbb{X}^{\lambda,2}_{s,t},\dots)\in T((\mathbb{R}^d))$, where
\begin{equation}\label{eq_efm_coords}
\mathbb{X}^{\lambda,n}_{s,t}=\int^{t}_s\int^{t_n}_s\cdots\int^{t_2}_s e^{-\lambda(t-t_1)}\odot dX_{t_1}\otimes\cdots\otimes e^{-\lambda(t-t_n)}\odot dX_{t_n}\in(\mathbb{R}^d)^{\otimes n},
\end{equation}
$e^{-\lambda r}$ denotes the entrywise exponential $(e^{-\lambda_1r},\dots,e^{-\lambda_dr})$ and $\odot$ the Hadamard product on $\mathbb{R}^d$, $x\odot y=(x_1y_1,\dots,x_dy_d)$.
\end{definition}

On a finite horizon the EFM is the EWS (using non-learnable parameters) with the time clock, the identity lift and $A=\mathrm{diag}(\lambda_1,\dots,\lambda_d)$. Indeed, $e^{-rA}$ is then diagonal with entries $e^{-\lambda_1r},\dots,e^{-\lambda_dr}$, so that $e^{-(t-t_l)A}dX_{t_l}=e^{-\lambda(t-t_l)}\odot dX_{t_l}$ for every $l$, and the iterated integrals of Definition~\ref{def:ews} become those of (\ref{eq_efm_coords}). The infinite horizon is what forces the rates to be positive, since the improper integrals converge only if the distant past is attenuated, and positivity restricts the weights to monotone decay. A spectrum confined to the positive reals is restrictive, and no choice of lift or readout produces the oscillatory kernels of Appendix~\ref{app:ews-spectral}, which need complex eigenvalues, or the polynomial--exponential kernels of Proposition~\ref{prop:jordan}, which need a Jordan block, although both can be approximated (Appendix~\ref{app:ews-spectral}). Appendix~\ref{app:ews-spectral} bounds how slowly a real diagonal generator approximates the oscillatory kernels. The weighting acts on each channel separately, so channels do not interact in the time propagation, and with no lift the number of modes is the number of channels.

\subsection{The Volterra signature}

At depth one, with the time clock, the EWS coincides with the Volterra signature of \citet{hager2026volterrasignature} for the kernel $K(t,s) = e^{-(t-s)A}$: both reduce to $\int_s^t e^{-(t-u)A} dX_u$ (under the convention $\tau = t_{n+1} = t$ in \cite[Definition 2.14]{hager2026volterrasignature}). However, beyond depth one, the two objects diverge. Rewriting \citep[Definition 2.14]{hager2026volterrasignature} component-wise, the depth-$n$ term of the Volterra signature for the word $i_1 \cdots i_n$ is
\begin{equation}
    \mathrm{VSig}(X;K)^{i_1\cdots i_n}_{s,t}
    = \int_s^t \int_s^{t_n} \cdots \int_s^{t_2}
    \sum_{j_1,\ldots,j_n \in \{1,\dots,d\}}
    \left(\prod_{k=1}^{n} K^{i_k, j_k}(t_{k+1}, t_k)\right)\,
    dX^{j_1}_{t_1} \cdots dX^{j_n}_{t_n},
\end{equation}
where the integration is over the simplex $\Delta^n_{s,t} = \{s \leq t_1 \leq \cdots \leq t_n \leq t\}$. In the EWS, the increment $dX_{t_k}$ at each vertex $t_k$ of the simplex is weighted by $e^{-(\theta_t - \theta_{t_k})A}$, which depends on the distance from $t_k$ to the apex $t$. In the Volterra signature, the same increment is weighted by $K(t_{k+1}, t_k)$, which depends on the gap to the next vertex $t_{k+1}$. The EWS weighting is thus global, anchored at the terminal time, and the Volterra weighting local, anchored at the next vertex. While both the EWS and the Volterra signature extend the diagonal operator of the EFM to a general matrix-valued operator, it is precisely the global anchoring to a fixed terminal 
time that makes the EWS the signature of a re-weighted path evaluated at that time (Theorem \ref{thm:structure}(i)), ensuring that the EWS is group-like while the Volterra signature is not \citep[Remark 2.17]{hager2026volterrasignature}.

\section{SSM Conversions}\label{app:ssm}

An SSM computes the recurrence $h_j=\bar Ah_{j-1}+\bar Bx_j$, usually obtained as the discretisation of the continuous dynamics $\dot h_t=-Ah_t+Bx_t$ (written in the sign convention of the EWS). Trained SSMs then either store the discrete pair $(\bar A,\bar B)$ directly, or the continuous pair $(A,B)$ alongside the step size $\Delta$ and the discretisation rule that produces the matrices used in the recurrence. However, in both cases it is the recurrence that is trained and not the continuous system. Section~\ref{sec:ssm} identifies the continuous dynamics of the SSM with a depth-one EWS; that is,  $h_t=S_{\mathbb{A}}(X;\theta)^{(1)}_{0,t}$ for $\mathbb{A}=(A,B)$, the time clock and $X_t=\int_0^tx_s\,ds$. Therefore, if a trained SSM stores its continuous $(A,B)$, constructing an EWS at depth one with $\mathbb{A} =(A,B)$ would produce a model that runs the continuous dynamics of the SSM exactly, rather than the discrete recurrence that it actually computes (unless the discretisation rule is the zero-order hold). Hence, the output of the SSM and the depth-one EWS would not match at the observation times. To reproduce a trained SSM as an EWS, we must instead match its recurrence with what the EWS obeys at the observation times, which we derive first. We then go through the models of Table~\ref{tab:conversion} individually and check the conversions numerically.

As in Section~\ref{sec:ssm}, $x$ is held constant between observations (rectilinear interpolation), equal to $x_j$ on $[t_{j-1},t_j]$, and so both $X$ and $\theta$ become linear on each interval $[t_{j-1}, t_j]$ of the grid (piecewise linear interpolation). We let $\Delta\theta_j=\theta_{t_j}-\theta_{t_{j-1}}$ and $\Delta X_j=X_{t_j}-X_{t_{j-1}}$ for the increments over it, and set
\begin{equation}
G_r(A)=\int_0^r e^{-sA}ds=r\varphi_1(-r A),\qquad \varphi_1(Z)=\int_0^1 e^{sZ}\,ds ,
\end{equation}
where $\varphi_1(Z)=(e^{Z}-I)Z^{-1}$ when $Z$ is invertible and $r \in \mathbb{R}$. Note that $G_r(A)$ integrates the weight $e^{-sA}$ between $0$ and $r$ and $\varphi_1(-r A)$ averages it, while remaining well-defined at $r=0$, which the stopping clocks of Table~\ref{tab:families} reach.

\subsection{The EWS at the observation times}\label{app:observation-times}

The recurrence of the EWS follows from the modified Chen identity at depth one. Between consecutive observations, with $S_j=S_{\mathbb{A}}(X;\theta)^{(1)}_{0,t_j}$, it reads
\begin{equation}\label{eq:app-chen}
S_j=e^{-\Delta\theta_jA}\,S_{j-1}+S_{\mathbb{A}}(X;\theta)^{(1)}_{t_{j-1},t_j}.
\end{equation}
By Definition~\ref{def:ews} the last term is
\begin{align}
S_{\mathbb{A}}(X;\theta)^{(1)}_{t_{j-1},t_j}
&=\int_{t_{j-1}}^{t_j}e^{-(\theta_{t_j}-\theta_u)A}\,B\,dX_u
=\int_0^{1}e^{-(1-\sigma)\Delta\theta_jA}\,B\,\Delta X_j\,d\sigma\notag\\
&=\Big(\int_0^{1}e^{-\sigma\Delta\theta_jA}\,d\sigma\Big)B\,\Delta X_j
=\varphi_1(-\Delta\theta_jA)\,B\,\Delta X_j ,\label{eq:app-segment}
\end{align}
where the second equality uses that $X$ and $\theta$ are linear on the interval, parametrised by $\sigma\in[0,1]$, the third substitutes $\sigma\mapsto1-\sigma$, and the last is the definition of $\varphi_1$. Substituting (\ref{eq:app-segment}) into (\ref{eq:app-chen}) gives
\begin{equation}\label{eq:app-rec}
S_j=e^{-\Delta\theta_jA}S_{j-1}+\varphi_1(-\Delta\theta_jA)\,B\,\Delta X_j .
\end{equation}
Nothing is approximated. When the clock stops, $\Delta\theta_j=0$ and the second term is $B\,\Delta X_j$. For the time clock on the grid $t_j=j\Delta$ of Section~\ref{sec:ssm}, $\Delta\theta_j=\Delta$ and $\Delta X_j=\Delta x_j$, and (\ref{eq:app-rec}) becomes
\begin{equation}\label{eq:app-rec-time}
S_j=e^{-\Delta A}S_{j-1}+\Delta\,\varphi_1(-\Delta A)\,Bx_j=e^{-\Delta A}S_{j-1}+G_\Delta(A)\,Bx_j ,
\end{equation}
which is (\ref{eq:level-one-recurrence}). The LTI models of Table~\ref{tab:conversion} have a fixed step, and are matched against (\ref{eq:app-rec-time}). A selective model computes its step from the input, and is matched against (\ref{eq:app-rec}) with that step as the clock increment.

\subsection{Converting a trained model}

Matching the coefficients of the model's recurrence $h_j=\bar Ah_{j-1}+\bar Bx_j$ with those of (\ref{eq:app-rec-time}), we get
\begin{equation}
    e^{-\Delta A}=\bar A, \qquad G_\Delta(A)B=\bar B,
\end{equation}
which gives 
\begin{equation}
    A=-\Delta^{-1}\log\bar A, \qquad B=G_\Delta(A)^{-1}\bar B.
\end{equation}
This is (\ref{eq:conversion}), and the EWS with these parameters reproduces the trained model at every observation, whatever the model stores and whatever rule produced $(\bar A,\bar B)$. When the model stores continuous parameters $(\hat A,\hat B)$ and discretises them by the zero-order hold, $\bar A=e^{-\Delta\hat A}$ and $\bar B=G_\Delta(\hat A)\hat B$, the matching gives
\begin{equation}
A=-\Delta^{-1}\log e^{-\Delta\hat A}=\hat A,\qquad B=G_\Delta(\hat A)^{-1}G_\Delta(\hat A)\hat B=\hat B,
\end{equation}
up to the branch of the logarithm. Reading off is therefore exact for the zero-order hold case. Under any other rule $(\bar A,\bar B)$ only approximates $(e^{-\Delta\hat A},G_\Delta(\hat A)\hat B)$, and the matched pair differs from the stored one. 

Equation (\ref{eq:conversion}) requires a real logarithm of $\bar A$ and an invertible $G_\Delta(A)$. A real logarithm exists if and only if $\bar A$ is invertible and, for each negative eigenvalue, the Jordan blocks of each size belonging to it occur an even number of times \citep{Culver1966OnTE}, which holds in particular when $\bar A$ has no eigenvalue on $(-\infty,0]$. 

The condition holds for every model in Table~\ref{tab:conversion}, except for LinOSS-IMEX on the boundary of its stability region and for D-LinOSS when its blocks have two negative eigenvalues. Such a block, with distinct eigenvalues, is still reproduced at the observation times by an EWS of twice its dimension. The logarithm is determined only up to multiples of $2\pi i$ in each eigenvalue, and every choice reproduces $\bar A$, which is why frequencies are recovered only modulo $2\pi/\Delta$. We take the one whose angles $\omega\Delta$ lie in $[-\pi,\pi]$. The matrix $G_\Delta(A)$, evaluated at this generator, is then invertible, since it is singular only when a mode without decay turns a whole number of times over one step, that is, when $\Delta A$ has an eigenvalue in $2\pi i\mathbb{Z}\setminus\{0\}$.

\subsection{The models}

Three of the models keep a complex state, which we put in real form before converting. S4D and S5 store a complex continuous $A$ and $B$, and the LRU a complex $\bar A$ and $\bar B$, all three with $A$ or $\bar A$ diagonal, a complex readout $C$ and the real output $y=\mathrm{Re}(Ch)$. The generator and lift of the EWS are real valued, and we therefore realify these matrices, writing each complex coordinate $h_k$ of the state as the pair $(\mathrm{Re}\,h_k,\mathrm{Im}\,h_k)$. The diagonal matrix then becomes block diagonal, with one $2\times2$ block per diagonal entry. The $k$-th row of the lift becomes the two rows of its real and imaginary parts, and the $k$-th column of the readout becomes the two columns of its real part and minus its imaginary part. These real matrices then run the same recurrence, and the conversion then proceeds as above, giving a real generator whose blocks are the rotation blocks of Table~\ref{tab:families}, one per mode, with the eigenvalues listed in Table~\ref{tab:conversion}. The dense complex generator of S4 is realified in the same way, entry by entry.

\paragraph{S4D.} S4D applies to each input channel $c$ a single-input single-output system in which $\bar A$ is complex diagonal, with diagonal entries $e^{\Delta_ca_k}$, the zero-order hold of the stored $a_k=-\gamma_k+i\omega_k$ and step $\Delta_c$. Reading off and realifying gives an EWS generator that is $2\times2$ block diagonal, with blocks of the form $\left(\begin{smallmatrix}\gamma_k&\omega_k\\-\omega_k&\gamma_k\end{smallmatrix}\right)$, and multiplying each channel's blocks and lift rows by $\Delta_c$ makes the whole layer a single EWS with the unit clock.

\paragraph{S5.} S5 applies one multi-input multi-output system to all input channels, in which $\bar A$ is complex diagonal, with diagonal entries $e^{\Delta t_k\Lambda_k}$, the zero-order hold of the stored $\Lambda_k=-\gamma_k+i\omega_k$ and of a step $\Delta t_k$ for each mode. Reading off and realifying with the unit clock gives an EWS generator that is $2\times2$ block diagonal, with blocks of the form $\Delta t_k\left(\begin{smallmatrix}\gamma_k&\omega_k\\-\omega_k&\gamma_k\end{smallmatrix}\right)$ and each lift row multiplied by the step of its mode, which is exact because the modes do not interact. The conversion is the same as for S4D, and the difference lies in the lift and readout, which in S5 are dense and mix all input channels in every mode, whereas in S4D each channel is connected only to its own modes.

\paragraph{S4.} S4 applies to each input channel $c$ a single-input single-output system in which $\bar A=(I-\tfrac{\Delta_c}{2}\mathcal{A})^{-1}(I+\tfrac{\Delta_c}{2}\mathcal{A})$ is the bilinear discretisation of the stored diagonal-plus-low-rank matrix $\mathcal{A}=\Lambda-PQ^*$ and step $\Delta_c$. Reading off would give the EWS generator $-\mathcal{A}$ of the continuous system, which the bilinear recurrence only approximates, and the trained model is instead converted by (\ref{eq:conversion}). This gives a dense EWS generator whose eigenvalues are $-\tfrac{2}{\Delta_c}\operatorname{artanh}(\tfrac{\Delta_c}{2}\mu_k)$ for the eigenvalues $\mu_k$ of $\mathcal{A}$. Multiplying each channel's generator and lift by $\Delta_c$ again makes the whole layer a single EWS with the unit clock.

\paragraph{LRU.} The LRU applies one multi-input multi-output system to all input channels, in which $\bar A$ is complex diagonal and stored directly, with diagonal entries $e^{-\nu_k+i\vartheta_k}$, and there is no continuous system and no step. Applying (\ref{eq:conversion}) with $\Delta=1$, the unit clock, and realifying gives an EWS generator that is $2\times2$ block diagonal, with blocks of the form $\left(\begin{smallmatrix}\nu_k&\vartheta_k\\-\vartheta_k&\nu_k\end{smallmatrix}\right)$, each $\vartheta_k$ taken modulo $2\pi$ in $[-\pi,\pi]$, and the lift $G_1(A)^{-1}\bar B$. The common form shows that the LRU and S5 are the same layer, parametrised and initialised differently (the normalisation of $\bar B$ in the LRU being one such choice), which affects training but not what the layer can compute.

\paragraph{LinOSS and D-LinOSS.} Both apply one multi-input multi-output system to all input channels, whose state holds the velocity $\dot y_k$ and position $y_k$ of each oscillator $\ddot y_k+\mathsf{G}_k\dot y_k+\mathsf{A}_ky_k=B_kx$, with stiffness $\mathsf{A}_k\ge0$, damping $\mathsf{G}_k\ge0$ and step $\Delta t_k$; LinOSS is the case $\mathsf{G}_k=0$. The models store these parameters in the block diagonal continuous $\mathsf{A}$ whose $2\times2$ blocks $\mathsf{A}^{(k)}$, acting on $(\dot y_k,y_k)$, are
\begin{equation}\label{eq:oscillator-block}
\mathsf{A}^{(k)}=\begin{pmatrix}\mathsf{G}_k & \mathsf{A}_k\\ -1 & 0\end{pmatrix}.
\end{equation}
Discretising the oscillator implicitly for LinOSS-IM and implicitly--explicitly for D-LinOSS yields a block diagonal $\bar{\mathsf{A}}$ with blocks
\begin{equation}\label{eq:linoss-discrete}
\bar{\mathsf{A}}^{(k)}_{\mathrm{IM}}=s_k\begin{pmatrix}1 & -\Delta t_k\mathsf{A}_k\\ \Delta t_k & 1\end{pmatrix}
\quad\text{or}\quad
\bar{\mathsf{A}}^{(k)}_{\mathrm{D}}=s_k\begin{pmatrix}1 & -\Delta t_k\mathsf{A}_k\\ \Delta t_k & 1+\Delta t_k\mathsf{G}_k-\Delta t_k^2\mathsf{A}_k\end{pmatrix},
\end{equation}
where $s_k=(1+\Delta t_k^2\mathsf{A}_k)^{-1}$ for LinOSS-IM and $s_k=(1+\Delta t_k\mathsf{G}_k)^{-1}$ for D-LinOSS; LinOSS-IMEX is the case of D-LinOSS with $\mathsf{G}_k=0$. With the unit clock, reading off gives blocks $\Delta t_k\mathsf{A}^{(k)}$, which reproduce the continuous system, whereas applying (\ref{eq:conversion}) gives blocks $-\log\bar{\mathsf{A}}^{(k)}$, which reproduce the trained model. Their eigenvalues are those of Table~\ref{tab:conversion}, with $\tan\psi_k=\Delta t_k\sqrt{\mathsf{A}_k}$ for LinOSS-IM and $\cos\varphi_k=(2+\Delta t_k\mathsf{G}_k-\Delta t_k^2\mathsf{A}_k)/(2\sqrt{1+\Delta t_k\mathsf{G}_k})$ for D-LinOSS. The determinant of each block's transition is $e^{-\Delta t_k\mathsf{G}_k}$ when read off and $s_k$ when converted. These differ for LinOSS-IM, where the first is $1$, leaving past inputs unattenuated, and the second is $(1+\Delta t_k^2\mathsf{A}_k)^{-1}<1$, attenuating them at every step. Thus, the attenuation of a trained LinOSS-IM comes entirely from its discretisation. LinOSS and D-LinOSS reach Jordan blocks only as edge cases. LinOSS reaches a $2\times2$ block with eigenvalue $0$ only when $\mathsf{A}_k=0$, and D-LinOSS a $2\times2$ block with eigenvalue $\tfrac12\log(1+\Delta t_k\mathsf{G}_k)$ only when $\mathsf{G}_k=2\sqrt{\mathsf{A}_k}+\Delta t_k\mathsf{A}_k$. Each needs an exact relation between the parameters, whereas the EWS reaches blocks of any size and eigenvalue on an open set of them. Similarly, the translated-Laguerre matrix of HiPPO \citep{NEURIPS2020_102f0bb6} is lower triangular with constant diagonal, and hence a single Jordan block, but it is fixed rather than learned.

\paragraph{Mamba-1.} Mamba-1 applies to each input channel a single-input single-output system with a real diagonal $A$ of positive entries, and computes from the input a step $\Delta_t$ for each channel and a shared lift $B_t$, updating the state by $h_t=e^{-\Delta_tA}h_{t-1}+\Delta_tB_tx_t$. The accumulated step $\theta_t=\sum_{s\leq t}\Delta_s$ is a clock, with increments $\Delta\theta_t=\Delta_t$, and matching this update with (\ref{eq:app-rec}) at the identity lift gives the path increments
\begin{equation}\label{eq:app-mamba}
\Delta X_t=\varphi_1(-\Delta_tA)^{-1}\Delta_tB_tx_t ,
\end{equation}
for which depth one of the EWS equals $h_t$ at every step. The input-dependent lift therefore enters through the path rather than through the parameters of the layer. The steps of different channels are not proportional, and a Mamba-1 layer is therefore a family of depth-one EWS, one per channel.

\paragraph{Mamba-2.} A Mamba-2 layer splits its channels into heads, each head having its own scalar $a$ and its own step $\Delta_t$, both shared by the $P$ channels of that head, and all heads sharing the lift $B_t$ the layer computes from the input. The head updates its matrix state by $H_t=e^{-a\Delta_t}H_{t-1}+\Delta_tB_tx_t^{\top}$, whose channels therefore share one clock. Matching this update with (\ref{eq:app-rec}), and identifying $\mathbb{R}^{N\times P}$ with $\mathbb{R}^{NP}$, the head is a single depth-one EWS with generator $aI$, identity lift, and path increments the entries of $\varphi_1(-a\Delta_t)^{-1}\Delta_tB_tx_t^{\top}$. The generator of a head is thus a single scalar times the identity, where Mamba-1 has a diagonal generator with a different rate for each coordinate of a channel. A Mamba-2 layer is therefore a family of depth-one EWS, one per head. 

\subsection{Numerical verification}\label{app:ssm-numerical}

Table~\ref{tab:conv-verify} compares each model's own recurrence, as implemented in its released code (or that of \citet{Walker2024LogNCDE}), with depth one of the converted EWS, computed by the EWS layer used in our experiments. The comparison runs over every step and coordinate of random inputs of length $128$, with parameters drawn from each model's initialisation, state sizes between $16$ and $256$, three seeds and double precision throughout. It reports, for the worst seed, the largest difference between the two states divided by the largest entry of the model's own state. Every row is at round-off except D-LinOSS, whose lift requires inverting an ill-conditioned $G_\Delta(A)$. 

\begin{table}[th]
\centering
\caption{\small Conversion errors on random parameters in double precision, worst of three seeds. The route is the read-off of the stored parameters, the matching (\ref{eq:conversion}), or the per-channel or per-head matching of a selective model.}
\label{tab:conv-verify}

\scriptsize
\setlength{\tabcolsep}{2.4pt}
\renewcommand{\arraystretch}{1.08}

\begin{tabular}{@{}llc@{}}
\toprule
Model & Route & Max. rel. error \\
\midrule

S4D
& read-off
& $9.9{\times}10^{-15}$ \\

S5
& read-off
& $4.5{\times}10^{-15}$ \\

S4
& (\ref{eq:conversion})
& $2.1{\times}10^{-14}$ \\

LRU
& (\ref{eq:conversion})
& $1.6{\times}10^{-14}$ \\

LinOSS-IM
& (\ref{eq:conversion})
& $4.6{\times}10^{-15}$ \\

LinOSS-IMEX
& (\ref{eq:conversion})
& $5.8{\times}10^{-15}$ \\

D-LinOSS
& (\ref{eq:conversion})
& $5.2{\times}10^{-10}$ \\

Mamba-1
& clocked
& $1.5{\times}10^{-14}$ \\

Mamba-2
& clocked
& $7.3{\times}10^{-15}$ \\

\bottomrule
\end{tabular}
\end{table}

On trained parameters, we convert every SSM layer of LinOSS-IM models trained on Heartbeat with five seeds, leaving the rest of each network unchanged. Test accuracy is unchanged on every seed, and the logits agree to within $10^{-13}$.

\section{Experimental Details}\label{app:experiments}

Every run uses a single NVIDIA H100 GPU. Hyperparameters are chosen on validation data alone, and every test score is taken at each seed's best-validation checkpoint, with the selection rule stated for each benchmark below. Apart from the synthetic experiments of Appendix~\ref{app:synthetic}, the grids given below were the final stage of the hyperparameter selection process. Prior to them, the generator family, clock, path and readout of each dataset were chosen by a Bayesian search with the tree-structured Parzen estimator of Optuna \citep{NIPS2011_86e8f7ab,akiba2019optuna} on the same validation splits, with $50$ to $100$ trials per dataset. Parameter counts for the EWS come from our code, and baseline counts are taken from their papers or released configurations. We use the following publicly available datasets, libraries and baseline implementations.
\begin{itemize}\itemsep2pt
\item UEA Multivariate Time Series Classification Archive \citep{bagnall2018ueamultivariatetimeseries}. Licence: GPL-3.0.\\ URL: \url{https://www.timeseriesclassification.com}, \url{https://github.com/time-series-machine-learning/tsml-repo}
\item PPG-DaLiA \citep{s19143079}, UCI Machine Learning Repository, id 495. Licence: CC BY 4.0.\\ URL: \url{https://archive.ics.uci.edu/dataset/495}
\item Weather forecasting data \citep{haoyietal-informer-2021}, released with the Informer code. Licence: Apache 2.0.\\ URL: \url{https://github.com/zhouhaoyi/Informer2020}
\item Formal language benchmark \citep{deletang2023neural}. Licence: Apache 2.0.\\ URL: \url{https://github.com/google-deepmind/neural_networks_chomsky_hierarchy}
\item TUH Abnormal EEG Corpus v3.0.1 \citep{10.3389/fnins.2016.00196,lopez2015abnormal}, used under the corpus data use agreement.\\ URL: \url{https://isip.piconepress.com/projects/nedc/html/tuh_eeg/}
\item JAX \citep{jax2018github}, Equinox \citep{kidger2021equinox} and Optax. Licence: Apache 2.0.\\ URL: \url{https://github.com/google/jax}, \url{https://github.com/patrick-kidger/equinox}, \url{https://github.com/google-deepmind/optax}
\item Signax. Licence: MIT.\\ URL: \url{https://github.com/Anh-Tong/signax}
\item Log-NCDE \citep{Walker2024LogNCDE}. Licence: CC BY 4.0.\\ URL: \url{https://github.com/Benjamin-Walker/log-neural-cdes}
\item LinOSS \citep{rusch2025oscillatory}. Licence: MIT.\\ URL: \url{https://github.com/tk-rusch/linoss}
\item D-LinOSS \citep{boyer2025learningdissipateenergyoscillatory}. Licence: MIT.\\ URL: \url{https://github.com/jaredbmit/damped-linoss}
\item BIOT \citep{NEURIPS2023_f6b30f3e}. Licence: MIT.\\ URL: \url{https://github.com/ycq091044/BIOT}
\item S4 and S4D \citep{gu2021efficiently,NEURIPS2022_e9a32fad}, Mamba and Mamba-2 \citep{gu2024mamba,dao2024transformers}. Licence: Apache 2.0.\\ URL: \url{https://github.com/state-spaces/s4}, \url{https://github.com/state-spaces/mamba}
\item PyTorch \citep{NEURIPS2019_bdbca288}. Licence: BSD-3-Clause.\\ URL: \url{https://github.com/pytorch/pytorch}
\item SLiCE \citep{walker2025structuredlinearcdesmaximally}. Licence: MIT.\\ URL: \url{https://github.com/Benjamin-Walker/structured-linear-cdes}
\end{itemize}

\subsection{The EWS block}\label{app:ews-block}

A layer acts on a path, and the input observations $x$ are read as its values at their observation times. The path may also be preprocessed before the layer sees it, by normalising its channels to a fixed range, by prepending a zero base-point, which lets the first observation contribute an increment of its own, by augmenting it with a channel carrying the clock, or by reading the observations as increments instead, in which case the path is their cumulative sum and is the integral $X_t=\int_0^tx_s\,ds$ of a sampled signal $x$. These choices are stated per experiment below and leave the layer itself unchanged.

The EWS layer of~(\ref{eq:ews-layer}) can be placed inside a residual block and stacked, as deep state-space architectures stack their layers. Write $H$ for the width of the residual stream, which an encoder produces from the $d$ channels of the data, and $u\in\mathbb{R}^{T\times H}$ for the block's input. Its layer has width $m$, depth $N$, and thus a state of $K=\sum_{n\leq N}m^n$ coordinates. The block computes
\begin{align}\label{eq:ews-block}
\theta&=\mathrm{clock}(u),\notag\\
S_t&=S_{\mathbb{A}}(P;\theta)^{\leq N}_{0,t},\\
z&=u+C\,\Gamma(S),\notag\\
y&=z+\mathrm{FF}\big(\mathrm{RMSNorm}(z)\big),\notag
\end{align}
where $P$ is the path built from $\mathrm{RMSNorm}(u)$ as above, and the lift $B$ of $\mathbb{A}$ carries the stream into the $m$ channels of the layer. Figure~\ref{fig:ews-block} shows the order of these steps.

\begin{figure}[th]
\centering
\begin{tikzpicture}[>=Stealth, node distance=3.2mm,
  box/.style={draw, rounded corners=2pt, inner xsep=3pt, inner ysep=2.5pt, font=\footnotesize},
  op/.style={draw, circle, inner sep=0.8pt, font=\footnotesize}]
\node (u) {$u$};
\node[box, right=of u] (n1) {RMSNorm};
\node[box, right=of n1] (ews) {EWS};
\node[box, right=of ews] (G) {$\Gamma$};
\node[box, right=of G] (C) {$C$};
\node[op, right=of C] (a1) {$+$};
\node[box, right=of a1] (n2) {RMSNorm};
\node[box, right=of n2] (ff) {FF};
\node[op, right=of ff] (a2) {$+$};
\node[right=of a2] (y) {$y$};
\foreach \a/\b in {u/n1,n1/ews,ews/G,G/C,C/a1,a1/n2,n2/ff,ff/a2,a2/y}{\draw[->] (\a)--(\b);}
\node[box, below=5mm of ews] (cl) {clock};
\draw[->] (u.south) |- (cl.west);
\draw[->] (cl) -- (ews);
\draw[->] (u.north) to[out=50,in=130] (a1.north);
\draw[->] (a1.north) to[out=50,in=130] (a2.north);
\end{tikzpicture}
\caption{\small The EWS block. The clock reads the block's input, the state of the layer is normalised at each depth separately by $\Gamma$ and read down by $C$, and both branches are residual.}
\label{fig:ews-block}
\end{figure}

The state $S_t$, which is returned at every position by a parallel associative scan, is graded, and thus its coordinates at each depth differ in their orders of magnitude (depth $n$ is bounded by $C^n/n!$ by Lemma~\ref{lem:decay}). Therefore, reading them through a single linear map would make the lower depths dominate the higher ones. The graded norm $\Gamma$ normalises each depth separately,
\begin{equation}\label{eq:graded-norm}
\Gamma(S)^{(n)}=g_n\,\frac{S^{(n)}}{\max\{\rho_n,\varepsilon_n\}},\qquad \rho_n=\Big(\tfrac{1}{m^n}\textstyle\sum_i (S^{(n)}_i)^2\Big)^{1/2},
\end{equation}
with a learned gain $g_n$ for each depth set from the first training batch. The read-down $C$ linearly maps the $K$ coordinates of $\Gamma(S)$ to the stream width, and the feed-forward branch, applied at each position separately, is $\mathrm{FF}(v)=W_1\,\mathrm{gelu}(v)\odot\sigma\big(W_2\,\mathrm{gelu}(v)\big)$ with $W_1$ and $W_2$ of width $H$. Dropout is applied to each residual branch.

The formal language tasks use two blocks, PPG-DaLiA four and weather one, whereas the synthetic tasks, the UEA archive and TUAB apply the layer of~(\ref{eq:ews-layer}) directly, with the readout given in each subsection.

\paragraph{Computing the local EWS.}\label{app:local-ews}
The EWS layer is computed by a parallel associative scan that combines the local EWS states over each interval through the modified Chen identity. Unlike an SSM, the EWS never discretises its dynamics at depth one, where the local EWS is exactly $\varphi_1(-\Delta\theta A)\,B\Delta X$ (Appendix~\ref{app:observation-times}), a matrix function that we evaluate to floating-point precision. Every layer of depth one in this paper is therefore exact. Deeper levels also have an exact expression, as the exponential of the structured matrix of \eqref{eq:flattened}, but it is expensive to evaluate for a general generator. However, the rotation and Jordan families have closed forms at every depth, which we implement at depth two. For other generators we instead approximate the deeper terms of the local EWS numerically yielding the small errors of Table~\ref{tab:local-error}. 

\begin{table}[th]
\centering
\scriptsize
\setlength{\tabcolsep}{2.4pt}
\renewcommand{\arraystretch}{1.08}
\caption{\small
Median relative error of the numerically approximated local EWS against
an accurate reference, over ten validation sequences at the initial
parameters of each configuration.
}
\label{tab:local-error}

\begin{tabular}{@{}lcccc@{}}
\toprule
Dataset & Depth & Level 1 & Level 2 & Level 3 \\
\midrule

SelfRegulationSCP1
& $2$
& $1.0{\times}10^{-4}$
& $1.5{\times}10^{-4}$
& -- \\

Heartbeat
& $3$
& $2.9{\times}10^{-5}$
& $4.7{\times}10^{-5}$
& $6.4{\times}10^{-5}$ \\

MotorImagery
& $3$
& $2.1{\times}10^{-4}$
& $2.9{\times}10^{-4}$
& $3.6{\times}10^{-4}$ \\

SelfRegulationSCP2
& $3$
& $4.9{\times}10^{-5}$
& $7.0{\times}10^{-5}$
& $8.6{\times}10^{-5}$ \\

\bottomrule
\end{tabular}
\end{table}
\subsection{Additional experiment}

The TUH Abnormal EEG Corpus (TUAB) \citep{lopez2015abnormal} is a subset of the TUH EEG Corpus \citep{10.3389/fnins.2016.00196} in which each recording is labelled normal or abnormal by a neurologist, with training and evaluation sets drawn from disjoint groups of patients. We follow the preprocessing and protocol of \citet{NEURIPS2023_f6b30f3e}: we form $16$ channels as differences between neighbouring electrodes,  resample them to $200$\,Hz without filtering, split each recording into non-overlapping windows of $10$ seconds ($2\,000$ steps), and give each window the label of its recording. This gives $316\,854/55\,656/36\,945$ train/val/test windows. Each window is classified on its own, and we report the area under the ROC curve and the balanced accuracy on the test windows. Balanced accuracy, the mean of the accuracies on abnormal and on normal windows, needs a threshold on the model's score, which we set on the validation windows to make the fraction predicted abnormal equal to the fraction that are abnormal as in \citet{NEURIPS2023_f6b30f3e}.

Inspired by the multi-head structure of Mamba-$2$, which Appendix~\ref{app:ssm} shows to be a family of depth-one EWS layers, one per head, the model runs eight EWS layers (\ref{eq:ews-layer}) side by side at depth two, without blocks or residual connections. Each has its own lift from the $16$ channels to $m=8$, a generator of two real Jordan blocks of size four, each a complex pair of multiplicity two, and a blended variation clock whose weight is learned. The path is the integral of the scaled signal, with the base-point augmentation. The outputs of the eight layers are concatenated, $584$ features in all, and passed through dropout to a two-layer perceptron. The configuration was selected on three-seed mean validation AUROC from a grid of dropout in $\{0.1,0.2,0.3\}$, learning rate in $\{2,3,4\}\times10^{-3}$ and MLP's width in $\{32,64\}$, with the other settings fixed by earlier searches. The selected model has dropout $0.1$, learning rate $4\times10^{-3}$, hidden width $32$, and has $20\,250$ parameters.

Training uses the cross-entropy loss and AdamW with weight decay $10^{-3}$, gradients clipped at a global norm of one, a constant learning rate and a batch of $64$ windows, for $5\,000$ steps. We evaluate every $1\,000$ steps, keep the checkpoint with the best validation AUROC, and report the mean over five seeds. 

Table~\ref{tab:tuab} compares the EWS both with models trained from scratch and with models pre-trained on several EEG datasets, with results and parameter counts taken from \citet{NEURIPS2023_f6b30f3e}, \citet{jiang2024large} and \citet{wang2025cbramod}. It also includes two variants of our model. One is the same model at depth one, and the other freezes the trained model and fits $600$ gradient-boosted trees to its $584$ features in place of the MLP.

\begin{table}[th]
\centering
\scriptsize
\setlength{\tabcolsep}{2.4pt}
\renewcommand{\arraystretch}{1.08}
\caption{\small
TUAB balanced accuracy and AUROC; mean $\pm$ standard deviation, over
five seeds for the EWS. CBraMod's balanced accuracy uses a fixed
threshold of $0.5$.
}
\label{tab:tuab}

\begin{tabular}{@{}lrcc@{}}
\toprule
Model & Params. & Balanced acc.\ (\%) & AUROC \\
\midrule

EEGNet \citep{Lawhern_2018}
& $0.003$M
& $76.42{\pm}0.36$
& $0.8412{\pm}0.0031$ \\

SPaRCNet \citep{jing2023sparcnet}
& $0.79$M
& $78.96{\pm}0.18$
& $0.8676{\pm}0.0012$ \\

ST-Transformer \citep{song2021transformerbasedspatialtemporalfeaturelearning}
& $3.5$M
& $79.66{\pm}0.23$
& $0.8707{\pm}0.0019$ \\

BIOT \citep{NEURIPS2023_f6b30f3e}
& $3.2$M
& $79.25{\pm}0.35$
& $0.8691{\pm}0.0033$ \\

EWS, depth one
& $0.003$M
& $62.47{\pm}3.48$
& $0.6665{\pm}0.0469$ \\

EWS, depth two
& $0.020$M
& $79.90{\pm}0.45$
& $0.8824{\pm}0.0028$ \\

EWS, depth two + boosting
& $0.020$M + $600$ trees
& $80.38{\pm}0.35$
& $0.8903{\pm}0.0021$ \\

\midrule

BIOT, pretrained
& $3.2$M
& $79.59{\pm}0.57$
& $0.8815{\pm}0.0043$ \\

CBraMod, without TUAB
& $4.0$M
& $82.49{\pm}0.25$
& $0.9156{\pm}0.0017$ \\

CBraMod \citep{wang2025cbramod}
& $4.0$M
& $82.89{\pm}0.22$
& $0.9227{\pm}0.0011$ \\

\bottomrule
\end{tabular}
\end{table}

The EWS exceeds every model trained from scratch while only using $100 \times$ fewer parameters than any transformer, and matches the pretrained BIOT model which makes use of six EEG datasets prior to training on TUAB. CBraMod remains ahead, as expected of a foundation model with $200\times$ as many parameters, pretrained on large EEG corpora. Depth is essential for this task, as the same model at depth one falls below every baseline, and even with gradient boosting in place of its MLP ($64.05\pm3.55\%$ balanced accuracy and  $0.6912 \pm 0.0495$ AUROC) it remains below the depth-two model with an MLP; hence, a stronger readout of depth-one features does not recover what depth two provides. Depth three, with $m^3$ coordinates per head, does not fit in memory at eight heads. Using gradient boosting instead raises the test scores on each seed. 

\subsection{Synthetic experiments}\label{app:synthetic}

The task and the three models used are those from Section~\ref{sec:main-synthetic}, all of which are trained with AdamW, a linear warmup followed by cosine decay, and the test score is taken at the best-validation checkpoint.

\paragraph{Expressivity.} A third task, not reported in the main text, verifies whether each model can fit targets from its own class and from the other two. Its paths are $750$ time-augmented Brownian paths $(t,W^1,W^2)$ on $[0,5]$ at $10\,000$ steps, split $70/15/15$. Each target is a depth-two cross term $S_{\mathbb{A}}(X;\theta)^{i,j}$, $i \neq j$, of the EWS of that path under a fixed generator, which is either zero, positive diagonal or has complex eigenvalues, placing the target in the class of the signature, EFM and EWS respectively.  By the shuffle identity, $S^{i,j} = \tfrac{1}{2} S^iS^j +\tfrac{1}{2}(S^{i,j} - S^{j,i})$, and the second term, the L{\'e}vy area of the re-weighted path, is not a function of the depth-one coordinates of the same EWS. Every model is trained at depth two, without a lift, for $30\,000$ steps at a batch size of $32$ on ten seeds. Table~\ref{tab:syn-inclass} reports the results. As expected, each model learns its own target and the signature target is learnt by all three. The EFM target is reached by the EWS but the EFM is unable to learn the EWS target due to its complex spectrum. 

\begin{table}[th]
\centering
\scriptsize
\setlength{\tabcolsep}{2.4pt}
\renewcommand{\arraystretch}{1.08}
\caption{\small
Expressivity. Test RMSE on targets scaled to $[0,1]$, mean $\pm$
standard deviation over ten seeds at depth two.
}
\label{tab:syn-inclass}

\begin{tabular}{@{}lccc@{}}
\toprule
Target $\backslash$ Model & Signature & EFM & EWS \\
\midrule

Signature
& $(2.00{\pm}0.08)\times10^{-5}$
& $(1.99{\pm}0.08)\times10^{-5}$
& $(2.02{\pm}0.17)\times10^{-5}$ \\

EFM
& $(4.28{\pm}0.53)\times10^{-2}$
& $(2.66{\pm}0.28)\times10^{-5}$
& $(2.77{\pm}0.28)\times10^{-5}$ \\

EWS
& $(2.41{\pm}0.26)\times10^{-2}$
& $(2.43{\pm}0.26)\times10^{-2}$
& $(9.18{\pm}8.80)\times10^{-4}$ \\

\bottomrule
\end{tabular}
\end{table}

\paragraph{Oscillator.} The path is $(t,W^1,W^2)$ and the target is the first component of
\begin{align}\label{eq:syn-sde}
dY^1_t&=(3\sin Y^2_t-\tfrac12Y^1_t)\,dt+\tfrac25\,dW^1_t,\notag\\
dY^2_t&=(3\cos Y^1_t-\tfrac12Y^2_t)\,dt+\tfrac25\,dW^2_t,
\end{align}
We generate $750$ trajectories on $[0,4]$ with $10\,000$ Euler--Maruyama steps, split $70/15/15$, scale the inputs and the target to $[0,1]$, prepend a basepoint, and train for $15\,000$ steps at batch $32$ on ten seeds. The learned generator of the EWS carries a complex pair on every seed, with real part $1.86\pm0.15$ and imaginary part $10.06\pm0.84$ at depth two, and $1.98\pm0.01$ and $10.68\pm0.04$ at depth one. 

\paragraph{Fourier targets.} The signal is a sum of tones on a grid of $256$ samples within the range $[0,1]$. One or two low tones at frequencies drawn uniformly from $1$ to $5$\,Hz are present throughout, and one or two high tones from $20$ to $30$\,Hz switch on for a burst of $38$ to $102$ samples, with no high tone at all in $15\%$ of samples; amplitudes are drawn from $[0.5,1.5]$, with Gaussian noise of standard deviation $0.05$ added. We generate $10 \,000$ samples and split them into training/validation/test with a ratio of $80/10/10$. We use the path with increments $(\Delta t,\,x_t\Delta t)$, whose second channel is the integral of the signal. The targets have four frequencies $\{2,4,22,27\}$\,Hz and a rate of $a=5$, and are computed for every sample whatever tones it carries. Models are trained for $12\,000$ steps at batch $64$ on three seeds, with a logistic readout for $F_3$, and at each depth we report the width with the best validation score, over a grid capped at $5\,000$ features (a model of depth $n$ and width $m$ has $\sum_{1 \leq k\leq n}m^k$ features).

Table~\ref{tab:syn-width} gives the EWS on $F_2$ and $F_3$ at every width. Depth one is flat in width on both targets, and depth two improves steeply with it. The EFM and the signature reach at best $0.908$ and $0.906$ on $F_2$, and $69.5\%$ on $F_3$, the majority class, at every depth and width of their grids.

Splitting $F_1$ by band shows where each model fails. At the selected width at depth one, the EWS reaches $1.2\times10^{-5}$ and $5.8\times10^{-5}$ on the low and high bands, the EFM $0.105$ and $0.918$, and the signature $0.880$ and $0.999$. Even the EFM's lowest test error anywhere in its grid, at depth three, is $0.798$ on the high band. At the smallest width the EWS is unreliable, one seed of three reaching $0.32$ against $0.005$ and $0.029$ for the other two, since its rotation blocks can miss the high band.
\begin{table}[th]
\centering
\scriptsize
\setlength{\tabcolsep}{2.4pt}
\renewcommand{\arraystretch}{1.08}
\caption{\small
The EWS on $F_2$ and $F_3$ by depth and width, mean $\pm$ standard deviation over three seeds.
}
\label{tab:syn-width}

\begin{tabular}{@{}lccccc@{}}
\toprule
& & \multicolumn{2}{c}{$F_2$, normalised MSE}
& \multicolumn{2}{c}{$F_3$, accuracy (\%)} \\
\cmidrule(lr){3-4}
\cmidrule(lr){5-6}
$m$ & $D$ at $N{=}2$ & $N{=}1$ & $N{=}2$ & $N{=}1$ & $N{=}2$ \\
\midrule

$8$
& $72$
& $0.939{\pm}0.042$
& $0.285{\pm}0.173$
& $69.4{\pm}0.1$
& $77.2{\pm}10.4$ \\

$16$
& $272$
& $0.921{\pm}0.023$
& $0.046{\pm}0.040$
& $69.4{\pm}0.1$
& $92.9{\pm}0.5$ \\

$32$
& $1\,056$
& $0.921{\pm}0.022$
& $(7.3{\pm}3.4)\times10^{-4}$
& $69.5{\pm}0.0$
& $94.7{\pm}0.2$ \\

\bottomrule
\end{tabular}
\end{table}

\subsection{UEA multivariate time series classification archive}\label{app:uea}

We use the six datasets of the long-range benchmark of \citet{Walker2024LogNCDE}, whose dimension, length and number of classes are given in Table~\ref{tab:uea-data}, with their 70/15/15 splits and their five seeds, and with each channel scaled to a fixed range as in their pipeline. The model is the layer of~(\ref{eq:ews-layer}) applied directly, with a linear readout and no block. All six use Adam at a constant learning rate, cross-entropy loss and a best-validation checkpoint. The final grid selected the configuration with the highest five-seed mean best-validation accuracy, and the test accuracy of Table~\ref{tab:uea} is that configuration's mean over the five seeds. The selected configurations are given in Table~\ref{tab:uea-config}, where the path column records how the observations are read and any augmentation, with base.\ for a basepoint and clock for a clock channel, the generator and clock columns use the families of Table~\ref{tab:families}, and EigenWorms' clamped clock has four units. The feature normalisation is a learned gain per level, except on EthanolConcentration, which uses a BatchNorm evaluated with training-set statistics.

On EthanolConcentration the EWS reads the derivative of the input as its path, rather than the input itself or its integral. To test whether this preprocessing explains its margin, we trained S5, LRU and LinOSS-IM models using the same derivative path, scaled and normalised as their original inputs. Over five seeds, they reach $32.2\pm 12.6\%$, $26.6 \pm 6.8 \%$ and $26.3 \pm 2.5 \%$ respectively. Therefore, the derivative path alone does not explain the margin of the EWS. Conversely, changing the input path of the EWS model to the raw observations yields $63.2 \pm 11.4 \%$, which is already much higher than the other baselines. The selected EWS at depth one is an LTI SSM that a single $S5$ layer can represent, and it already reaches $66.3 \pm 8.9\%$, so the gap reflects parametrisation and training strengths of the EWS, rather than expressivity.

\begin{table}[th]
\centering
\small
\caption{\small The six datasets.}
\label{tab:uea-data}
\begin{tabular}{@{}lccc@{}}
\toprule
Dataset & Dimension & Observations & Classes \\
\midrule
EigenWorms & $6$ & $17\,984$ & $5$ \\
EthanolConcentration & $3$ & $1\,751$ & $4$ \\
Heartbeat & $61$ & $405$ & $2$ \\
MotorImagery & $64$ & $3\,000$ & $2$ \\
SelfRegulationSCP1 & $6$ & $896$ & $2$ \\
SelfRegulationSCP2 & $7$ & $1\,152$ & $2$ \\
\bottomrule
\end{tabular}
\end{table}

\begin{table}[th]
\centering
\scriptsize
\setlength{\tabcolsep}{2.4pt}
\renewcommand{\arraystretch}{1.08}
\caption{\small Selected configurations on the UEA archive. Skew denotes a dense skew-symmetric generator.}
\label{tab:uea-config}
\begin{tabular}{@{}lllllcccccc@{}}
\toprule
Dataset & Path & Gen. & Clock & Readout & $N$ & $m$ & Batch & Steps & LR & Params \\
\midrule
EigenWorms
& value, clock
& Jordan
& clamped
& mean
& $2$ & $8$ & $4$ & $6\,000$ & $4.5{\times}10^{-3}$ & $446$ \\

Ethanol
& derivative
& rotation
& time
& terminal
& $2$ & $24$ & $8$ & $12\,000$ & $2.5{\times}10^{-3}$ & $2\,508$ \\

Heartbeat
& value, base.
& rotation
& blended
& mean
& $3$ & $10$ & $8$ & $25\,000$ & $3.0{\times}10^{-4}$ & $2\,847$ \\

Motor
& value
& skew
& time
& terminal
& $3$ & $16$ & $16$ & $12\,000$ & $4.3{\times}10^{-3}$ & $9\,887$ \\

SCP1
& value
& skew
& variation
& mean
& $2$ & $17$ & $32$ & $25\,000$ & $2.5{\times}10^{-3}$ & $856$ \\

SCP2
& value
& skew
& blended
& terminal
& $3$ & $14$ & $8$ & $25\,000$ & $5.0{\times}10^{-4}$ & $6\,104$ \\

\bottomrule
\end{tabular}
\end{table}

The grids were, for EigenWorms, the learning rate, the width and the damping at initialisation over $27$ cells; for EthanolConcentration, the learning rate, the width and the frequency scale over $27$ cells; for Heartbeat, the initialisation variance of the generator, the learning rate and the width over $18$ cells; for MotorImagery, the initialisation variance, the scale of the readout and the width over $18$ cells; for SelfRegulationSCP1, the width, the initialisation variance and the scale of the readout over $18$ cells; and for SelfRegulationSCP2, the scale of the readout, a penalty on the readout weights and the initialisation variance over $18$ cells. 

Table~\ref{tab:uea-depth} gives each selected configuration trained at depths one, two and three, with everything else unchanged. Validation selects depth two or three on every dataset, and the selected improves the test accuracy of that configuration over depth one on five of the six. Table~\ref{tab:uea-paired} widens this to $33$ configurations per dataset (a reference and $32$ drawn from a wider grid around it; $32$ for EigenWorms, where one did not complete), each trained at the three depths on three seeds, at the same width for $6\,000$ steps. It counts how many configurations improve, tie and fall in accuracy, alongside the median change.

\begin{table}[th]
\centering
\scriptsize
\setlength{\tabcolsep}{2.4pt}
\renewcommand{\arraystretch}{1.08}
\caption{\small
Selected configuration at each depth, mean $\pm$ standard deviation of
validation and test accuracy (\%) over five seeds.
}
\label{tab:uea-depth}

\begin{tabular}{@{}lcccccc@{}}
\toprule
& \multicolumn{2}{c}{$N=1$}
& \multicolumn{2}{c}{$N=2$}
& \multicolumn{2}{c}{$N=3$} \\
\cmidrule(lr){2-3}
\cmidrule(lr){4-5}
\cmidrule(lr){6-7}
Dataset & Val. & Test & Val. & Test & Val. & Test \\
\midrule

EigenWorms
& $90.9{\pm}4.9$
& $85.6{\pm}4.8$
& $\mathbf{96.0{\pm}2.9}$
& $\mathbf{93.3{\pm}1.4}$
& $93.7{\pm}4.6$
& $92.2{\pm}4.1$ \\

Ethanol
& $69.1{\pm}13.7$
& $66.3{\pm}8.9$
& $\mathbf{76.0{\pm}7.7}$
& $\mathbf{71.9{\pm}4.8}$
& $73.4{\pm}6.5$
& $71.4{\pm}3.7$ \\

Heartbeat
& $77.1{\pm}3.7$
& $76.1{\pm}5.9$
& $74.4{\pm}5.6$
& $77.4{\pm}2.7$
& $\mathbf{77.7{\pm}5.6}$
& $\mathbf{80.0{\pm}4.5}$ \\

Motor
& $57.2{\pm}3.6$
& $50.5{\pm}6.8$
& $58.3{\pm}6.0$
& $52.6{\pm}7.4$
& $\mathbf{66.0{\pm}2.1}$
& $\mathbf{61.1{\pm}5.0}$ \\

SCP1
& $90.0{\pm}2.5$
& $86.1{\pm}3.8$
& $\mathbf{90.2{\pm}3.4}$
& $\mathbf{87.1{\pm}2.5}$
& $84.8{\pm}3.8$
& $81.7{\pm}2.1$ \\

SCP2
& $58.3{\pm}3.9$
& $58.6{\pm}5.3$
& $58.2{\pm}3.9$
& $47.4{\pm}6.2$
& $\mathbf{61.4{\pm}2.7}$
& $\mathbf{56.5{\pm}5.8}$ \\

\bottomrule
\end{tabular}
\end{table}

\begin{table}[th]
\centering
\scriptsize
\setlength{\tabcolsep}{2.4pt}
\renewcommand{\arraystretch}{1.08}
\caption{\small
Paired depth comparison over the $33$ configurations of each dataset. Each entry gives the number that improve/tie/fall, followed by the median change in accuracy (\%).
}
\label{tab:uea-paired}

\begin{tabular}{@{}lcccc@{}}
\toprule
& \multicolumn{2}{c}{$N=1\to2$}
& \multicolumn{2}{c}{$N=1\to3$} \\
\cmidrule(lr){2-3}
\cmidrule(lr){4-5}
Dataset & Val. & Test & Val. & Test \\
\midrule

EigenWorms
& $29/0/3,\ +10.7$
& $23/0/9,\ +9.1$
& $28/1/3,\ +11.1$
& $20/0/12,\ +8.2$ \\

Ethanol
& $22/1/10,\ +3.8$
& $27/0/6,\ +4.2$
& $22/0/11,\ +3.4$
& $26/0/7,\ +4.6$ \\

Heartbeat
& $25/3/5,\ +1.1$
& $23/2/8,\ +1.1$
& $28/2/3,\ +1.6$
& $23/4/6,\ +1.1$ \\

Motor
& $25/0/8,\ +1.8$
& $20/5/8,\ +1.2$
& $28/0/5,\ +4.7$
& $19/2/12,\ +1.2$ \\

SCP1
& $17/2/14,\ +0.0$
& $3/2/28,\ -1.6$
& $6/1/26,\ -2.4$
& $5/1/27,\ -2.4$ \\

SCP2
& $24/1/8,\ +2.3$
& $10/0/23,\ -3.5$
& $29/0/4,\ +4.3$
& $14/0/19,\ -0.9$ \\

\bottomrule
\end{tabular}
\end{table}

The EWS uses $3.8$k parameters per dataset on average, the fewest of any model in Table~\ref{tab:uea-params} with an available count. 

\begin{table}[th]
\centering
\scriptsize
\setlength{\tabcolsep}{2.4pt}
\renewcommand{\arraystretch}{1.08}
\caption{\small Parameter counts at the configuration each paper reports, as published by \citet{rusch2025oscillatory} and \citet{boyer2025learningdissipateenergyoscillatory} where available and otherwise computed from the released code. A dash marks a model with no released configuration for that dataset.}
\label{tab:uea-params}
\begin{tabular}{@{}lrrrrrrr@{}}
\toprule
Model & EigenWorms & Ethanol & Heartbeat & Motor & SCP1 & SCP2 & Mean \\
\midrule
Log-NCDE     & $37\,977$  & $31\,452$     & $168\,320$    & $81\,391$  & $91\,557$  & $36\,379$  & $74\,513$ \\
S5           & $22\,007$  & $76\,214$     & $158\,310$    & $17\,496$  & $226\,328$ & $5\,652$   & $84\,335$ \\
S6           & $15\,045$  & $5\,780$      & $6\,674$      & $52\,802$  & $24\,898$  & $26\,018$  & $21\,870$ \\
Mamba        & $27\,381$  & $1\,032\,772$ & $1\,034\,242$ & $228\,226$ & $184\,194$ & $356\,290$ & $477\,184$ \\
DE-SLiCE     & $116\,357$ & $16\,964$     & $16\,914$     & $17\,458$  & $25\,154$  & $115\,970$ & $51\,470$ \\
D-DE-SLiCE   & $4\,245$   & $1\,796$      & $16\,914$     & $17\,458$  & $2\,402$   & $3\,858$   & $7\,779$ \\
BD-SLiCE     & $5\,253$   & $1\,604$      & $5\,010$      & $5\,170$   & $2\,114$   & $4\,866$   & $4\,003$ \\
LinOSS-IM    & $134\,279$ & $6\,728$      & $10\,936$     & $91\,844$  & $991\,240$ & $399\,112$ & $272\,357$ \\
D-LinOSS     & $134\,279$ & $71\,112$     & $4\,356$      & $20\,598$  & $992\,776$ & $399\,496$ & $270\,436$ \\
PD-SSM       & $94\,853$  & $63\,236$     & $366\,722$    & $21\,954$  & $167\,810$ & $94\,594$  & $134\,862$ \\
FLASH PD-SSM & --         & --            & --            & --         & --         & --         & -- \\
\midrule
EWS          & $446$      & $2\,508$      & $2\,847$      & $9\,887$   & $856$      & $6\,104$   & $\mathbf{3\,775}$ \\
\bottomrule
\end{tabular}
\end{table}

Table~\ref{tab:uea-time} reports training time and memory. Every model was timed on the same card with no other process on it, at a batch of $32$, reduced to $16$ on EthanolConcentration and $8$ on EigenWorms for memory, taking the median of the steady intervals of a run of $4\,000$ steps with the first interval dropped. Memory is the peak of the process as reported by the driver. The EWS is thus parameter-efficient but not uniformly compute-efficient as its peak memory exceeds that of all other models on three of the six UEA datasets.

\begin{table}[th]
\centering
\scriptsize
\setlength{\tabcolsep}{2.4pt}
\renewcommand{\arraystretch}{1.08}
\caption{\small Training time in seconds per $1\,000$ steps and peak memory in MB.}
\label{tab:uea-time}

\begin{tabular}{@{}lcccccc@{}}
\toprule
Model
& EigenWorms
& Ethanol
& Heartbeat
& Motor
& SCP1
& SCP2 \\
\midrule

Log-NCDE
& $1955 / 2\,744$
& $2116 / 880$
& $785 / 2\,938$
& $779 / 4\,744$
& $616 / 904$
& $586 / 882$ \\

S5
& $11.4 / 3\,036$
& $4.6 / 1\,430$
& $7.5 / 1\,694$
& $23.9 / 5\,040$
& $8.7 / 2\,462$
& $6.2 / 1\,180$ \\

S6
& $25.2 / 2\,774$
& $5.3 / 1\,012$
& $7.3 / 918$
& $38.3 / 4\,372$
& $3.2 / 1\,178$
& $8.7 / 1\,540$ \\

Mamba
& $34.7 / 3\,786$
& $80.5 / 3\,022$
& $37.2 / 1\,958$
& $28.7 / 3\,498$
& $5.2 / 1\,556$
& $31.7 / 2\,800$ \\

LRU
& $22.7 / 4\,768$
& $10.9 / 1\,406$
& $5.4 / 1\,654$
& $38.5 / 8\,846$
& $5.1 / 1\,140$
& $4.7 / 1\,146$ \\

BD-SLiCE
& $10.2 / 2\,742$
& $10.1 / 1\,136$
& $33.5 / 2\,932$
& $28.2 / 4\,738$
& $4.9 / 892$
& $4.8 / 1\,136$ \\

LinOSS-IM
& $18.5 / 4\,798$
& $10.7 / 1\,148$
& $9.1 / 1\,672$
& $21.7 / 4\,732$
& $26.8 / 5\,268$
& $18.0 / 3\,206$ \\

D-LinOSS
& $15.3 / 3\,708$
& $20.3 / 3\,590$
& $6.6 / 888$
& $17.4 / 4\,740$
& $28.1 / 5\,070$
& $18.3 / 3\,210$ \\

\midrule

EWS
& $19.7 / 4\,788$
& $27.8 / 5\,964$
& $14.6 / 2\,764$
& $83.6 / 34\,506$
& $10.2 / 2\,752$
& $29.5 / 9\,540$ \\

\bottomrule
\end{tabular}
\end{table}

\subsection{PPG-DaLiA and weather}\label{app:ppg-weather}

Both tasks follow the training protocol of \citet{boyer2025learningdissipateenergyoscillatory}, which on PPG-DaLiA is that of \citet{Walker2024LogNCDE}, a budget of $100\,000$ steps with evaluation every $1\,000$ steps and early stopping after ten evaluations without an improvement; the best-validation checkpoint is reported on five seeds. On both, our configuration was the best on five-seed mean validation error among sixteen drawn at random from a grid over the learning rate, $H$, $m$ and the number of blocks. The path is taken to be the integral of the input stream with the basepoint augmentation.

\paragraph{PPG-DaLiA.}  Each subject's recording is split in time into training, validation and test parts in the ratio $70/15/15$, and the test set therefore measures generalisation to later times within the same subjects, not to new subjects. The grid was the learning rate in $\{2.5,5,10\} \times 10^{-4}$, $H$ in $\{8,16,32\}$, $m$ in $\{32, 64\}$ and the number of blocks in $\{2,4,6\}$. Our configuration has four EWS blocks at depth one with $H=16$ and $m=64$, a time channel augmentation, a Jordan generator of $16$ complex Jordan blocks and the clamped clock with four units; it is trained with the squared error and Adam at a constant rate of $2.5\times10^{-4}$ with gradients clipped at a global norm of one and a batch of $8$.

\paragraph{Weather.} We use the matching loaders of \citet{gu2021efficiently} and \citet{boyer2025learningdissipateenergyoscillatory}, which split the series in time into $70/10/20$ and window it at a stride of one. The model's grid was the learning rate in $\{2.5,5,10\}\times10^{-5}$, $H$ in $\{32,64,128\}$, $m$ in $\{16,32,64\}$ and the number of blocks in $\{1,2\}$. Our configuration has one EWS block at depth one with $H=64$ and $m=64$, a block-diagonal generator of sixteen real Jordan blocks of size four, each a complex pair of multiplicity two, and the time clock, and is trained with the $L^1$ loss and Adam at a constant rate of $5\times10^{-5}$ with a batch of $16$. Table~\ref{tab:wth-time} compares its training time and memory with D-LinOSS, the one baseline we timed.

D-LinOSS publishes a score of $0.486$ from a single seed. The mean of the five seeds committed with its code is $0.538\pm0.052$, which is the value we report in Table~\ref{tab:regression}, as confirmed by the authors in private communication.

\begin{table}[th]\centering\small
\caption{\small Weather, parameters, training time in seconds per $1\,000$ steps and peak memory in MB, at a batch of $16$.}\label{tab:wth-time}
\begin{tabular}{@{}lrcc@{}}
\toprule
Model & Params & Time & Memory \\
\midrule
EWS & $18\,669$ & $8.95$ & $1\,722$ \\
D-LinOSS & $301\,196$ & $7.97$ & $2\,880$ \\
\bottomrule
\end{tabular}
\end{table}

\subsection{Regular language tasks}\label{appendix:language}

We use the four regular language tasks of the formal language benchmark of \citet{deletang2023neural}, following the protocol of \citet{walker2025structuredlinearcdesmaximally}.
\begin{enumerate}\itemsep1pt
\item \emph{Cycle navigation.} A walk on a cycle of length five that starts at the origin, and each step moves forward, moves back or stays; the target is the final position. A random guess gets $20\%$.
\item \emph{Even pairs.} The input is a string using a two-symbol alphabet, and the target is whether it contains an even number of symbol changes, which is the case if and only if its first and last symbols agree. A random guess scores $50\%$.
\item \emph{Modular arithmetic without brackets.} The input is an expression using the digits $0$ to $4$ with addition, subtraction and multiplication, evaluated modulo five with multiplication first. A random guess scores $20\%$.
\item \emph{Parity.} The input is a string using an alphabet of two symbols, and the target is whether the second symbol occurs an even number of times. A random guess scores $50\%$.
\end{enumerate}
Models are trained on sequences of length $3$ to $40$ and evaluated on $8\,192$ sequences of length $40$ to $256$, thus testing length generalisation. Each sequence is labelled only at its final position, where the cross-entropy loss and the accuracy are both computed. We report the best accuracy over evaluations every $5\,000$ steps on five seeds, as \citet{walker2025structuredlinearcdesmaximally} do, and a run stops early once it exceeds $99.95\%$. These evaluations therefore also select the checkpoint. To check that this selection does not inflate the results, we evaluate each selected checkpoint once more on $8\,192$ newly drawn sequences of lengths $40-256$. The EWS then reaches $99.99 \pm 0.1$, $100.00 \pm 0.00$, $96.54 \pm 0.76$ and $100.00 \pm 0.00$ on cycle navigation, even pairs, modular arithmetic and parity, all within $0.2$ of Table~\ref{tab:formal}.

Training runs for $100\,000$ steps at batch $256$ with AdamW, weight decay $0.01$ on every parameter but the embedding, a linear warmup over the first tenth of training followed by cosine decay to $10^{-5}$ from a peak of $4.9\times10^{-3}$; gradients are clipped at a norm of one within each parameter group. We use no dropout, and label smoothing of $0.05$ on parity and even pairs and $0.1$ on the other two; on modular arithmetic the generators have no weight decay and a penalty of $10^{-3}$ on the sum of their Frobenius norms. The configurations are given in Table~\ref{tab:fl-config}. 

\begin{table}[h]
\centering
\scriptsize
\setlength{\tabcolsep}{2.4pt}
\renewcommand{\arraystretch}{1.08}
\caption{\small
Configurations on the regular language tasks, using the names of
Table~\ref{tab:families}. Block one has $m=64$ and block two $m=24$,
both at depth one.
}
\label{tab:fl-config}

\begin{tabular}{@{}lcllllcr@{}}
\toprule
& & \multicolumn{2}{c}{Block 1} & \multicolumn{2}{c}{Block 2} & & \\
\cmidrule(lr){3-4}
\cmidrule(lr){5-6}
Task & $H$ & Generator & Clock & Generator & Clock & Smoothing & Params. \\
\midrule

Cycle nav.
& $128$
& rotation
& clamped
& rotation
& variation
& $0.1$
& $92\,219$ \\

Even pairs
& $128$
& rotation
& gated
& rotation
& variation
& $0.05$
& $90\,290$ \\

Mod.\ arith.
& $128$
& skew
& gated
& rotation
& variation
& $0.1$
& $96\,165$ \\

Parity
& $128$
& rotation
& gated
& skew
& variation
& $0.05$
& $90\,554$ \\

\bottomrule
\end{tabular}
\end{table}

\end{document}